\documentclass[11pt]{article}
\usepackage[letterpaper, left=1in, right=1in, top=1in,bottom=1in]{geometry}

\usepackage{parskip}

\usepackage{booktabs} 

\usepackage[utf8]{inputenc}

\usepackage{geometry}

\usepackage{graphpap,amscd,mathrsfs,graphicx,lscape,enumitem,dsfont,bm,url,subfigure}
\usepackage{epsfig,amstext,xspace}
\usepackage{algorithm,comment}
\usepackage{algorithmic}
\usepackage{thmtools,thm-restate}
\usepackage{pifont}

\usepackage{amssymb}
\usepackage{amsmath}
\usepackage{amsthm}

\usepackage{algorithm,algorithmic}
\usepackage{auxiliary}
\usepackage[utf8]{inputenc} 
\usepackage[T1]{fontenc}    
\usepackage{hyperref}       
\usepackage{url}            
\usepackage{booktabs}       
\usepackage{amsfonts}       
\usepackage{nicefrac}       
\usepackage{microtype}      
\usepackage{framed}
\usepackage{paralist}
\usepackage{makecell}

\theoremstyle{plain}

\usepackage[textsize=tiny]{todonotes}

\usepackage{hyperref}       
\usepackage{url}            
\usepackage{booktabs}       
\usepackage{amsfonts}       
\usepackage{nicefrac}       
\usepackage{microtype}      
\usepackage{xcolor}         
\usepackage{pifont}
\usepackage{amsthm}
\usepackage{amsfonts,amssymb}
\usepackage{amsmath}
\usepackage{amssymb}
\usepackage{amsmath}
\usepackage{amsthm}
\usepackage[utf8]{inputenc} 
\usepackage[T1]{fontenc}    
\usepackage{hyperref}       
\usepackage{url}            
\usepackage{booktabs}       
\usepackage{amsfonts}       
\usepackage{nicefrac}       
\usepackage{microtype}      
\usepackage{color}
\usepackage{makecell}

\usepackage{algorithm}
\usepackage{algorithmic}
\usepackage{extarrows}
\usepackage{caption}
\usepackage{subfigure}
\usepackage{graphicx}
\usepackage{sidecap}
\usepackage{float}
\usepackage{framed}
\usepackage{tabularx}
\usepackage{twoopt}
\usepackage{adjustbox}

\usepackage{lscape}

\usepackage{tikz}
\usetikzlibrary{fit}

\newcommand\tikzmark[1]{%
  \tikz[remember picture,overlay]\node[inner xsep=0pt] (#1) {};}
\newcommandtwoopt\Textbox[5][14cm][2cm]{%
\begin{tikzpicture}[remember picture,overlay]
  \coordinate (aux) at ([xshift=#1]#4);
  \node[inner ysep=3pt,yshift=0.5ex,draw=pink,thick,
    fit=(#3) (aux),baseline] 
    (box) {};
  \node[text width=#2,anchor=north east,
    font=\sffamily\footnotesize,
  align=right
    ] 
    at (box.north east) {#5};
\end{tikzpicture}%
}

\usepackage{bm}
\usepackage{makecell}

\newcommand{\policy}{\pi_{\bm{\theta}}}
\newcommand{\policyy}{\pi_{{\bm\theta}^{'}}}

\newcommand{\calA}{\mathcal{A}}

\newcommand{\calC}{\mathcal{C}}
\newcommand{\calD}{\mathcal{D}}

\newcommand{\calL}{\mathcal{L}}
\newcommand{\calS}{\mathcal{S}}
\newcommand{\calM}{\mathcal{M}}

\newcommand{\bA}{\mathbf{A}}
\newcommand{\bG}{\mathbf{G}}
\newcommand{\bD}{\mathbf{D}}
\newcommand{\bH}{\mathbf{H}}
\newcommand{\bI}{\mathbf{I}}
\newcommand{\bL}{\mathbf{L}}
\newcommand{\bP}{\mathbf{P}}

\newcommand{\Pro}{\mathbb{P}}

\newcommand{\ba}{\mathbf{a}}
\newcommand{\bd}{\mathbf{d}}

\newcommand{\bg}{\mathbf{g}}

\newcommand{\bv}{\mathbf{v}}

\usepackage{bm}
\newcommand{\mathbr}[1]{\bm{\mathbf{#1}}}

\newcommand{\vrho}{\mathbr{\rho}}

\newcommand{\hyscoreprime}[1][\vtheta]{\nabla_{\vrho'}\log\nu_{\vrho'}}

\usepackage{authblk}
\begin{document}
\title{CUP: A Conservative Update Policy Algorithm for Safe Reinforcement Learning \footnote{L.Yang do partial work of this submission when studying at Zhejiang University, and L.Yang now is with Peking University.}}

 \author[1,2]{Long Yang}
  \author[2]{Jiaming Ji}
  \author[2]{Juntao Dai}
      \author[3]{Linrui Zhang}
   \author[4]{Yu Zhang}
   \author[2]{Pengfei Li}
\author[2]{Gang Pan}
\affil[1]{School of Artificial Intelligence, Peking University, Beijing, China}
\affil[2]{College of Computer Science and Technology, Zhejiang University, China}
\affil[3]{Tsinghua University,Beijing,China}
\affil[4]{Netease Games AI Lab, HangZhou, China}
\affil[ ]{\textsuperscript{1}\texttt{\{yanglong001\}@pku.edu.cn}}
\affil[ ]{\textsuperscript{2}\texttt{\{juntaodai,gpan\}@zju.edu.cn}}
\date{\today}

\maketitle
\begin{abstract}
Safe reinforcement learning (RL) is still very challenging since it requires the agent to consider both return maximization and safe exploration.
In this paper, we propose CUP, a \textbf{C}onservative \textbf{U}pdate \textbf{P}olicy algorithm with a theoretical safety guarantee.
We derive the CUP based on the new proposed performance bounds and surrogate functions.
Although using bounds as surrogate functions to design safe RL algorithms have appeared in some existing works, we develop them at least three aspects: 
\textbf{(i)} We provide a rigorous theoretical analysis to extend the surrogate functions to generalized advantage estimator (GAE).
GAE significantly reduces variance empirically while maintaining a tolerable level of bias, which is an efficient step for us to design CUP;
\textbf{(ii)} The proposed bounds are tighter than existing works, i.e., using the proposed bounds as surrogate functions are better local approximations to the objective and safety constraints.
\textbf{(iii)} The proposed CUP provides a non-convex implementation via first-order optimizers, which does not depend on any convex approximation.
Finally, extensive experiments show the effectiveness of CUP where the agent satisfies safe constraints. We have opened the source code of CUP at \url{https://github.com/RL-boxes/Safe-RL}.
\end{abstract}

\section{Introduction}

Reinforcement learning (RL) \cite{sutton1998reinforcement} has achieved significant successes in many fields \cite{mnih2015human,silver2017mastering,openaifive2019,afsar2021reinforcement},Dota \cite{openaifive2019},
robotics \cite{deisenroth2013survey}, playing Go \cite{silver2016mastering,silver2017mastering}, 
recommendation system \cite{afsar2021reinforcement},
and Starcraft \cite{vinyals2019alphastar}.
However, most RL algorithms improve the performance under the assumption that an agent is free to explore any behaviors.
In real-world applications, only considering return maximization is not enough, and we also need to consider safe behaviors.
For example, a robot agent should avoid playing actions that irrevocably harm its hardware, and a recommender system should avoid presenting offending items to users.
Thus, it is crucial to consider \emph{safe exploration} for RL, which is usually formulated as constrained Markov decision processes (CMDP) \cite{altman1999constrained}.

It is challenging to solve CMDP since traditional approaches (e.g., Q-learning \cite{watkins1989learning} \& policy gradient \cite{williams1992simple}) usually violate the safe exploration constraints, which is undesirable for safe RL.
Recently, \cite{AchiamHTA17,yang2020projection,bharadhwaj2021conservative} suggest to use some surrogate functions to replace the objective and constraints.
However, their implementations involve some convex approximations to the non-convex objective and safe constraints,
which leads to many error sources and troubles.
Concretely, \cite{AchiamHTA17,yang2020projection,bharadhwaj2021conservative} approximate the non-convex objective (or constraints) with first-order or second Taylor expansion,
but their implementations still lack a theory to show the error difference between the original objective (or constraints) and its convex approximations.
Besides, their approaches involve the inverse of a high-dimension Fisher information matrix, which causes their algorithms to require a costly computation for each update when solving high-dimensional RL problems.

\subsection{Our Main Work}
To address above problems, we propose the \emph{conservative update policy} (CUP) algorithm with a theoretical safety guarantee.
We derive the CUP bases on new surrogate functions, and provide a practical implementation of CUP that does not depend on any convex approximation to adapt high-dimensional safe RL.

Concretely, in Section \ref{sec:generalized-bound}, Theorem \ref{them:general-performance-difference} shows generalized difference bounds between two arbitrary policies for the objective and constraints.
Those bounds provide principled approximations to the objective and constraints, which are theoretical foundations for us to use those bounds as surrogate functions to replace 
objective and constraints to design algorithms.

Although using difference bound to replace objective or constraints has appeared in some existing works (e.g., \cite{kakade2002approximately,schulman2015trust,AchiamHTA17}), Theorem \ref{them:general-performance-difference} improves their bounds at least two aspects:
\textbf{(i)} Firstly, our rigorous theoretical analysis extends the bound w.r.t. generalized advantage estimator (GAE) \cite{schulman2016high}.
GAE significantly reduces variance while maintains a tolerable level of bias, which is one of the critical steps for us to design efficient algorithms in the later section.
Although \cite{zhang2020first,kang2021learning} have applied GAE to solve safe RL problems, their approaches are empirical and lack a theoretical analysis w.r.t. GAE.
Thus, our result provides a theory to illustrate the effectiveness of the work \cite{zhang2020first,kang2021learning}.
\textbf{(ii)} Our new bounds in refine classic difference bounds.
For example, our bounds are more compact than \cite{AchiamHTA17}, i,e., using our new bounds as surrogate functions are better local approximations to the objective and constraints.
Besides, the surrogate functions w.r.t. our new bounds are more accessible to be estimated from the samples than the approaches appears in \cite{kakade2002approximately,schulman2015trust}), for more discussions, please see Remark \ref{remark:Comparison-with-kakade}.

In Section \ref{sec:algorithm}, we provide the necessary details of the proposed CUP.
The CUP contains two steps: it performs a policy improvement at first, then it projects the policy back onto the safe region to reconcile the constraint violation.
Theorem \ref{them-re-cost} shows a lower bound on policy improvement and an upper bound on constraint violation for CPU at each update.
Notably, the result in Theorem \ref{them-re-cost} shows the bound of CUP is more compact than state-of-the-art safe RL algorithms: CPO \cite[Proposition 1-2]{AchiamHTA17}, PCPO \cite[Theorem 1]{yang2020projection} and FOCOPS \cite{zhang2020first}, which provides a partial explanation for why CUP is so good in practice.
For more discussions, please refer to Remark \ref{remark:Comparison-with-soda}.
Finally, we provide a practical implementation of sample-based CUP. Such an implementation allows us to use deep neural networks to train a model.
Mainly, CUP does not depend on any convex approximation for objective and constraints, and it optimizes the objective according to the first-order optimizer.
Extensive high-dimensional experiments on continuous control tasks show the effectiveness of CUP where the agent satisfies safe constraints.

\section{Preliminaries}
\label{Background and Notations}

Reinforcement learning (RL) \cite{sutton1998reinforcement} is often formulated as 
a \emph{Markov decision process} (MDP) \cite{puterman2014markov} that is a tuple $\mathcal{M}=(\mathcal{S},\mathcal{A},\mathbb{P},{r},\rho_0,\gamma)$.
Here $\mathcal{S}$ is state space, $\mathcal{A}$ is action space.
$\mathbb{P}(s^{'}|s,a)$ is probability of state transition from $s$ to $s^{'}$ after playing $a$.
$r(\cdot):\mathcal{S}\times\mathcal{S}\times\mathcal{A}\rightarrow \R$,
and $r(s'|s,a)$ denotes the reward that the agent observes when state transition from $s$ to $s^{'}$ after it plays $a$.
$\rho_{0}(\cdot):\mathcal{S}\rightarrow[0,1]$ is the initial state distribution and $\gamma\in(0,1)$.

A stationary parameterized policy $\policy$ is a probability distribution defined on $\mathcal{S}\times\mathcal{A}$, $\policy(a|s)$ denotes the probability of playing $a$ in state $s$.
We use $\Pi_{\bm{\theta}}$ to denote the set of all stationary policies, where $\Pi_{{{\bm{\theta}}}}=\{\pi_{{{\bm{\theta}}}}:{{\bm{\theta}}}\in\R^{p}\}$, and ${{\bm{\theta}}}$ is a parameter needed to be learned.
Let $\mathbf{P}_{\pi_{\bm \theta}}\in\R^{|\calS|\times|\calS|}$ be a state transition probability matrix, and their components are:
$
\mathbf{P}_{\pi_{\bm \theta}}[s,s'] =\sum_{a\in\mathcal{A}}\pi_{\bm{\theta}}(a|s)\mathbb{P}(s'|s,a)=:\Pro_{\policy}(s^{'}|s),
$
which denotes one-step state transformation probability from $s$ to $s^{'}$ by executing $\policy$.
Let $\tau=\{s_{t}, a_{t}, r_{t+1}\}_{t\ge0}$ be a trajectory generated by $\policy$, 
where $s_{0}\sim\rho_{0}(\cdot)$, $a_{t}\sim\policy(\cdot|s_t)$, $s_{t+1}\sim \mathbb{P}(\cdot|s_{t},a_{t})$, and $r_{t+1}=r(s_{t+1}|s_t,a_t)$.
We use $\mathbb{P}_{\policy}(s_t=s^{'}|s)$ to denote the probability of visiting the state $s^{'}$ after $t$
time steps from the state $s$ by executing $\policy$.
Due to the Markov property, $\mathbb{P}_{\policy}(s_t=s^{'}|s)$ is $(s,s^{'})$-th component of the matrix $\mathbf{P}^{t}_{\pi_{\bm \theta}}$, i.e.,
$
\mathbb{P}_{\policy}(s_t=s^{'}|s)=\mathbf{P}^{t}_{\pi_{\bm \theta}}[s,s^{'}].
$
Finally, let $d_{\policy}^{s_0}(s)=(1-\gamma)\sum_{t=0}^{\infty}\gamma^{t}\mathbb{P}_{\policy}(s_t=s|s_0)$ be the stationary state distribution of the Markov chain (starting at $s_0$) induced by policy $\policy$.
We define
$
d_{\policy}^{\rho_0}(s)=\mathbb{E}_{s_0\sim\rho_{0}(\cdot)}[d_{\policy}^{s_0}(s)]
$
as the discounted state visitation distribution on initial distribution $\rho_0 (\cdot)$.

The \emph{state value function} of $\policy$ is defined as $V_{\policy}(s) = \mathbb{E}_{\policy}[\sum_{t=0}^{\infty}\gamma^{t}r_{t+1}|s_{0} = s],$
where $\mathbb{E}_{\policy}[\cdot|\cdot]$ denotes a conditional expectation on actions which are selected by $\policy$.
Its \emph{state-action value function} is $Q_{\policy}(s,a) = \mathbb{E}_{\policy}[\sum_{t=0}^{\infty}\gamma^{t}r_{t+1}|s_{0} = s,a_{0}=a]$, 
and advantage function is $A_{\policy}(s,a)=Q_{\policy}(s,a) -V_{\policy}(s)$.
The goal of reinforcement learning is to maximize $J(\policy)$:
\begin{flalign}
  \label{J-objectiove}
  J(\policy)=\E_{s\sim d^{\rho_0}_{\policy}(\cdot)}[V_{\policy}(s)].
\end{flalign}

\subsection{Policy Gradient and GAE}

Policy gradient \cite{williams1992simple,sutton2000policy} is widely used to solve policy optimization, which maximizes the expected total reward by repeatedly estimating the gradient $\nabla J(\policy)$.
The work
\cite{schulman2016high} summarize several different related expressions for the policy gradient:
\begin{flalign}
\nabla J(\policy)=\E\left[
\sum_{t=0}^{\infty}\Psi_{t}\nabla\log\policy(a_t|s_t)
\right],
\end{flalign}
where $\Psi_{t}$ can be total discounted reward of the trajectory, value function, advantage function or temporal difference (TD) error.
As stated by \cite{schulman2016high}, the choice $\Psi_{t}=A(s_t,a_t)$ yields almost the lowest possible variance, which is consistent with the theoretical analysis \cite{greensmith2004variance,wu2018variance}.
Furthermore, \cite{schulman2016high} propose generalized advantage estimator (GAE) $\hat{A}^{\text{GAE}(\gamma,\lambda)}_{t}(s_t,a_t)$ to replace $\Psi_{t}$: for any $\lambda\in[0,1]$,
\begin{flalign}
\hat{A}^{\text{GAE}(\gamma,\lambda)}_{t}(s_t,a_t)=\sum_{\ell=0}^{\infty}(\gamma\lambda)^{\ell}\delta^{V}_{t+\ell},
\end{flalign}
where $\delta^{V}_{t}=r_{t+1}+\gamma V(s_{t+1})-V(s_{t})$ is TD error, and $V(\cdot)$ is an estimator of value function.
GAE is an efficient technique for data efficiency and reliable performance of reinforcement learning.

\subsection{Safe Reinforcement Learning}

Safe RL is often formulated as 
a constrained MDP (CMDP) $\calM\cup\calC$ \cite{altman1999constrained},
which is a standard MDP $\calM$ augmented with an additional constraint set $\calC$.
The set $\calC=\{(c_i,b_i)\}_{i=1}^{m}$,
where $c_i$ are cost functions: $c_i : \calS\times\calA \rightarrow \R$, and limits are $b_i$, $i = 1,\cdot,m$. 
The \emph{cost-return} is defined as:
$J^{c_i}(\policy)=\E_{\policy}\left[\sum_{t=0}^{\infty}\gamma^{t}c_{i}(s_{t},a_{t})\right]$, 
the feasible policy set $\Pi_{\calC}$ is defined as:
\[
\Pi_{\calC}=
\bigcap_{i=1}^{m}
\left\{
\policy\in\Pi_{\bm{\theta}}~~\text{and}~~J^{c_i}(\policy)_\leq b_i
\right\}.
\]
The goal of safe RL is to search the optimal policy $\pi_{\star}$ s.t.
\begin{flalign}
\label{def:problem-setting}
\pi_{\star}=\arg\max_{\policy\in\Pi_{\calC}} J(\policy).
\end{flalign}
Furthermore, we define value functions, action-value functions, and advantage functions for the auxiliary costs in analogy to $V_{\policy}, Q_{\policy}$, and $A_{\policy}$, 
with $c_i$ replacing $r$ respectively, we denote them as $V^{c_i}_{\policy}, Q^{c_i}_{\policy}$, and $A^{c_i}_{\policy}$.
For example, $V^{c_i}_{\policy}(s) = \mathbb{E}_{\policy}\left[\sum_{t=0}^{\infty}\gamma^{t}c_i(s_{t},a_{t})|s_{0} = s\right]$.
Without loss of generality, we will restrict our discussion to the case of one constraint with a cost function $c$ and upper bound $b$.
Finally, we extend the GAE w.r.t. auxiliary cost function $c$:
\begin{flalign}
\label{def:gae-cost}
\hat{A}^{\text{GAE}(\gamma,\lambda)}_{C,t}(s_t,a_t)=\sum_{\ell=0}^{\infty}(\gamma\lambda)^{\ell}\delta^{C}_{t+\ell},
\end{flalign}
where $\delta^{C}_{t}=r_{t+1}+\gamma C(s_{t+1})-C(s_{t})$ is TD error, and $C(\cdot)$ is an estimator of cost function $c$.

\section{Generalized Policy Performance Difference Bounds}

\label{sec:generalized-bound}

In this section, we show some generalized policy optimization performance bounds for $J(\policy)$ and $J^{c}(\policy)$. 
The proposed bounds provide some new certain surrogate functions w.r.t. the objective and cost function, which are theoretical foundations for us to design efficient algorithms to improve policy performance and satisfy constraints.
Additionally, those bounds refine or extend some existing works (e.g., \cite{kakade2002approximately,schulman2015trust,AchiamHTA17}) to GAE case that significantly
reduces variance while maintains a tolerable level of bias, which is one of the key steps for us to propose efficient algorithms in the later section.

Before we present our new bounds, let us revisit a classic result about policy performance difference from \cite{kakade2002approximately}, i.e., the next Eq.(\ref{performance-difference-2002}),
\begin{flalign}
\label{performance-difference-2002}
J(\pi_{{\bm{\theta}}})-J(\policyy)
=(1-\gamma)^{-1}\E_{s\sim d_{\policy}^{\rho_0}(\cdot),a\sim\policy (\cdot|s)}\left[A_{\policyy}(s,a)\right].
\end{flalign}
Eq.(\ref{performance-difference-2002}) shows a difference between two arbitrary policies $\pi_{\bm\theta}$ and $\pi_{{\bm\theta}^{'}}$ with different parameters $\bm{\theta}$ and $\bm{\theta}^{'}$.
According to (\ref{performance-difference-2002}), we rewrite the policy optimization (\ref{def:problem-setting}) as follows
\begin{flalign}
\label{def:rewriting-problem-setting}
\pi_{\star}=\arg\max_{\policy\in\Pi_{\calC}} \E_{s\sim d_{\policy}^{\rho_0}(\cdot),a\sim\policy (\cdot|s)}\left[A_{\policyy}(s,a)\right].
\end{flalign}
However, Eq.(\ref{performance-difference-2002}) or (\ref{def:rewriting-problem-setting}) is very intractable for sampling-based policy optimization since it requires the data comes from the (unknown) policy $\policy$ that needed to be learned.

In this section, our new bound refines the result (\ref{performance-difference-2002}), which provide the sights for surrogate functions to solve safe RL problem (\ref{def:problem-setting}).
For more discussions about the difference between our new bound and Eq.(\ref{performance-difference-2002}), see Remark \ref{remark:Comparison-with-kakade}.

\subsection{Some Additional Notations}

We use a bold lowercase letter to denote a vector, e.g., $\ba=(a_1,a_2,\cdots,a_n)$, and its $i$-th element $\ba[i]=:a_{i}$.
Let $\varphi(\cdot):\calS\rightarrow\R$ be a function defined on $\calS$, $\delta_t^{\varphi}=r(s_{t+1}|s_t,a_t)+\gamma\varphi(s_{t+1})-\varphi(s_{t})$ is TD error w.r.t. $\varphi(\cdot)$. 
For two arbitrary policies $\pi_{\bm\theta}$ and $\pi_{{\bm\theta}^{'}}$, we denote $\delta^{\varphi}_{\policy,t}(s)$ as the expectation of TD error, and define $ \Delta_{t}^{\varphi}(\policy,\policyy,s)$ as the difference between $\delta^{\varphi}_{\policy,t}(s)$ and $\delta^{\varphi}_{\policyy,t}(s)$: $\forall s\in\calS$, 
\begin{flalign}
\nonumber
\delta^{\varphi}_{\policy,t}(s)=\underset{\begin{subarray}{c} s_t \sim \Pro_{\policyy}(\cdot|s)\\ a_{t}\sim{\policyy}(\cdot|s_t)\\ s_{t+1}\sim\Pro(\cdot|s_t,a_t) \end{subarray}}
\E\left[\delta_t^{\varphi}\right],~~~~~~
\Delta_{t}^{\varphi}(\policy,\policyy,s)=\underset{\begin{subarray}{c} s_t \sim \Pro_{\policyy}(\cdot|s)\\ a_{t}\sim{\policyy}(\cdot|s_t)\\ s_{t+1}\sim\Pro(\cdot|s_t,a_t) \end{subarray}}
\E\left[\left(\dfrac{\policy(a_t|s_t)}{\policyy(a_t|s_t)}-1\right)\delta_t^{\varphi}\right].
\end{flalign}
Furthermore, we introduce  two vectors $\bm{\delta}^{\varphi}_{\policy,t},\bm{\Delta}_{t}^{\varphi}(\policy,\policyy)\in\R^{|\calS|}$, and their corresponding components are:
\begin{flalign}
\nonumber
\bm{\delta}^{\varphi}_{\policy,t}[s]=\delta^{\varphi}_{\policy,t}(s),~~\bm{\Delta}_{t}^{\varphi}(\policy,\policyy)[s]=\Delta_{t}^{\varphi}(\policy,\policyy,s).
\end{flalign}
Let matrix $\mathbf{P}^{(\lambda)}_{\pi_{\bm \theta}}=(1-\gamma\lambda)\sum_{{t}=0}^{\infty}(\gamma\lambda)^{{t}}\bP^{{t}+1}_{\policy}$, 
where $\lambda\in[0,1]$.
It is similar to the normalized discounted distribution $d_{\pi_{\bm {\theta}}}^{\rho_0}(s)$, we extend it to $\lambda$-version and denote it as ${d}_{\pi_{\bm {\theta}}}^{\lambda}(s)$:
\begin{flalign}
\nonumber
{d}_{\pi_{\bm {\theta}}}^{\lambda}(s)&=\E_{s_0\sim\rho_{0}(\cdot)}
\left[
(1-\tilde\gamma)\sum_{t=0}^{\infty}{\tilde\gamma}^{t}{\mathbb{P}}^{(\lambda)}_{\pi_{\bm {\theta}}}(s_t=s|s_0)
\right]\\
\nonumber
&=\dfrac{1-\gamma}{1-\gamma\lambda}\E_{s_0\sim\rho_{0}(\cdot)}
\left[
\sum_{t=0}^{\infty}\left(\dfrac{\gamma(1-\lambda)}{1-\gamma\lambda}\right)^{t}{\mathbb{P}}^{(\lambda)}_{\pi_{\bm {\theta}}}(s_t=s|s_0)
\right],
\end{flalign}
where $\tilde{\gamma}=\frac{\gamma(1-\lambda)}{1-\gamma\lambda}$, the probability $\mathbb{P}^{(\lambda)}_{\pi_{\bm {\theta}}}(s_t=s|s_0)$ is the $(s_0,s)$-th component of the matrix product 
\[\left(\mathbf{P}^{(\lambda)}_{\pi_{\bm \theta}}\right)^{t}=\prod_{i=1}^{t} \mathbf{P}^{(\lambda)}_{\pi_{\bm \theta}}=\underbrace{\mathbf{P}^{(\lambda)}_{\pi_{\bm \theta}}\cdot\mathbf{P}^{(\lambda)}_{\pi_{\bm \theta}}\cdots\mathbf{P}^{(\lambda)}_{\pi_{\bm \theta}}}_{t ~~\text{times}}.\]
Finally, we introduce a vector $\bd_{\pi_{\bm {\theta}}}^{\lambda}\in\R^{|\calS|}$, and its components are: $\bd_{\pi_{\bm {\theta}}}^{\lambda}[s]=d_{\pi_{\bm {\theta}}}^{\lambda}(s).$

\subsection{Main Results}

\begin{theorem}
[Generalized Policy Performance Difference]
\label{them:general-performance-difference}
For any function $\varphi(\cdot):\calS\rightarrow\R$, for two arbitrary policies $\pi_{\bm\theta}$ and $\pi_{{\bm\theta}^{'}}$,  
for any $p,q\in[1,\infty)$ such that $\frac{1}{p}+\frac{1}{q}=1$, we define two error terms:
\begin{flalign}
\label{error-01}
&\epsilon^{\varphi,(\lambda)}_{p,q,t}(\policy,\policyy)=:\|\bd_{\pi_{\bm {\theta}}}^{\lambda}-\bd_{\policyy}^{\lambda}\|_{p}\|{\bm{\delta}}^{\varphi}_{\policy,t}\|_{q},\\
\label{error-02}
L^{\varphi, \pm}_{p,q}(\policy,\policyy)&=:
\dfrac{1}{1-\tilde\gamma}
\sum_{t=0}^{\infty}\gamma^t\lambda^{t}\E_{s\sim{d}_{\policyy}^{\lambda}(\cdot)} \left[
\Delta_{t}^{\varphi}(\policy,\policyy,s) \pm\epsilon^{\varphi,(\lambda)}_{p,q,t}(\policy,\policyy)\right].
\end{flalign}
Then, the following bound w.r.t. policy performance difference $J(\pi_{\bm \theta})-J(\pi_{{\bm \theta}^{'}})$ holds:
\begin{flalign} 
\label{bound-diff-01}
L^{\varphi,-}_{p,q,}(\policy,\policyy)
\leq J(\pi_{\bm \theta})-J(\pi_{{\bm \theta}^{'}})
\leq
L^{\varphi,+}_{p,q,}(\policy,\policyy)
.
\end{flalign}
\end{theorem}
\begin{proof}
See Appendix \ref{sec:proof-them-01}.
\end{proof}
The bound (\ref{bound-diff-01}) is \emph{tight}, i.e., if $\policy=\policyy$, all the three terms in Eq.(\ref{bound-diff-01}) are zero identically.
From Eq.(\ref{error-02}), we know the performance difference bound
$L^{\varphi, \pm}_{p,q}(\policy,\policyy)$ (\ref{bound-diff-01}) can be interpreted by two distinct difference parts:
\textbf{(i)} the first difference part, i.e., the expectation $\Delta_{t}^{\varphi}(\policy,\policyy,s)$, which is determined by the difference between TD errors of $\policy$ and $\policyy$;
\textbf{(ii)} the second difference part, i.e., the discounted distribution difference $\epsilon^{\varphi,(\lambda)}_{p,q,t}(\policy,\policyy)$, which is determined by the gap between the normalized discounted distribution of $\policy$ and $\policyy$.
Thus, the difference of both TD errors and discounted distribution determine the policy difference $J(\policy)-J(\policyy)$.

The different choices of $p$ and $q$ lead Eq.(\ref{bound-diff-01}) to be different bounds.
If $p=1,q=\infty$, we denote \[\epsilon^{\varphi}_{\policy,t}=:\|{\bm{\delta}}^{\varphi}_{\policy,t}\|_{q}=\max_{s_{t}\in\calS}\E_{a_t\sim\policy(\cdot|s_t),s_{t+1}\sim\Pro(\cdot|s_t,a_t)}[|\delta_{t}^{\varphi}|],\]
then,
according to Lemma \ref{lem:difference-distri} (see Appendix \ref{sec:difference-distri}), when $p=1,q=\infty$, then error $\epsilon^{\varphi,(\lambda)}_{p,q,t}(\policy,\policyy)$ is reduced to:
\begin{flalign}
\nonumber
\epsilon^{\varphi,(\lambda)}_{p,q,t}(\policy,\policyy)\big|_{p=1,q=\infty}
\leq\dfrac{1}{1-\tilde\gamma}\cdot
\dfrac{\gamma(1-\lambda)\epsilon^{\varphi}_{\policy,t}}{\left|1-2\gamma\lambda|\calS||\calA|\right|}
\E_{s\sim d_{\policyy}^{\lambda}(\cdot)}
\left[2D_{\text{TV}}(\policyy,\policy)[s]\right],
\end{flalign}
where $D_{\text{TV}}(\policyy,\policy)[s]$ is the total variational divergence between action distributions at state $s$, i.e.,
\[
2D_{\text{TV}}(\policyy,\policy)[s]=\sum_{a\in\calA}\left|{{\policyy}}(a|s)-{{\policy}}(a|s)\right|.
\]
Finally, let $\varphi=V_{\policyy}$, the left side of (\ref{bound-diff-01}) in Theorem \ref{them:general-performance-difference} implies a lower bound of performance difference, which illustrates the worse case of approximation error, we present it in Proposition \ref{propo-01}.
\begin{proposition}[Worse case approximation error]
\label{propo-01}
For any two policies $\pi_{\bm\theta}$ and $\pi_{{\bm\theta}^{'}}$, let $\epsilon^{V}_{\policy}(\policyy)=:\sup_{t\in\N^{+}}\{\epsilon^{\varphi}_{\policy,t}: \varphi=V_{\policyy}\}$, then the following bound holds
\begin{flalign}
\label{pro1-bound-01}
J(\policy)-J(\policyy)
\ge\dfrac{1}{1-\tilde\gamma}\E_{s\sim{d}_{\policyy}^{\lambda}(\cdot),a\sim\policy(\cdot|s)}
\left[
A^{\emph{GAE}(\gamma,\lambda)}_{\policyy}(s,a)
-
\frac{2\gamma(1-\lambda)\epsilon^{V}_{\policy}(\policyy)}{(1-\gamma\lambda)\left|1-2\gamma\lambda|\calS||\calA|\right|}
D_{\emph{TV}}(\policyy,\policy)[s]
\right].
\end{flalign}
\end{proposition}
If $\lambda\rightarrow0$, then the distribution ${d}_{\policyy}^{\lambda}(\cdot)$ is reduced to ${d}_{\policyy}^{\rho_0}(\cdot)$ and the bound (\ref{pro1-bound-01}) is reduced to 
\begin{flalign}
\label{bound-lam-0}
J(\policy)-J(\policyy)
\ge
\frac{1}{1-\gamma}\E_{s\sim{d}_{\policyy}^{\rho_0}(\cdot),a\sim\policy(\cdot|s)}
\left[
A_{\policyy}(s,a)
-2\gamma\epsilon^{V}_{\policy}(\policyy)
D_{\text{TV}}(\policyy,\policy)[s]
\right].
\end{flalign}
Let us review \cite[Corollary 1]{AchiamHTA17}, which shows 
\begin{flalign}
\label{bound-achiam17-icml}
J(\policy)-J(\policyy)
\ge
\frac{1}{1-\gamma}\E_{s\sim{d}_{\policyy}^{\rho_0}(\cdot),a\sim\policy(\cdot|s)}
\left[
A_{\policyy}(s,a)
-2\frac{\gamma\epsilon^{V}_{\policy}(\policyy)}{1-\gamma}
D_{\text{TV}}(\policyy,\policy)[s]
\right].
\end{flalign}
Comparing (\ref{bound-lam-0}) to (\ref{bound-achiam17-icml}), our new bound (\ref{bound-lam-0}) is slightly tighter than the bound shown by the work \cite{AchiamHTA17}, concretely, our result improves the bound (\ref{bound-achiam17-icml}) by a factor $\dfrac{1}{1-\gamma}$.
The bound (\ref{bound-achiam17-icml}) has been used as a surrogate function for $J(\policy)-J(\policyy)$, and this idea has been developed as \emph{constrained policy optimization} by extensive works (e.g., \cite{koller2018learning,zhang2020first,yang2020accelerating,yang2020projection,zanger2021safe}).
Since the refined bound (\ref{pro1-bound-01}) contains GAE technique that significantly reduces variance while maintains a tolerable level of bias \cite{schulman2016high}, which implies using the bound (\ref{pro1-bound-01}) as a surrogate function could improve performance potentially.

\begin{remark}[Comparison with \cite{kakade2002approximately}]
\label{remark:Comparison-with-kakade}
The result (\ref{pro1-bound-01}) develops the classic performance difference (\ref{performance-difference-2002}) at least two aspects.
Firstly, the bound (\ref{pro1-bound-01}) extends from the advantage $A_{\policyy}$ (\ref{performance-difference-2002}) to GAE function $A_{\policyy}^{\emph{GAE}(\gamma,\lambda)}$.
Secondly, the following term in Eq.(\ref{pro1-bound-01}):
\begin{flalign}
\label{gap-01}
\dfrac{1}{1-\tilde\gamma}\E_{s\sim{d}_{\policyy}^{\lambda}(\cdot),a\sim\policy(\cdot|s)}\left[A^{\emph{GAE}(\gamma,\lambda)}_{\policyy}(s,a)\right]
\end{flalign} 
is an approximation for the difference $J(\policy)-J(\policyy)$,
while Eq.(\ref{performance-difference-2002}) shows an identity for difference $J(\policy)-J(\policyy)$.
It seems that Eq.(\ref{performance-difference-2002}) is a natural objective for return maximization, however, Eq.(\ref{performance-difference-2002}) is very intractable for sampling-based policy optimization since Eq.(\ref{performance-difference-2002}) requires the data comes from a the policy $\policy$ that needed to be learned.
The approximation (\ref{gap-01}) solves this problem by the expectation over the state distribution $d_{\policyy}^{\lambda}(\cdot)$ w.r.t. policy $\policyy$ and action distribution of another policy $\policy$.
Thus, bound (\ref{pro1-bound-01}) provides a tractable objective for sample-based optimization.
\end{remark}

Let $\varphi=V^{c}_{\policyy}$, Theorem \ref{them:general-performance-difference} implies an upper bound of cost function as presented in the next Proposition \ref{pro-02}, we will use it to make guarantee for safe policy optimization.

\begin{proposition}
\label{pro-02}
For any two policies $\pi_{\bm\theta}$ and $\pi_{{\bm\theta}^{'}}$, let $\epsilon^{C}_{\policy}(\policyy)=:\sup_{t\in\N^{+}}\{\epsilon^{\varphi}_{\policy,t}: \varphi=V^{c}_{\policyy}\}$, then
 the following bound holds
\begin{flalign}
\label{pro2-bound-02}
J^{c}(\policy)-J^{c}(\policyy)
\leq
\dfrac{1}{1-\tilde\gamma}\E_{s\sim{d}_{\policyy}^{\lambda}(\cdot),a\sim\policy(\cdot|s)}
\left[
A^{\emph{GAE}(\gamma,\lambda)}_{\policyy,C}(s,a)
+
\frac{2\gamma(1-\lambda)\epsilon^{C}_{\policy}(\policyy)}{(1-\gamma\lambda)\left|1-2\gamma\lambda|\calS||\calA|\right|}
D_{\emph{TV}}(\policyy,\policy)[s]
\right],
\end{flalign}
where we calculate $A^{\emph{GAE}(\gamma,\lambda)}_{\policyy,C}(s,a)$ according to the data sampled from $\policyy$ and (\ref{def:gae-cost}).
\end{proposition}

All above bound results (\ref{pro1-bound-01}) and (\ref{pro2-bound-02}) can be extended for a total variational divergence to KL-divergence between policies, which are desirable for policy optimization.
We obtain
\begin{flalign}
\label{inequlities}
\E_{s\sim{d}_{\policyy}^{\lambda}(\cdot)}\left[D_{\text{TV}}(\policyy,\policy)[s]\right]
\leq&
\E_{s\sim{d}_{\policyy}^{\lambda}(\cdot)}\left[\sqrt{\frac{1}{2}\text{KL}(\policyy,\policy)[s]}\right]
\leq
\sqrt{\frac{1}{2}\E_{s\sim{d}_{\policyy}^{\lambda}(\cdot)}\left[\text{KL}(\policyy,\policy)[s]\right]},
\end{flalign}
where $\text{KL}(\cdot,\cdot)$ is KL-divergence, and \[\text{KL}(\policyy,\policy)[s]=\text{KL}(\policyy(\cdot|s),\policy(\cdot|s));\] the first inequality follows Pinsker's inequality \cite{csiszar2011information} and the second inequality follows Jensen's inequality.
According to (\ref{inequlities}), we obtain the next Proposition \ref{propo-03}.
\begin{proposition}
\label{propo-03}
All the bounds in (\ref{pro1-bound-01}) and (\ref{pro2-bound-02}) hold if we make the following substitution:
\[
\E_{s\sim{d}_{\policyy}^{\lambda}(\cdot)}\left[D_{\emph{TV}}(\policyy,\policy)[s]\right]
\leftarrow
\sqrt{\frac{1}{2}\E_{s\sim{d}_{\policyy}^{\lambda}(\cdot)}\left[\emph{KL}(\policyy,\policy)[s]\right]}.
\]
\end{proposition}

\section{CUP: Conservative Update Policy}
\label{sec:algorithm}

According to the bounds in Proposition \ref{propo-01}-\ref{propo-03}, we develop new surrogate functions to replace the objective and constraints.
Inspired by two recent works \cite{yang2020projection,zhang2020first}, 
we propose the CUP (conservative update policy) algorithm that is a two-step approach contains \emph{performance improvement} and \emph{projection}.
Theorem \ref{them-re-cost} shows the proposed CUP guarantees the policy improvement and safe constraints.

\subsection{Methodology}

\textbf{Step 1: Performance Improvement.}

According to Proposition \ref{propo-01} and Proposition \ref{propo-03}, for an appropriate coefficient $\alpha_k$, we update policy as follows,
\begin{flalign}
\label{performance-improvement-01}
\pi_{{\bm{\theta}}_{k+\frac{1}{2}}}&=\arg\max_{\pi_{{\bm{\theta}}}\in\Pi_{{\bm{\theta}}}}
\left\{
\E_{s\sim{d}_{\pi_{\bm{\theta}_k}}^{\lambda}(\cdot),a\sim\policy(\cdot|s)}[
A^{\text{GAE}(\gamma,\lambda)}_{{\pi_{\bm{\theta}_k}}}(s,a)]-\alpha_k\sqrt{
\E_{s\sim{d}_{{\pi_{\bm{\theta}_k}}}^{\lambda}(\cdot)}\left[\text{KL}(\pi_{\bm{\theta}_k},\policy)[s]\right]}
\right\}\\
\label{performance-improvement}
&=\arg\max_{\pi_{{\bm{\theta}}}\in\Pi_{{\bm{\theta}}}}
\left\{
\underset{\begin{subarray}{c}s \sim{d}_{\pi_{\bm{\theta}_k}}^{\lambda}(\cdot)\\a\sim\pi_{\bm{\theta}_k}(\cdot|s)\end{subarray}}
\E\left[\frac{\pi_{\bm{\theta}}(a|s)}{\pi_{{\bm{\theta}}_k}(a|s)}
A^{\text{GAE}(\gamma,\lambda)}_{{\pi_{\bm{\theta}_k}}}(s,a)\right]-\alpha_k
\sqrt{
\E_{s\sim{d}_{{\pi_{\bm{\theta}_k}}}^{\lambda}(\cdot)}\left[\text{KL}(\pi_{\bm{\theta}_k},\policy)[s]\right]}
\right\}.
\end{flalign}
We replace (\ref{performance-improvement-01}) with an importance sampling with respect to $\pi_{\bm{\theta}_k}$ to obtain the expectation (\ref{performance-improvement}), and all the remains is to replace the expectation (\ref{performance-improvement}) by sample averages according to the trajectories collected by $\pi_{\bm{\theta}_k}$.

\textbf{Step 2: Projection.}

According to Proposition \ref{pro-02} and Proposition \ref{propo-03}, for an appropriate coefficient $\beta_k$, we project the policy $\pi_{{\bm{\theta}}_{k+\frac{1}{2}}}$ onto the safe constraint set.
Concretely, we use a measure $D(\cdot,\cdot)$ (e.g., KL divergence or $\ell_2$-norm) to minimize distance between $\pi_{{\bm{\theta}}_{k+\frac{1}{2}}}$ and $\policy$, and require the new policy satisfies the safe constraint:
\begin{flalign}
\label{projection}
&~~~~~~~~~~~~~~~~~~~~~~~~~~~~~~~~~~~~~~~~~\pi_{{\bm{\theta}}_{k+1}}=\arg\min_{\pi_{{\bm{\theta}}}\in\Pi_{{\bm{\theta}}}}~D\left(\pi_{{\bm{\theta}}},\pi_{{\bm{\theta}}_{k+\frac{1}{2}}}\right),\\
\nonumber
&\text{s.t.}~J^{c}(\pi_{{\bm{\theta}}_k})+\dfrac{1}{1-\tilde\gamma}\E_{s\sim{d}_{\pi_{\bm{\theta}_k}}^{\lambda}(\cdot),a\sim\policy(\cdot|s)}
\left[
A^{\text{GAE}(\gamma,\lambda)}_{{\pi_{\bm{\theta}_k}},C}(s,a)\right]+\beta_k
\sqrt{
\E_{s\sim{d}_{{\pi_{\bm{\theta}_k}}}^{\lambda}(\cdot)}\left[\text{KL}(\pi_{\bm{\theta}_k},\policy)[s]\right]}\leq b.
\end{flalign}
Until now, the particular choice of surrogate function is heuristically motivated, we show the policy and safe constraint guarantee of the proposed CUP  in Theorem \ref{them-re-cost}, and its proof shown in Appendix \ref{sec:app-them2}.
\begin{theorem}
\label{them-re-cost}
Let $\chi_k=\E_{s\sim{d}_{{\pi_{\bm{\theta}_k}}}^{\lambda}(\cdot)}\left[\emph{KL}\left(\pi_{\bm{\theta}_k},\pi_{\bm{\theta}_{k+\frac{1}{2}}}\right)[s]\right]$, if $\pi_{\bm{\theta}_k}$ and 
$\pi_{\bm{\theta}_{k+1}}$ are related to (\ref{performance-improvement})-(\ref{projection}),
then the lower bound on policy improvement, and upper bound on constraint violation are
\begin{flalign}
\nonumber
J(\pi_{\bm{\theta}_{k+1}})-J(\pi_{\bm{\theta}_{k}})\ge-\frac{\gamma(1-\lambda)\alpha_k\sqrt{2\chi_k}\epsilon^{V}_{\policy}(\policyy)}{(1-\gamma)\left|1-2\gamma\lambda|\calS||\calA|\right|},
J^{c}(\pi_{\bm{\theta}_{k+1}})\leq b+\frac{\gamma(1-\lambda)\beta_k\sqrt{2\chi_k}\epsilon^{C}_{\policy}(\policyy)}{(1-\gamma)\left|1-2\gamma\lambda|\calS||\calA|\right|}.
\end{flalign}
\end{theorem}
\begin{remark}
\label{remark:Comparison-with-soda}
Let $\lambda\rightarrow0$, according to Theorem \ref{them-re-cost}, the performance and cost constraint of CUP satisfies
\begin{flalign}
\label{bound-cost-performance}
J(\pi_{\bm{\theta}_{k+1}})-J(\pi_{\bm{\theta}_{k}})\ge-\frac{\gamma\alpha_k\sqrt{2\chi_k}\epsilon^{V}_{\policy}(\policyy)}{(1-\gamma)},
J^{c}(\pi_{\bm{\theta}_{k+1}})\leq b+\frac{\gamma\beta_k\sqrt{2\chi_k}\epsilon^{C}_{\policy}(\policyy)}{(1-\gamma)}.
\end{flalign}
The bounds of CUP in (\ref{bound-cost-performance}) achieves at $\mathcal{O}(\frac{\alpha_k\gamma}{1-\gamma})$ or $\mathcal{O}(\frac{\beta_k\gamma}{1-\gamma})$,
which is more tight than the bounds of CPO \cite[Proposition 1-2]{AchiamHTA17}, PCPO \cite[Theorem 1]{yang2020projection} and FOCOPS \cite{zhang2020first} where their bounds achieve at $\mathcal{O}(\frac{\gamma}{(1-\gamma)^2})$.
\end{remark}

\textbf{Practical Implementation}

Now, we present our sample-based implementation for CUP (\ref{performance-improvement})-(\ref{projection}).
Our main idea is to estimate the objective and constraints in (\ref{performance-improvement})-(\ref{projection}) with samples collected by current policy $\pi_{\bm{\theta}_k}$, then solving its optimization problem via first-order optimizer.

We denote the empirical KL-divergence w.r.t $\policy$  and $\policyy$ as follows,
\[\hat{D}_{\text{KL}}(\policy,\policyy)=\frac{1}{T}\sum_{t=1}^{T}\text{KL}(\policy(a_t|s_t),\policyy(a_t|s_t)).\]
Let $\{(s_t,a_t,r_{t+1},c_{t+1})\}_{t=1}^{T}\sim\pi_{\bm{\theta}_k}$, we update performance improvement (\ref{performance-improvement}) step as follows,
\begin{flalign}
\nonumber
\pi_{{\bm{\theta}}_{k+\frac{1}{2}}}&=\arg\max_{\pi_{\bm{\theta}}\in\Pi_{\theta}}\left\{\hat{\mathcal{L}}_{\text{R}}(\policy,\pi_{{\bm{\theta}}_k})\right\}
\\
\nonumber
\hat{\mathcal{L}}_{\text{R}}(\policy,\pi_{{\bm{\theta}}_k})&=\frac{1}{T}\sum_{t=1}^{N}\dfrac{\pi_{\bm{\theta}}(a_t|s_t)}{\pi_{{\bm{\theta}}_k}(a_t|s_t)}\hat{A}_t-\alpha_k\sqrt{\hat{D}_{\text{KL}}(\pi_{{\bm{\theta}}_k},\pi_{{\bm{\theta}}})},
\end{flalign}
where $\hat{A}_t$ is an estimator of $A^{\text{GAE}(\gamma,\lambda)}_{\pi_{{\bm{\theta}}_k}}(s,a)$.

Then we update projection step by replacing the distance function $D$ by KL-divergence, and we solve the constraint problem (\ref{projection}) by the following primal-dual approach, 
\begin{flalign}
\nonumber
(\pi_{{\bm{\theta}}_{k+1}},\nu_{k+1})=\arg\min_{\pi_{\bm{\theta}}\in\Pi_{\bm{\theta}}}\max_{\nu\ge0}
\left\{
\hat{\mathcal{L}}_{\text{c}}\left(\policy,\pi_{{\bm{\theta}}_k},{\bm{\theta}}_{k+\frac{1}{2}}\right)
\right\}\\
\nonumber
\hat{\mathcal{L}}_{\text{c}}\left(\policy,\pi_{{\bm{\theta}}_k},{\bm{\theta}}_{k+\frac{1}{2}}\right)=
\hat{D}_{\text{KL}}(\pi_{{\bm{\theta}}_{k+\frac{1}{2}}},\pi_{{\bm{\theta}}})
+\nu \hat{C}(\policy,\pi_{{\bm{\theta}}_k}),
\end{flalign}
where the empirical constraint function
\[\hat{C}(\policy,\pi_{{\bm{\theta}}_k})=\hat{J}^{C}+\frac{1}{1-\tilde\gamma}\cdot\frac{1}{T}\sum_{t=1}^{T}\frac{\pi_{\bm{\theta}}(a_{t}|s_{t})}{\pi_{{\bm{\theta}}_k}(a_{t}|s_{t})}\hat{A}^{C}_{t}+\beta_k\sqrt{\hat{D}_{\text{KL}}(\pi_{{\bm{\theta}}_k},\pi_{{\bm{\theta}}})}-b,\] $\hat{J}^{C}$and $\hat{A}^{C}_{t}$ are estimators for cost-return and cost-advantage correspondingly.

Due to the limitation of space, we have presented all the details for the implementation of CUP in Appendix \ref{sec-app-cpu} and Algorithm \ref{alg-app-cpu}.

\section{Related Work}

This section reviews some typical ways to solve safe reinforcement learning: local policy search, Lagrangian approach, and constrained policy optimization (CPO).
We mainly focus on CPO since those algorithms also use surrogate functions to replace the objective and constraints, which resembles the proposed CUP.
We provide more comparisons and discussion in Appendix \ref{app-related-work} and Table \ref{app-table-com}.

\textbf{Local Policy Search and Lagrangian Approach}.
A direct way to solve CMDP (\ref{def:problem-setting}) is to apply \emph{local policy search} \cite{peter2008Reinforcement,2013Safepolicy} over the policy space $\Pi_{\calC}$, i.e.,
\begin{flalign}
\label{local-policy-search}
\pi_{{{\bm{\theta}}}_{k+1}}=\arg\max_{\pi_{{\bm{\theta}}}\in\Pi_{{\bm{\theta}}}}J(\pi_{{\bm{\theta}}}),~\text{s.t}.~J^{c}(\pi_{{\bm{\theta}}})\leq b,~\text{and}~D(\pi_{{\bm{\theta}}},\pi_{{{\bm{\theta}}}_{k}})<\delta,
\end{flalign}
where $\delta$ is a positive scalar, $D(\cdot,\cdot)$ is some distance measure.
For practice, the local policy search (\ref{local-policy-search}) is challenging to implement because it requires evaluation of the constraint function $c$ to determine whether a proposed point $\pi$ is feasible \cite{zhang2020first}. 
Besides, when updating policy according to samples, local policy search (\ref{local-policy-search}) requires off-policy evaluation \cite{AchiamHTA17}, which is very challenging for high-dimension control problem \cite{duan2016benchmarking,yang-ijcai2018-414,yang2020convergence}.
Thus, local policy search (\ref{local-policy-search}) looks simple, but it is impractical for high-dimension policy optimization.

The standard way to solve CMDP (\ref{def:problem-setting}) is Lagrangian approach \cite{chow2017risk} that is also known as primal-dual policy optimization \cite{chen2021primal}:
\begin{flalign}
\label{min-max-search}
(\pi_{\star},\lambda_{\star})=\arg\min_{\lambda\ge0}\max_{\policy\in\Pi_{\bm{\theta}}}
\left\{
J(\policy)-\lambda(J^{c}(\policy)-b)
\right\}.
\end{flalign}
Although extensive canonical algorithms are proposed to solve problem (\ref{min-max-search}), e.g., \cite{liang2018accelerated,tessler2019reward,paternain2019constrained,le2019batch,russel2020robust,xu2020primal,satija2020constrained,chen2021primal},
the policy updated by Lagrangian approach may be infeasible w.r.t. CMDP (\ref{def:problem-setting}).
This is hazardous in reinforcement learning when one needs to execute the intermediate policy (which may be unsafe) during training \cite{chow2018lyapunov}.

\textbf{Constrained Policy Optimization (CPO)}.
Recently, CPO \cite{AchiamHTA17} suggests to replace the cost constraint with a surrogate cost function which evaluates the constraint $J^{c}(\pi_{{\bm{\theta}}})$ according to the samples collected from the current policy $\pi_{{\bm{\theta}}_k}$:
\begin{flalign}
\label{cpo-objective}
&\pi_{{\bm{\theta}}_{k+1}}=\arg\max_{\pi_{{\bm{\theta}}}\in\Pi_{{\bm{\theta}}}}~~~~\E_{s\sim d^{\rho_0}_{\pi_{{\bm{\theta}}_k}}(\cdot),a\sim\pi_{{\bm{\theta}}}(\cdot|s)}\left[A_{\pi_{{\bm{\theta}}_k}}(s,a)\right]\\
\label{cost-constraint}
&\text{s.t.}~~J^{c}(\pi_{{\bm{\theta}}_k})+\frac{1}{1-\gamma}\E_{s\sim d^{\rho_0}_{\pi_{{\bm{\theta}}_k}}(\cdot),a\sim\pi_{{\bm{\theta}}}(\cdot|s)}\left[A^{c}_{\pi_{{\bm{\theta}}_k}}(s,a)\right]\leq b,\\
\label{trust-region}
&\bar{D}_{\text{KL}}(\pi_{{\bm{\theta}}},\pi_{{\bm{\theta}}_k})=\E_{s\sim d^{\rho_0}_{\pi_{{\bm{\theta}}_k}}(\cdot)}[\text{KL}(\pi_{{\bm{\theta}}},\pi_{{\bm{\theta}}_k})[s]]\leq\delta.
\end{flalign}
Existing recent works (e.g., \cite{AchiamHTA17,vuong2019supervised,yang2020projection,hanreinforcementl2020,ijcai2020-632,bharadhwaj2021conservative}) try to find some convex approximations to replace the term $A_{\pi_{{\bm{\theta}}_k}}(s,a)$ and $\bar{D}_{\text{KL}}(\pi_{{\bm{\theta}}},\pi_{{\bm{\theta}}_k})$ Eq.(\ref{cpo-objective})-(\ref{trust-region}).
Concretely, according to (\ref{performance-difference-2002}), \cite{AchiamHTA17} suggest to use first-order Taylor expansion to replace (\ref{cpo-objective})-(\ref{cost-constraint}), use second-oder approximation to replace (\ref{trust-region}).
Such first-order and second-order approximations turn a non-convex problem (\ref{cpo-objective})-(\ref{trust-region}) to be a convex problem, 
it seems to make a simple solution, but this approach results in many error sources and troubles in practice.
Firstly, it still lacks a theory analysis to show the difference between the non-convex problem (\ref{cpo-objective})-(\ref{trust-region}) and its convex approximation.
Policy optimization is a typical non-convex problem \cite{yang2021sample}; its convex approximation may introduce some error for its original issue.
Secondly, CPO updates parameters according to conjugate gradient \cite{suli2003introduction}, and its solution involves the inverse Fisher information matrix,
which requires expensive computation for each update.
Later, \cite{yang2020projection} propose projected-based constrained policy optimization (PCPO) that also uses second-order approximation, which also results in an expensive computation.

Instead of using a convex approximation for the objective function, the proposed CUP algorithm improves CPO and PCPO at least two aspects.
Firstly, the CUP directly optimizes the surrogate objective function via the first-order method, and it does not depend on any convex approximation.
Thus, the CUP effectively avoids the expensive computation for the inverse Fisher information matrix.
Secondly, CUP extends the surrogate objective function to GAE. Although \cite{zhang2020first} has used the GAE technique in experiments, to the best of our knowledge, it still lacks a rigorous theoretical analysis involved GAE before we propose CUP.

\begin{figure*}[t]
    \centering
    \includegraphics[width=17cm, height=5.5cm]{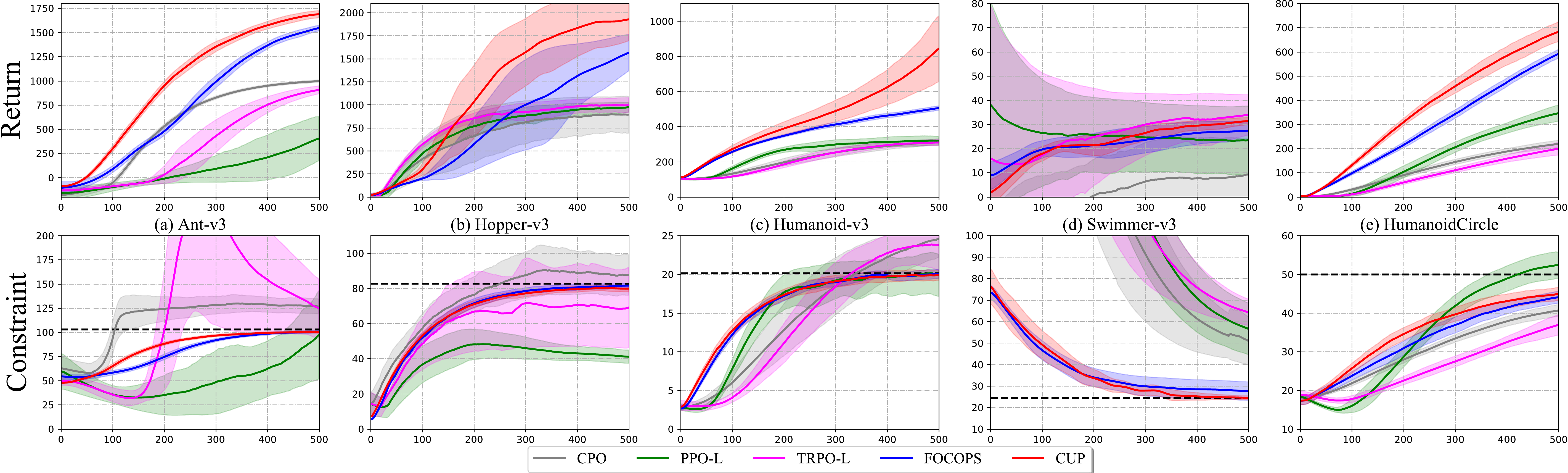}
    \caption{Performance for CPO, PPO-L, TRPO-L, FOCOPS and CUP over 10 seeds. The first row shows the learning curves of objective return, and the second row shows the curves of constraint return. The x-axis is training iteration. CUP quickly stabilizes the constraint return around the limit value while converging the objective return to higher values faster.} 
    \label{fig:comparison}
\end{figure*}
\begin{table*}[t]
  \centering
  
  \vskip 0.1in
  \begin{adjustbox}{width=0.95\textwidth}
  \begin{tabular}{*7c}
  \toprule
  Environment         & {}              & CPO & TRPO-L & PPO-L &FOCOPS & \textbf{CUP} \\
  \midrule
  Ant-v3                   &Return     & $1030.17\pm8.15$ & $480.86\pm161.05$ & $1012.02\pm17.26$ & $1662.53\pm17.40$ & $\boldsymbol{1743.66\pm40.5}$ \\
  (103.12)                &Constraint          & ${\color{black}120.76\pm4.80}$ & ${\color{black}131.07\pm67.9}$ & ${\color{black}112.45\pm15.48}$& $101.31\pm0.41$ & $99.11\pm0.93$ \\
  \midrule
  Hopper-v3             &Return     & $875.89\pm285.17$ & $1025.49\pm10.68$ & $1010.2\pm61.48$& $1687.72\pm24.38$ & $\boldsymbol{2025.56\pm122.35}$ \\
  (82.75)                  &Constraint          & $76.6\pm10.62$ & $40.36\pm4.75$ & ${\color{black}83.28\pm31.19}$ & ${\color{black}102.3\pm1.455}$& $79.98\pm2.306$ \\
  \midrule
   Swimmer-v3         &Return    & $18.77\pm6.56$ & $27.35\pm10.07$ & $35.58\pm5.68$& $28.15\pm4.30$ & $\boldsymbol{33.38\pm0.54}$ \\
  (24.52)                  &Constraint          & ${\color{black}42.07\pm3.31}$ & ${\color{black}49.58\pm7.46}$ & ${\color{black}{54.91\pm3.93}}$ & ${\color{black}26.54\pm4.16}$& $23.31\pm0.052$ \\
  \midrule
  Humanoid-v0        &Return     & $326.95\pm16.00$ & $307.71\pm24.71$& $322.11\pm25.54$ & $542.5\pm4.76$ & $\boldsymbol{1066.83\pm266.12}$ \\
  (20.14)                  &Constraint          & ${\color{black}26.13\pm2.13}$ & $18.22\pm3.04$ & ${\color{black}22.94\pm4.54}$ & $20.04\pm0.19$& $19.91\pm0.36$  \\
    \midrule
  Humanoid-Circle   &Return     & $237.54\pm23.20$ & $384.45\pm47.66$ & $243.35\pm37.90$& $713.04\pm9.25$ & $\boldsymbol{768.65\pm63.70}$ \\
  (50.00)                   &Constraint          & $43.64\pm1.91$ & ${\color{black}53.77\pm1.48}$ & $41.17\pm3.98$ & $47.73\pm0.64$& $48.23\pm0.65$ \\
  \bottomrule
  \end{tabular}
  \end{adjustbox}
  
    \caption{Average results for CPO, PPO-L, TRPO-L, FOCOPS and CUP over 10 seeds after 500 iterations. The agent interacts with the environment 5000 times per iteration. Constraint limit are in brackets under the environment names.}
  \label{tab:mujoco}
\end{table*}

\section{Experiment}
In this section, we verify the effectiveness and stability of CUP in terms of policy improvement while satisfying safety.
We aim to answer the following three issues:

\textbf{(I)}. Does CUP satisfy the safety constraints in different environments? For the same environment with different cost limit, Does CUP also performs well?
   
\textbf{(II)}. How does CUP compare to the state-of-the-art safe RL algorithms? Does CUP achieve higher rewards under the constraint of cost threshold?

\textbf{(III)}. Does CUP play a sensibility during the hyper-parameters (e.g., step-size $\nu$, the coefficient $\alpha$ with respect KL-divergence) tuning processing?

\textbf{Environments}: We train different robotic agents using five MuJoCo physical simulators \cite{todorov2012mujoco} which are open by OpenAI Gym API \cite{brockman2016openai}. 

 \begin{figure*}[t!]
  \centering
  \subfigure{\includegraphics[width=16.5cm,height=2.5cm]{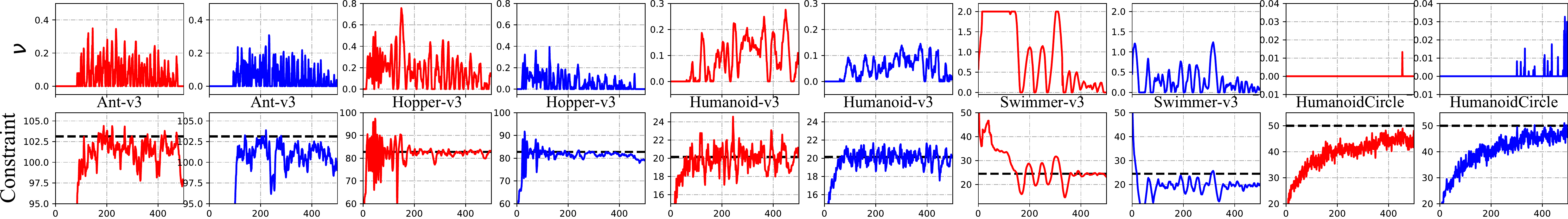}} 
  \caption{Cost constraint with respect to hyper-parameter $\nu$ (defined in Projection step).}
  \label{differntlim-01}
\end{figure*}

\textbf{Baseline Algorithms}: Baselines includes CPO\cite{AchiamHTA17}, TRPO Lagrangian (TRPO-L), PPO Lagrangian (PPO-L) and FOCOPS\cite{zhang2020first}. 
TRPO-L and PPO-L are improved by \cite{chow2018lyapunov,Ray2019}, which are based on TRPO \cite{schulman2015trust} and PPO \cite{schulman2017proximal}. 
These two algorithms use the Lagrangian method \cite{bertsekas1997nonlinear}, which applies adaptive penalty coefficients to satisfy the constraint.

\subsection{Evaluation CUP and Comparison Analysis}

We have shown the Learning curves for CUP, and other baselines in Figure \ref{fig:comparison}, and Table \ref{tab:mujoco} summarizes the performance of all algorithms.
We find that CUP and FOCOPS successfully enforce the constraints in all experiments.
In most cases, the traces of their constraint almost coincide with the dashed black line of the limit. 
By contrast, the others frequently suffer from over or under the correction. 
Although it seems that being below the limit also satisfies the constraint, this usually results in a poor return.
The initial policy is not guaranteed to be feasible, such as in the Swimmer-v3 environment. 
We observed that CUP brings the policy back to the feasible range faster than other baselines.
From Table \ref{tab:mujoco}, we know although PPO-L achieves a reward of $35.58\pm 5.68$ outperforms CUP in Swimmer-v3, PPO-L obtain a cost with $54.91 \pm 3.93$ that violates the cost limit of $24.52$ significantly, which implies PPO-L learns a dangerous policy under this setting.

On the other hand, Figure \ref{fig:comparison} shows that CUP generally gains higher returns than different baselines while enforcing the cost constraint. 

In contrast, after equal iterations, CUP has a greater speed of stabilizing the constraint return around to the limit value and is quicker to find feasible policies to gain a more significant objective return.

\subsection{Sensitivity Analysis for Hyper-Parameters Tuning}
Hyper-parameters tuning is necessary to achieve efficient and stable policy improvement and enforce constraints. 
Now, we investigate the performance with respect to the parameters: $\nu$, step-size $\alpha$, and cost limit $b$.
In Figure \ref{differntlim-01}-\ref{differntlim-03}, after exploring the influence of hyper-parameters introduced by CUP methods under different experimental settings, we discover that the performance of CUP is robust to hyper-parameters tuning.

We have visualized the changes of cost constraint with respect to $\nu$ over different MuJoCo in Figure \ref{differntlim-01}.
From Figure \ref{differntlim-01} we know if the estimated cost under the target threshold $b$, then $\nu$ keeps calm, which implies $\nu$ is not activated.
Such an empirical phenomenon gives significant expression on the Humanoid environment.
While if the estimated cost exceeds the target threshold $b$, $\nu$ will be activated, which requires the agent to play a policy on the safe region.
Those empirical result shown in Figure \ref{differntlim-01} is consistent with the update rule of $\nu$: $\nu_{k+1}=\{\nu_{k}+\eta(\hat{J}_{k}^{C}-b)\}_{+}$, which implies the projection of CUP plays an important role for the agent to learning a safe policy.

Furthermore, we investigate the influence of reward performance and cost constraint with respect to the penalty factor $\alpha$, see Figure \ref{differntlim-02}.
Results show that the performance of CUP is still very stable for different settings of $\alpha$, where we run $\alpha$ among $\{0.1,0.15,0.2,0.25,0.3\}$.
Additionally, the constraint value of CUP also still fluctuates around the target value. The different value achieved by CUP in different setting $\alpha$ is affected by the simulated environment and constraint thresholds, which are easy to control.

Finally, we compare different cost limit $b$ to verify the sensitiveness of CUP. We have shown the results in Figure \ref{differntlim-03}.
We compare policy performance and cost under different cost limit settings. For example, in the Swimmer-v3, we set cost limit $b$ among $\{10,15,20,25\}$.
As illustrated in Figure \ref{differntlim-03}. we demonstrate that the CUP algorithm continuously improves the policy under the constrain.
This experiment implies CUP is robust to different cost limit settings, and CUP is scalable to various safe RL tasks.

 \begin{figure}[t!]
  \centering
  \subfigure{\includegraphics[width=16.5cm,height=4.5cm]{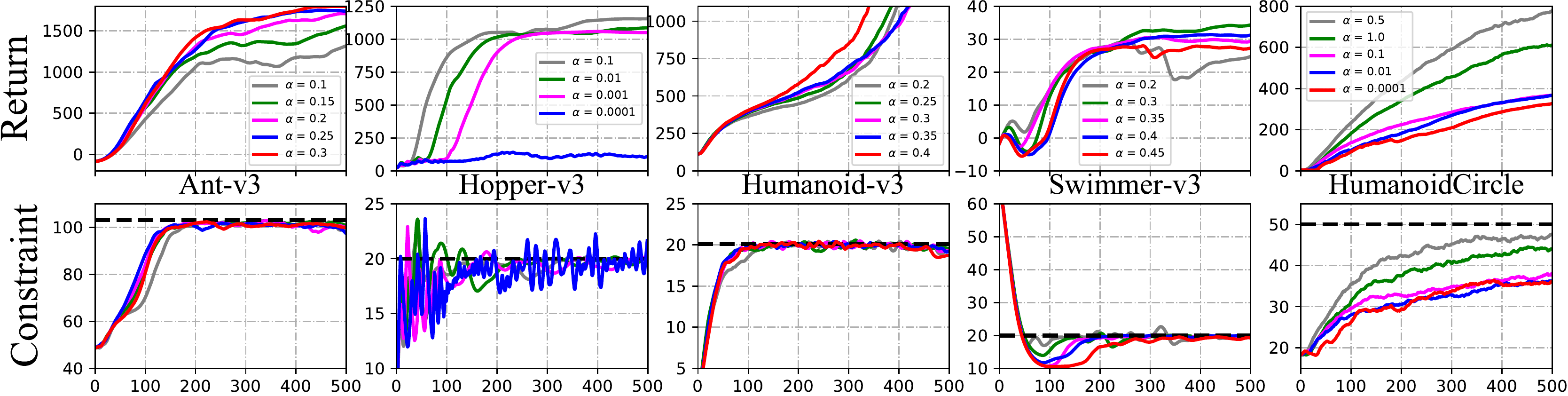}} 
  \caption{Performance with respect to penalty factor $\alpha$.}
  \label{differntlim-02}
\end{figure}
\begin{figure}[t!]
  \centering
  \subfigure{\includegraphics[width=16.5cm,height=4.5cm]{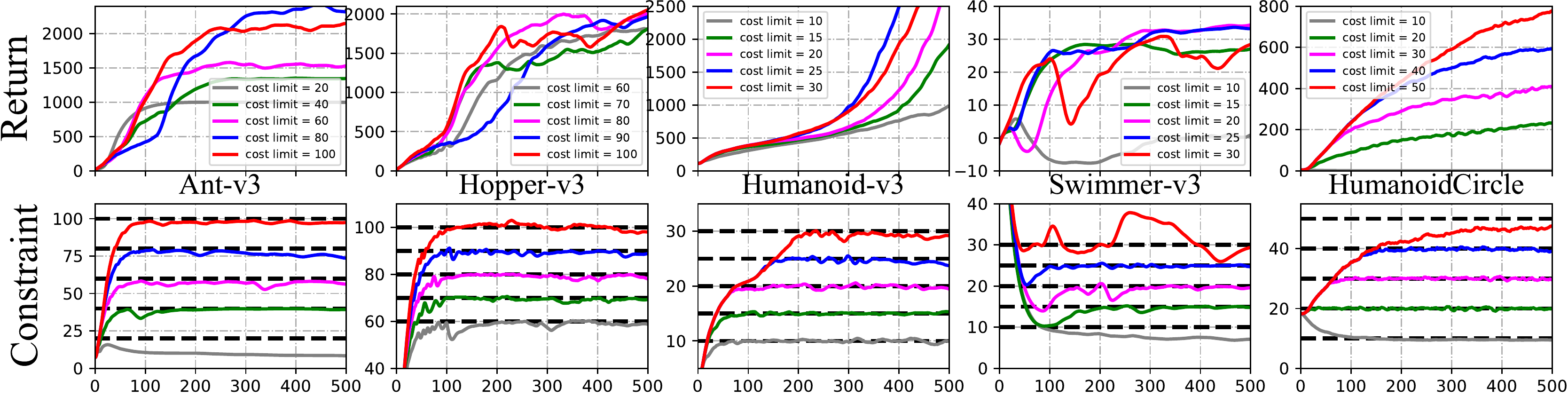}} 
  \caption{Performance with respect to cost limit.}
  \label{differntlim-03}
\end{figure}

\section{Conclusion}
This paper proposes the CUP algorithm with a theoretical safety guarantee.
We derive the CUP based on some new proposed surrogate functions with respect to objective and constraints, and the practical implementation of CUP does not depend on any convex approximation.
Extensive experiments on continuous control tasks show the effectiveness of CUP where the agent satisfies safe constraints.

\clearpage

\bibliographystyle{apalike}
\bibliography{reference}

\begin{thebibliography}{}

\bibitem[Achiam et~al., 2017]{AchiamHTA17}
Achiam, J., Held, D., Tamar, A., and Abbeel, P. (2017).
\newblock Constrained policy optimization.
\newblock In {\em Proceedings of International Conference on Machine Learning
  (ICML)}, volume~70, pages 22--31.

\bibitem[Afsar et~al., 2021]{afsar2021reinforcement}
Afsar, M.~M., Crump, T., and Far, B. (2021).
\newblock Reinforcement learning based recommender systems: A survey.
\newblock {\em arXiv preprint arXiv:2101.06286}.

\bibitem[Altman, 1999]{altman1999constrained}
Altman, E. (1999).
\newblock {\em Constrained Markov decision processes}.
\newblock CRC Press.

\bibitem[Bellman, 1957]{bellman1957markovian}
Bellman, R. (1957).
\newblock A markovian decision process.
\newblock {\em Journal of mathematics and mechanics}, 6(5):679--684.

\bibitem[Bharadhwaj et~al., 2021]{bharadhwaj2021conservative}
Bharadhwaj, H., Kumar, A., Rhinehart, N., Levine, S., Shkurti, F., and Garg, A.
  (2021).
\newblock Conservative safety critics for exploration.
\newblock In {\em International Conference on Learning Representations (ICLR)}.

\bibitem[Bisi et~al., 2020]{ijcai2020-632}
Bisi, L., Sabbioni, L., Vittori, E., Papini, M., and Restelli, M. (2020).
\newblock Risk-averse trust region optimization for reward-volatility
  reduction.
\newblock In Bessiere, C., editor, {\em Proceedings of the Twenty-Ninth
  International Joint Conference on Artificial Intelligence, {IJCAI-20}}, pages
  4583--4589.

\bibitem[Brockman et~al., 2016]{brockman2016openai}
Brockman, G., Cheung, V., Pettersson, L., Schneider, J., Schulman, J., Tang,
  J., and Zaremba, W. (2016).
\newblock Openai gym.
\newblock {\em arXiv preprint arXiv:1606.01540}.

\bibitem[Chen et~al., 2021]{chen2021primal}
Chen, Y., Dong, J., and Wang, Z. (2021).
\newblock A primal-dual approach to constrained markov decision processes.
\newblock {\em arXiv preprint arXiv:2101.10895}.

\bibitem[Chow et~al., 2017]{chow2017risk}
Chow, Y., Ghavamzadeh, M., Janson, L., and Pavone, M. (2017).
\newblock Risk-constrained reinforcement learning with percentile risk
  criteria.
\newblock {\em The Journal of Machine Learning Research}, 18(1):6070--6120.

\bibitem[Chow et~al., 2018]{chow2018lyapunov}
Chow, Y., Nachum, O., Duenez-Guzman, E., and Ghavamzadeh, M. (2018).
\newblock A lyapunov-based approach to safe reinforcement learning.
\newblock In {\em Advances in Neural Information Processing Systems (NeurIPS)}.

\bibitem[Csisz{\'a}r and K{\"o}rner, 2011]{csiszar2011information}
Csisz{\'a}r, I. and K{\"o}rner, J. (2011).
\newblock {\em Information theory: coding theorems for discrete memoryless
  systems}.
\newblock Cambridge University Press.

\bibitem[Deisenroth et~al., 2013]{deisenroth2013survey}
Deisenroth, M.~P., Neumann, G., and Peters, J. (2013).
\newblock A survey on policy search for robotics.
\newblock {\em Foundations and Trends{\textregistered} in Machine Learning}.

\bibitem[Duan et~al., 2016]{duan2016benchmarking}
Duan, Y., Chen, X., Houthooft, R., Schulman, J., and Abbeel, P. (2016).
\newblock Benchmarking deep reinforcement learning for continuous control.
\newblock In {\em International Conference on Machine Learning (ICML)}, pages
  1329--1338.

\bibitem[Greensmith et~al., 2004]{greensmith2004variance}
Greensmith, E., Bartlett, P.~L., and Baxter, J. (2004).
\newblock Variance reduction techniques for gradient estimates in reinforcement
  learning.
\newblock {\em Journal of Machine Learning Research (JMLR)}, 5(Nov):1471--1530.

\bibitem[Han et~al., 2020]{hanreinforcementl2020}
Han, M., Tian, Yuanand~Zhang, L., Wang, J., and Pan, W. (2020).
\newblock Reinforcement learning control of constrained dynamic systems with
  uniformly ultimate boundedness stability guarantee.
\newblock {\em arXiv preprint arXiv:2011.06882}.

\bibitem[Kakade and Langford, 2002]{kakade2002approximately}
Kakade, S. and Langford, J. (2002).
\newblock Approximately optimal approximate reinforcement learning.
\newblock In {\em Proceedings of International Conference on Machine Learning
  (ICML)}, volume~2, pages 267--274.

\bibitem[Kang et~al., 2021]{kang2021learning}
Kang, B., Mannor, S., and Feng, J. (2021).
\newblock Learning safe policies with cost-sensitive advantage estimation.
\newblock \url{https://openreview.net/forum?id=uVnhiRaW3J}.

\bibitem[Koller et~al., 2018]{koller2018learning}
Koller, T., Berkenkamp, F., Turchetta, M., and Krause, A. (2018).
\newblock Learning-based model predictive control for safe exploration.
\newblock In {\em Conference on Decision and Control (CDC)}, pages 6059--6066.
  IEEE.

\bibitem[Le et~al., 2019]{le2019batch}
Le, H., Voloshin, C., and Yue, Y. (2019).
\newblock Batch policy learning under constraints.
\newblock In {\em International Conference on Machine Learning (ICML)}, pages
  3703--3712.

\bibitem[Liang et~al., 2018]{liang2018accelerated}
Liang, Q., Que, F., and Modiano, E. (2018).
\newblock Accelerated primal-dual policy optimization for safe reinforcement
  learning.
\newblock {\em arXiv preprint arXiv:1802.06480}.

\bibitem[Mnih et~al., 2015]{mnih2015human}
Mnih, V., Kavukcuoglu, K., Silver, D., Rusu, A.~A., Veness, J., Bellemare,
  M.~G., Graves, A., Riedmiller, M., Fidjeland, A.~K., Ostrovski, G., et~al.
  (2015).
\newblock Human-level control through deep reinforcement learning.
\newblock {\em Nature}, 518(7540):529.

\bibitem[OpenAI, 2019]{openaifive2019}
OpenAI (2019).
\newblock Openai five defeats dota 2 world champions.
\newblock
  \url{https://openai.com/blog/openai-five-defeats-dota-2-world-champions/}.

\bibitem[Paternain et~al., 2019]{paternain2019constrained}
Paternain, S., Chamon, L.~F., Calvo-Fullana, M., and Ribeiro, A. (2019).
\newblock Constrained reinforcement learning has zero duality gap.
\newblock In {\em Advances in Neural Information Processing Systems (NeurIPS)}.

\bibitem[Peters and Schaal, 2008]{peter2008Reinforcement}
Peters, J. and Schaal, S. (2008).
\newblock Reinforcement learning of motor skills with policy gradients.
\newblock {\em Neural Netw}, 21(4):682--697.

\bibitem[Pirotta et~al., 2013]{2013Safepolicy}
Pirotta, M., Restelli, M., Pecorino, A., and Calandriello, D. (2013).
\newblock Safe policy iteration.
\newblock In {\em International Conference on Machine Learning (ICML)}, pages
  307--315.

\bibitem[Puterman, 2014]{puterman2014markov}
Puterman, M.~L. (2014).
\newblock {\em Markov decision processes: discrete stochastic dynamic
  programming}.
\newblock John Wiley \& Sons.

\bibitem[Ray et~al., 2019]{Ray2019}
Ray, A., Achiam, J., and Amodei, D. (2019).
\newblock {Benchmarking Safe Exploration in Deep Reinforcement Learning}.

\bibitem[Russel et~al., 2020]{russel2020robust}
Russel, R.~H., Benosman, M., and Van~Baar, J. (2020).
\newblock Robust constrained-mdps: Soft-constrained robust policy optimization
  under model uncertainty.
\newblock {\em arXiv preprint arXiv:2010.04870}.

\bibitem[Satija et~al., 2020]{satija2020constrained}
Satija, H., Amortila, P., and Pineau, J. (2020).
\newblock Constrained markov decision processes via backward value functions.
\newblock In {\em International Conference on Machine Learning (ICML)}, pages
  8502--8511.

\bibitem[Schulman et~al., 2015]{schulman2015trust}
Schulman, J., Levine, S., Abbeel, P., Jordan, M., and Moritz, P. (2015).
\newblock Trust region policy optimization.
\newblock In {\em International Conference on Machine Learning (ICML)}, pages
  1889--1897.

\bibitem[Schulman et~al., 2016]{schulman2016high}
Schulman, J., Moritz, P., Levine, S., Jordan, M., and Abbeel, P. (2016).
\newblock High-dimensional continuous control using generalized advantage
  estimation.
\newblock {\em International Conference on Learning Representations (ICLR)}.

\bibitem[Schulman et~al., 2017]{schulman2017proximal}
Schulman, J., Wolski, F., Dhariwal, P., Radford, A., and Klimov, O. (2017).
\newblock Proximal policy optimization algorithms.
\newblock {\em arXiv preprint arXiv:1707.06347}.

\bibitem[Silver et~al., 2016]{silver2016mastering}
Silver, D., Huang, A., Maddison, C.~J., Guez, A., Sifre, L., Van Den~Driessche,
  G., Schrittwieser, J., Antonoglou, I., Panneershelvam, V., Lanctot, M.,
  et~al. (2016).
\newblock Mastering the game of go with deep neural networks and tree search.
\newblock {\em nature}, 529(7587):484.

\bibitem[Silver et~al., 2017]{silver2017mastering}
Silver, D., Schrittwieser, J., Simonyan, K., Antonoglou, I., Huang, A., Guez,
  A., Hubert, T., Baker, L., Lai, M., Bolton, A., et~al. (2017).
\newblock Mastering the game of go without human knowledge.
\newblock {\em Nature}, 550(7676):354.

\bibitem[S{\"u}li and Mayers, 2003]{suli2003introduction}
S{\"u}li, E. and Mayers, D.~F. (2003).
\newblock {\em An introduction to numerical analysis}.
\newblock Cambridge university press.

\bibitem[Sutton and Barto, 1998]{sutton1998reinforcement}
Sutton, R.~S. and Barto, A.~G. (1998).
\newblock {\em Reinforcement learning: An introduction}.
\newblock MIT press.

\bibitem[Sutton et~al., 2000]{sutton2000policy}
Sutton, R.~S., McAllester, D.~A., Singh, S.~P., and Mansour, Y. (2000).
\newblock Policy gradient methods for reinforcement learning with function
  approximation.
\newblock In {\em Advances in Neural Information Processing Systems (NeurIPS)},
  pages 1057--1063.

\bibitem[Tessler et~al., 2019]{tessler2019reward}
Tessler, C., Mankowitz, D.~J., and Mannor, S. (2019).
\newblock Reward constrained policy optimization.
\newblock {\em International Conference on Learning Representation (ICLR)}.

\bibitem[Todorov et~al., 2012]{todorov2012mujoco}
Todorov, E., Erez, T., and Tassa, Y. (2012).
\newblock Mujoco: A physics engine for model-based control.
\newblock In {\em 2012 IEEE/RSJ International Conference on Intelligent Robots
  and Systems}, pages 5026--5033. IEEE.

\bibitem[Vinyals et~al., 2019]{vinyals2019alphastar}
Vinyals, O., Babuschkin, I., Chung, J., Mathieu, M., Jaderberg, M., Czarnecki,
  W.~M., Dudzik, A., Huang, A., Georgiev, P., Powell, R., et~al. (2019).
\newblock Alphastar: Mastering the real-time strategy game starcraft ii.
\newblock {\em DeepMind blog}, 2.

\bibitem[Vuong et~al., 2019]{vuong2019supervised}
Vuong, Q., Zhang, Y., and Ross, K.~W. (2019).
\newblock Supervised policy update for deep reinforcement learning.
\newblock In {\em International Conference on Learning Representation (ICLR)}.

\bibitem[Watkins, 1989]{watkins1989learning}
Watkins, C. J. C.~H. (1989).
\newblock Learning from delayed rewards.

\bibitem[Williams, 1992]{williams1992simple}
Williams, R.~J. (1992).
\newblock Simple statistical gradient-following algorithms for connectionist
  reinforcement learning.
\newblock {\em Machine learning}, 8(3-4):229--256.

\bibitem[Wu et~al., 2018]{wu2018variance}
Wu, C., Rajeswaran, A., Duan, Y., Kumar, V., Bayen, A.~M., Kakade, S.,
  Mordatch, I., and Abbeel, P. (2018).
\newblock Variance reduction for policy gradient with action-dependent
  factorized baselines.
\newblock {\em International Conference on Learning Representation (ICLR)}.

\bibitem[Xu et~al., 2020]{xu2020primal}
Xu, T., Liang, Y., and Lan, G. (2020).
\newblock A primal approach to constrained policy optimization: Global
  optimality and finite-time analysis.
\newblock {\em arXiv preprint arXiv:2011.05869}.

\bibitem[Yang et~al., 2018]{yang-ijcai2018-414}
Yang, L., Shi, M., Zheng, Q., Meng, W., and Pan, G. (2018).
\newblock A unified approach for multi-step temporal-difference learning with
  eligibility traces in reinforcement learning.
\newblock In {\em Proceedings of the Twenty-Seventh International Joint
  Conference on Artificial Intelligence, {IJCAI-18}}, pages 2984--2990.

\bibitem[Yang et~al., 2021a]{yang2020convergence}
Yang, L., Zheng, G., Zhang, Y., Zheng, Q., Li, P., and Pan, G. (2021a).
\newblock On convergence of gradient expected sarsa ($\lambda$).
\newblock In {\em AAAI}.

\bibitem[Yang et~al., 2021b]{yang2021sample}
Yang, L., Zheng, Q., and Pan, G. (2021b).
\newblock Sample complexity of policy gradient finding second-order stationary
  points.
\newblock In {\em AAAI}.

\bibitem[Yang et~al., 2020a]{yang2020accelerating}
Yang, T.-Y., Rosca, J., Narasimhan, K., and Ramadge, P.~J. (2020a).
\newblock Accelerating safe reinforcement learning with constraint-mismatched
  policies.
\newblock {\em arXiv preprint arXiv:2006.11645}.

\bibitem[Yang et~al., 2020b]{yang2020projection}
Yang, T.-Y., Rosca, J., Narasimhan, K., and Ramadge, P.~J. (2020b).
\newblock Projection-based constrained policy optimization.
\newblock In {\em International Conference on Learning Representation (ICLR)}.

\bibitem[Zanger et~al., 2021]{zanger2021safe}
Zanger, M.~A., Daaboul, K., and Z{\"o}llner, J.~M. (2021).
\newblock Safe continuous control with constrained model-based policy
  optimization.
\newblock {\em arXiv preprint arXiv:2104.06922}.

\bibitem[Zhang et~al., 2020]{zhang2020first}
Zhang, Y., Vuong, Q., and Ross, K. (2020).
\newblock First order constrained optimization in policy space.
\newblock In {\em Advances in Neural Information Processing Systems (NeurIPS)},
  volume~33.

\end{thebibliography}

\clearpage

\appendix

\section{Key Notations}

\subsection{Matrix Index}
In this paper,
we use a bold capital letter to denote matrix, e.g., $\bA=(a_{i,j})\in\R^{m\times n}$, and its $(i,j)$-th element denoted as $\bA[i,j]=:a_{i,j},$ where $1\leq i\leq m,1\leq j\leq n$.
Similarly, a bold  lowercase letter denotes a vector, e.g., $\ba=(a_1,a_2,\cdots,a_n)\in\R^{n}$, and  its $i$-th element denoted as $\ba[i]=:a_{i},$ where $1\leq i\leq n$.

\subsection{Key Notations of Reinforcement Learning}

For convenience of reference, we list key notations that have be used in this paper.

\subsection{Value Function and Dynamic System of MDP.}
\begin{tabular}{r c p{13cm}}
   \hline
   $\mathbf{r}_{\pi_{\bm \theta}},~R_{\policy}(s),$ &: & $\mathbf{r}_{\pi_{\bm \theta}}\in\R^{|\calS|}$ is the expected vector reward according to $\pi_{\bm \theta}$, i.e., their components are:
     \vspace{-5pt}
    \[\mathbf{r}_{\pi_{\bm \theta}}[s] =\sum_{a\in\mathcal{A}}\sum_{s^{'}\in\mathcal{S}}\pi_{\bm{\theta}}(a|s)r(s'|s,a)=:R_{\policy}(s),~s\in\calS.\]\\
     \vspace{-15pt}
   $\bv_{\pi_{\bm{\theta}}},~V_{\policy}(s),$ &: & $\bv_{\pi_{\bm{\theta}}}\in\R^{|\calS|}$ is the vector that stores all the state value functions, and its components are:
     \vspace{-10pt}
\[
\bv_{\pi_{\bm{\theta}}}[s]=V_{\pi_{\bm{\theta}}}(s),~s\in\calS.
\]
\\
 \hline
$\rho(\cdot),\bm{\rho}$ &: & $\rho(s)$: the initial state distribution of state $s$; $\bm{\rho}\in\R^{|\calS|}$, and $\bm{\rho}[s]=\rho(s)$. \\
   \hline
        $\mathbf{P}_{\pi_{\bm \theta}}$ &: & Single-step state transition matrix by executing $\policy$. \\
        $\Pro_{\policy}(s^{'}|s)$ &: & Single-step state transition probability from $s$ to $s^{'}$ by executing $\policy$, and it is the $(s,s^{'})$-th component of the matrix $\mathbf{P}_{\pi_{\bm \theta}}$, i.e., $\mathbf{P}_{\pi_{\bm \theta}}[s,s^{'}]=\Pro_{\policy}(s^{'}|s)$.\\
        $\mathbb{P}_{\policy}(s_t=s^{'}|s)$&: & The probability of visiting the state $s^{'}$ after $t$
time steps from the state $s$ by executing $\policy$, and it is the $(s,s^{'})$-th component of the matrix $\mathbf{P}_{\pi_{\bm \theta}}$, i.e., $\mathbf{P}^{t}_{\pi_{\bm \theta}}[s,s^{'}]=\mathbb{P}_{\policy}(s_t=s^{'}|s)$.\\
  \hline
$d_{\pi_{\bm {\theta}}}^{s_0}(s),~d_{\pi_{\bm {\theta}}}^{\rho_0}(s)$  &: & The normalized discounted distribution of the future state $s$ starting at $s_0$ by executing $\pi_{\bm {\theta}}$:
\[d_{\pi_{\bm {\theta}}}^{s_0}(s)=:(1-\gamma)\sum_{t=0}^{\infty}\gamma^{t}\mathbb{P}_{\pi_{\bm {\theta}}}(s_t=s|s_0).\] Since $s_0\sim\rho(\cdot)$, we define $d_{\pi_{\bm {\theta}}}^{\rho_0}(s)=:\E_{s_0\sim\rho(\cdot)}[d_{\pi_{\bm {\theta}}}^{s_0}(s)]$.\\
$\bd_{\pi_{\bm {\theta}}}^{\rho_0}$&: &  It stores all the normalized discounted state distributions $d_{\pi_{\bm {\theta}}}^{\rho_0}(s)$, $\in\calS$, i.e., $\bd_{\pi_{\bm {\theta}}}^{\rho_0}\in\R^{|\calS|}$, and its components are:
$
\bd_{\pi_{\bm {\theta}}}^{\rho_0}[s]=d_{\pi_{\bm {\theta}}}^{\rho_0}(s).
$\\
   \hline
\end{tabular}

\subsection{Extend them to $\lambda$-version. }

\begin{tabular}{r c p{13cm}}
 \hline
    $\mathbf{P}^{(\lambda)}_{\pi_{\bm \theta}}$&: &$\mathbf{P}^{(\lambda)}_{\pi_{\bm \theta}}=(1-\gamma\lambda)\sum_{{t}=0}^{\infty}(\gamma\lambda)^{{t}}\bP^{{t}+1}_{\policy}.$
    \\
    $\Pro_{\policy}^{(\lambda)}(s^{'}|s)$&: &$\Pro_{\policy}^{(\lambda)}(s^{'}|s)=:\mathbf{P}^{(\lambda)}_{\pi_{\bm \theta}}[s,s^{'}]=(1-\gamma\lambda)\sum_{{t}=0}^{\infty}(\gamma\lambda)^{{t}}\Pro_{\policy}(s_{t+1}=s^{'}|s)$.
    \\
       \hline
   $\mathbf{r}^{(\lambda)}_{\pi_{\bm \theta}},~R^{(\lambda)}_{\pi_{\bm \theta}}(s)$ &: &$\mathbf{r}^{(\lambda)}_{\pi_{\bm \theta}}=\sum_{{t}=0}^{\infty}(\gamma\lambda\bP_{\policy})^{{t}}\mathbf{r}_{\pi_{\bm \theta}};~ R^{(\lambda)}_{\pi_{\bm \theta}}(s)=:
  \mathbf{r}^{(\lambda)}_{\pi_{\bm \theta}}[s].$\\
     \hline
   $\tilde{\gamma}$&: &$\tilde{\gamma}=\dfrac{\gamma(1-\lambda)}{1-\gamma\lambda}$.\\
      \hline
       $d_{\pi_{\bm {\theta}}}^{s_0,\lambda}(s)$&: &$d_{\pi_{\bm {\theta}}}^{s_0,\lambda}(s)=(1-\tilde \gamma)\sum_{t=0}^{\infty}\tilde{\gamma}^{t}\mathbb{P}^{(\lambda)}_{\pi_{\bm {\theta}}}(s_t=s|s_0)$.\\
          $d_{\pi_{\bm {\theta}}}^{\lambda}(s),~\bd_{\pi_{\bm {\theta}}}^{\lambda}$&: & $d_{\pi_{\bm {\theta}}}^{\lambda}(s)=\E_{s_0\sim\rho_{0}(\cdot)}\left[d_{\pi_{\bm {\theta}}}^{s_0,\lambda}(s)\right],~\bd_{\pi_{\bm {\theta}}}^{\lambda}[s]=d_{\pi_{\bm {\theta}}}^{\lambda}(s)$.      
         \\
      \hline
\end{tabular}

\subsection{
 TD error w.r.t. any function $\varphi(\cdot)$.
 }

\begin{tabular}{r c p{13cm}}
 \hline
     $\delta_t^{\varphi}$&: &$ \delta_t^{\varphi}=r(s_{t+1}|s_t,a_t)+\gamma\varphi(s_{t+1})-\varphi(s_{t}).$
     \\
    $ \delta^{\varphi}_{\policy,t}(s)$&: &$ \delta^{\varphi}_{\policy,t}(s)=\E_{s_{t}\sim\Pro_{\policy}(\cdot|s),a_{t}\sim{\policy}(\cdot|s_t),s_{t+1}\sim\Pro(\cdot|s_t,a_t)}\left[\delta_t^{\varphi}\right]$.\\
      $\bm{\delta}^{\varphi}_{\policy,t}$&: & $\bm{\delta}^{\varphi}_{\policy,t}[s]={{\delta}}^{\varphi}_{\policy,t}(s).$
      \\
      $\Delta_{t}^{\varphi}(\policy,\policyy,s)$&: & $\E_{s_{t}\sim\Pro_{\policyy}(\cdot|s),a_{t}\sim{\policyy}(\cdot|s_t),s_{t+1}\sim\Pro(\cdot|s_t,a_t)}\left[\left(\dfrac{\policy(a_t|s_t)}{\policyy(a_t|s_t)}-1\right)\delta_t^{\varphi}\right]$.\\
      $\bm{\Delta}_{t}^{\varphi}(\policy,\policyy)$&: & $\bm{\Delta}_{t}^{\varphi}(\policy,\policyy)[s]=\Delta_{t}^{\varphi}(\policy,\policyy,s)$.
      \\
          \hline
\end{tabular}

\clearpage
\section{Practical Implementation of CUP}
\label{sec-app-cpu}

\begin{algorithm}[t]
\caption{Conservative Update Policy (CUP)} 
\label{alg-app-cpu}
\begin{algorithmic}[1]
\STATE \textbf{Initialize:} policy network parameters ${\bm{\theta}}_0$; value network parameter $\bm{\omega}_0$; cost value function parameter $\bm{\nu}_0$, step-size $\nu_{0,0}$;
\STATE \textbf{Hyper-parameters:} trajectory horizon $T$; discount rate $\gamma$; episode number $M,N$, mini-batch size $B$, positive constant $\alpha,\eta$;
\FOR{$k=0,1,2,\ldots $}
\STATE Collect batch data of $M$ episodes of horizon $T$ in $\cup_{i=1}^{M}\cup_{t=0}^{T}\left\{(s_{i,t},a_{i,t},r_{i,t+1},c_{i,t+1})\right\}$ according to current policy $\pi_{{\bm{\theta}}_k}$;
\STATE Estimate $c$-return by discount averaging on each  episode: $\hat{J}_{i}^{C}=\sum_{t=0}^{T}\gamma^{t}c_{i,t+1};$
\STATE Compute TD errors $\cup_{i=1}^{M}\cup_{t=0}^{T}\{\delta_{i,t}\}$, cost TD errors $\cup_{i=1}^{M}\cup_{t=0}^{T}\{\delta^{C}_{i,t}\}$:
\[
\delta_{i,t}=r_{i,t}+\gamma V_{\bm{\omega}_k}(s_{i,t})- V_{\bm{\omega}_k}(s_{i,t-1}),~\delta^{C}_{i,t}=c_{i,t}+\gamma V^{C}_{\bm{\nu}_k}(s_{i,t})- V^{C}_{\bm{\nu}_k}(s_{i,t-1});
\]
\STATE Compute GAE: $\cup_{i=1}^{M}\cup_{t=0}^{T}\{\hat{A}_{i,t},\hat{A}^{C}_{i,t}\}$: $\hat{A}_{i,t}=\sum_{j=t}^{T}(\gamma\lambda)^{j-t}\delta_{i,j}, ~\hat{A}^{C}_{i,t}=\sum_{j=t}^{T}(\gamma\lambda)^{j-t}\delta^{C}_{i,j};$
\STATE Compute target function for value function and cost value function as follows,
\[
V^{\text{target}}_{i,t}=\hat{A}_{i,t}+ V_{\bm{\omega}_k}(s_{i,t}),~~V^{\text{target},C}_{i,t}=\hat{A}_{i,t}^{C}+ V^{C}_{\bm{\nu}_k}(s_{i,t})
;
\]
\STATE Store data: $\calD_k=\cup_{i=1}^{M}\cup_{t=0}^{T}\left\{(a_{i,t},s_{i,t},\hat{A}_{i,t},\hat{A}^{C}_{i,t},V^{\text{target}}_{i,t},V^{\text{target},C}_{i,t})\right\}$;
\STATE\tikzmark{start1}$\pi_{\text{old}}\leftarrow\pi_{\bm{\theta}_k}$;~~~~~~~~~~~~~~~~~~~~~~~~~~~~~~~~~~~~~~~~~~~~~~~~~~~~~~~~~~~~~~~~~~~~~~~~~~~~~~{\color{blue}{\texttt{Policy Improvement}}}
\FOR{$i=0,1,2,\ldots,M$}
\STATE
\[
{\bm{\theta}}_{k+\frac{1}{2}}=\arg\max_{{\bm{\theta}}}\left\{\frac{1}{T}\sum_{t=1}^{T}\dfrac{\pi_{\bm{\theta}}(a_{i,t}|s_{i,t})}{\pi_{\text{old}}(a_{i,t}|s_{i,t})}\hat{A}_{i,t}-\alpha\sqrt{\frac{1}{T}\sum_{t=1}^{T}\text{KL}(\pi_{\text{old}}(\cdot|s_{i,t}),\pi_{{\bm{\theta}}}(\cdot|s_{i,t}))}\right\};
\]
\ENDFOR
\tikzmark{end1}
\STATE \tikzmark{start2}$\pi_{\text{old}}\leftarrow\pi_{\bm{\theta}_{k+\frac{1}{2}}}$;~~~~~~~~~~~~~~~~~~~~~~~~~~~~~~~~~~~~~~~~~~~~~~~~~~~~~~~~~~~~~~~~~~~~~~~~~~~~~~~~~~~~~~~~~~~~~~~~~~~~~~~~~~~{\color{blue}{\texttt{Projection}}}
\FOR{$i=0,1,2,\ldots,M$}
\STATE
\begin{flalign}
\nonumber
&~~~~~~~~~~~~~~~~~~~~~~~~~~~~~~~~~~\nu_{i,k+1}=\left\{\nu_{i,k}+\eta(\hat{J}_{i}^{C}-b)\right\}_{+}
\\
\nonumber
{\bm{\theta}}_{k+1}&=\arg\min_{{\bm{\theta}}} \dfrac{1}{T}\sum_{t=1}^{T}
\left\{
\text{KL}(\pi_{{\bm{\theta}}_{\text{old}}}(\cdot|s_{i,t}),\pi_{{\bm{\theta}}}(\cdot|s_{i,t}))+\nu_{i,k+1}
\dfrac{1-\gamma\lambda}{1-\gamma}\dfrac{\pi_{\bm{\theta}}(a_{i,t}|s_{i,t})}{\pi_{{\bm{\theta}}_k}(a_{i,t}|s_{i,t})}\hat{A}^{C}_{i,t}
\right\};
\end{flalign}
\ENDFOR
\tikzmark{end2}
\FOR{each mini-batch $\{(a_{j},s_{j},\hat{A}_{j},\hat{A}^{C}_{j},V^{\text{target}}_{j},V^{\text{target},C}_{j})\}$ of size $B$ from $\calD_k$}
\STATE
\[
\bm{\omega}_{k+1}=\arg\min_{\bm{\omega}}
\sum_{j=1}^{B}
\left(
V_{\bm{\omega}}(s_j)-V^{\text{target}}_{j}
\right)^2,
\bm{\nu}_{k+1}=\arg\min_{\bm{\nu}}
\sum_{j=1}^{B}
\left(
V^{c}_{\bm{\nu}}(s_j)-V^{\text{target},C}_{j}
\right)^2;
\]
\ENDFOR
\Textbox{start1}{end1}{}
\Textbox{start2}{end2}{}
\ENDFOR
\end{algorithmic}
\end{algorithm}

In this section, we present the practical implementation of CUP.

\subsection{Step 1: Policy Improvement}

\textbf{Objective of Policy Improvement.}

For the first step w.r.t. policy improvement (\ref{performance-improvement-01}), 
\begin{flalign}
\nonumber
\pi_{{\bm{\theta}}_{k+\frac{1}{2}}}=&\arg\max_{\pi_{{\bm{\theta}}}\in\Pi_{{\bm{\theta}}}}
\left\{
\E_{s\sim{d}_{\pi_{\bm{\theta}_k}}^{\lambda}(\cdot),a\sim\policy(\cdot|s)}
\left[
A^{\text{GAE}(\gamma,\lambda)}_{{\pi_{\bm{\theta}_k}}}(s,a)\right]-\alpha_k
\sqrt{
\E_{s\sim{d}_{{\pi_{\bm{\theta}_k}}}^{\lambda}(\cdot)}\left[\text{KL}(\pi_{\bm{\theta}_k},\policy)[s]\right]}
\right\},
\end{flalign}
according to 
\begin{flalign}
\label{app-01-estimator}
\E_{s\sim{d}_{\pi_{\bm{\theta}_k}}^{\lambda}(\cdot),a\sim\policy(\cdot|s)}
\left[
A^{\text{GAE}(\gamma,\lambda)}_{{\pi_{\bm{\theta}_k}}}(s,a)\right]=
\E_{s\sim d_{\pi_{{\bm{\theta}}_k}}^{\lambda}(\cdot),a\sim\pi_{{\bm{\theta}}_k}(\cdot|s)}\left[\dfrac{\pi_{{\bm{\theta}}}(a|s)}{\pi_{{\bm{\theta}}_k}(a|s)}A^{\text{GAE}(\gamma,\lambda)}_{\pi_{{\bm{\theta}}_k}}(s,a)\right],
\end{flalign}
which implies Eq.(\ref{performance-improvement}):
\begin{flalign}
\label{performance-improvement-02}
\pi_{{\bm{\theta}}_{k+\frac{1}{2}}}=&\arg\max_{\pi_{{\bm{\theta}}}\in\Pi_{{\bm{\theta}}}}
\left\{
\E_{s\sim d_{\pi_{{\bm{\theta}}_k}}^{\lambda}(\cdot),a\sim\pi_{{\bm{\theta}}_k}(\cdot|s)}\left[\dfrac{\pi_{{\bm{\theta}}}(a|s)}{\pi_{{\bm{\theta}}_k}(a|s)}A^{\text{GAE}(\gamma,\lambda)}_{\pi_{{\bm{\theta}}_k}}(s,a)\right]
-\alpha_k
\sqrt{
\E_{s\sim{d}_{{\pi_{\bm{\theta}_k}}}^{\lambda}(\cdot)}\left[\text{KL}(\pi_{\bm{\theta}_k},\policy)[s]\right]}
\right\}.
\end{flalign}
We replace (\ref{performance-improvement-01}) with an importance sampling with respect to $\pi_{\bm{\theta}_k}$ to obtain the expectation (\ref{performance-improvement-02}), and all the remains is to replace the expectation (\ref{performance-improvement-02}) by sample averages according to the trajectories collected by $\pi_{\bm{\theta}_k}$.

\textbf{Learning from Samples.}

For each trajectories (with size $M$) sampled from 
\[\bigcup_{i=1}^{M}\bigcup_{t=0}^{T}\left\{(s_{i,t},a_{i,t},r_{i,t+1},c_{i,t+1})\right\}\sim\pi_{{\bm{\theta}}_k},\] 
we learn the parameter $\bm{\theta}_{k+\frac{1}{2}}$ as follows: for each $i=1,2,\cdots,M$,
\begin{flalign}
{\bm{\theta}}_{k+\frac{1}{2}}=\arg\max_{{\bm{\theta}}}\left\{\dfrac{1}{T}\sum_{t=1}^{T}\frac{\pi_{\bm{\theta}}(a_{i,t}|s_{i,t})}{{\pi_{\bm{\theta}_{k}}}(a_{i,t}|s_{i,t})}\hat{A}_{i,t}-\alpha\sqrt{\dfrac{1}{T}\sum_{t=1}^{T}\text{KL}(\pi_{\bm{\theta}_{k}}(\cdot|s_{i,t}),\pi_{{\bm{\theta}}}(\cdot|s_{i,t}))}\right\},
\end{flalign}
 which can be solved via the first order optimizer, where the following three terms
 \[
\hat{A}_{i,t},~~\dfrac{1}{T}\sum_{t=1}^{T}\frac{\pi_{\bm{\theta}}(a_{i,t}|s_{i,t})}{{\pi_{\bm{\theta}_{k}}}(a_{i,t}|s_{i,t})}\hat{A}_{i,t},~~\dfrac{1}{T}\sum_{t=1}^{T}\text{KL}(\pi_{\bm{\theta}_{k}}(\cdot|s_{i,t}),\pi_{{\bm{\theta}}}(\cdot|s_{i,t}))
 \]
 are the estimators (according to the $i$-th trajectory $\{(s_{i,t},a_{i,t},r_{i,t+1},c_{i,t+1})\}_{t=0}^{T}\sim\pi_{\bm{\theta}_k}$ of the following three expectations correspondingly:
 \[
 A^{\text{GAE}(\gamma,\lambda)}_{\pi_{{\bm{\theta}}_k}},~~\E_{s\sim d_{\pi_{{\bm{\theta}}_k}}^{\lambda}(\cdot),a\sim\pi_{{\bm{\theta}}_k}(\cdot|s)}\left[\dfrac{\pi_{{\bm{\theta}}}(a|s)}{\pi_{{\bm{\theta}}_k}(a|s)}A^{\text{GAE}(\gamma,\lambda)}_{\pi_{{\bm{\theta}}_k}}(s,a)\right],~~
\sqrt{
\E_{s\sim{d}_{{\pi_{\bm{\theta}_k}}}^{\lambda}(\cdot)}\left[\text{KL}(\pi_{\bm{\theta}_k},\policy)[s]\right]}.
 \]

\subsection{Step 2: Projection}

\textbf{Objective of Projection.}

Recall Proposition \ref{propo-03} with respected to cost
\begin{flalign}
\nonumber
J^{c}(\policy)-J^{c}(\policyy)\leq&
\dfrac{1}{1-\tilde\gamma}\E_{s\sim{d}_{\policyy}^{\lambda}(\cdot),a\sim\policy(\cdot|s)}
\Bigg[
A^{\text{GAE}(\gamma,\lambda)}_{\policyy,C}(s,a)\\
\nonumber
&~~~~~~~~~~~+\left.\frac{2\gamma(1-\lambda)\epsilon^{C}_{\policy}(\policyy)}{(1-\gamma\lambda)\left|1-2\gamma\lambda|\calS||\calA|\right|}
\sqrt{\dfrac{1}{2}\E_{s\sim{d}_{\policyy}^{\lambda}(\cdot)}\left[\text{KL}(\policyy,\policy)[s]\right]}
\right],
\end{flalign}
and we introduce a new surrogate function with respected to cost function as follows
\[
C_{\policyy}(\policy,\beta)=
\dfrac{1}{1-\tilde\gamma}\E_{s\sim{d}_{\policyy}^{\lambda}(\cdot),a\sim\policy(\cdot|s)}
\left[
A^{\text{GAE}(\gamma,\lambda)}_{\policyy,C}(s,a)+\beta
\sqrt{\dfrac{1}{2}\E_{s\sim{d}_{\policyy}^{\lambda}(\cdot)}\left[\text{KL}(\policyy,\policy)[s]\right]}
\right],
\]
where $\beta$ is adaptive to $\frac{2\gamma(1-\lambda)\epsilon^{C}_{\policy}(\policyy)}{(1-\gamma\lambda)\left|1-2\gamma\lambda|\calS||\calA|\right|}$.

Now, consider the projection step:
\begin{flalign}
\label{eq:projection}
\pi_{{\bm{\theta}}_{k+1}}=&\arg\min_{\pi_{{\bm{\theta}}}\in\Pi_{{\bm{\theta}}}}~D\left(\pi_{{\bm{\theta}}},\pi_{{\bm{\theta}}_{k+\frac{1}{2}}}\right),~~\text{s.t.}~C_{\pi_{\bm{\theta}_k}}(\policy,\beta)\leq b.
\end{flalign}
We turn the projection step as the following unconstrained problem:
\begin{flalign}
\label{app-proj-01}
\max_{\nu\ge0}\min_{\pi_{\bm{\theta}}}
\left\{
D\left(\pi_{{\bm{\theta}}},\pi_{{\bm{\theta}}_{k+\frac{1}{2}}}\right)+\nu
\left(
C_{\pi_{\bm{\theta}_k}}(\policy,\beta)-b
\right)
\right\}.
\end{flalign}
In our implementation, we use KL-divergence as the distance measure $D(\cdot,\cdot)$, then
 \begin{flalign}
\label{app-d-measure}
 D\left(\pi_{{\bm{\theta}}},\pi_{{\bm{\theta}}_{k+\frac{1}{2}}}\right)=\E_{s\sim{d}_{{\pi_{\bm{\theta}_k}}}^{\lambda}(\cdot)}\left[\text{KL}\left(\pi_{\bm{\theta}_{k+\frac{1}{2}}},\policy\right)[s]\right],
 \end{flalign}
 which implies we can rewrite the problem (\ref{app-proj-01}) as follows,
 \begin{flalign}
\label{app-proj-01}
\max_{\nu\ge0}\min_{\pi_{\bm{\theta}}}
\left\{
\E_{s\sim{d}_{{\pi_{\bm{\theta}_k}}}^{\lambda}(\cdot)}\left[\text{KL}\left(\pi_{\bm{\theta}_{k+\frac{1}{2}}},\policy\right)[s]\right]+\nu
\left(
C_{\pi_{\bm{\theta}_k}}(\policy,\beta)-b
\right)
\right\}.
\end{flalign}
 Recall 
 \begin{flalign}
 \label{cost-sur}
 C_{\pi_{\bm{\theta}_k}}(\policy,\beta)=J^{c}(\pi_{{\bm{\theta}}_k})+\dfrac{1}{1-\tilde\gamma}\E_{s\sim{d}_{\pi_{\bm{\theta}_k}}^{\lambda}(\cdot),a\sim\policy(\cdot|s)}
\left[
A^{\text{GAE}(\gamma,\lambda)}_{{\pi_{\bm{\theta}_k}},C}(s,a)\right]+\beta_k
\sqrt{
\E_{s\sim{d}_{{\pi_{\bm{\theta}_k}}}^{\lambda}(\cdot)}\left[\text{KL}(\pi_{\bm{\theta}},\policy)[s]\right]},
 \end{flalign}
to simplify the problem, we ignore the term $\beta
\sqrt{
\E_{s\sim{d}_{{\pi_{\bm{\theta}_k}}}^{\lambda}(\cdot)}\left[\text{KL}(\pi_{\bm{\theta}_k},\policy)[s]\right]}$ in Eq.(\ref{cost-sur}) due to the following two aspects:
(i) firstly, $\beta$ is adapted to the term $\frac{\gamma(1-\lambda)\epsilon^{C}_{\pi_{\bm{\theta}_{k+1}}}(\pi_{\bm{\theta}_{k}})}{(1-\gamma\lambda)\left|1-2\gamma\lambda|\calS||\calA|\right|}$, and for the high-dimensional state space or continuous action space, then $\beta$ is very small;
(ii) secondly, if $D$ is a KL-divergence measure, then the direction of the policy optimization $D\left(\pi_{{\bm{\theta}}},\pi_{{\bm{\theta}}_{k+\frac{1}{2}}}\right)$ (\ref{app-d-measure}) is proportional to $\beta
\sqrt{
\E_{s\sim{d}_{{\pi_{\bm{\theta}_k}}}^{\lambda}(\cdot)}\left[\text{KL}(\pi_{\bm{\theta}_k},\policy)[s]\right]}$, thus, in practice, we can only optimize the distance $D\left(\pi_{{\bm{\theta}}},\pi_{{\bm{\theta}}_{k+\frac{1}{2}}}\right)$.

 Above discussions implies that instead of (\ref{app-proj-01}), we can consider the problem
 \begin{flalign}
\nonumber
\max_{\nu\ge0}\min_{\pi_{\bm{\theta}}}
\calL(\bm{\theta},\nu),
\end{flalign}
where 
\[
\calL(\bm{\theta},\nu)
=\E_{s\sim{d}_{{\pi_{\bm{\theta}_k}}}^{\lambda}(\cdot)}\left[\text{KL}\left(\pi_{\bm{\theta}_{k+\frac{1}{2}}},\policy\right)[s]\right]+\nu
\bigg(
J^{c}(\pi_{{\bm{\theta}}_k})+\dfrac{1}{1-\tilde\gamma}\E_{s\sim{d}_{\pi_{\bm{\theta}_k}}^{\lambda}(\cdot),a\sim\policy(\cdot|s)}
\left[
A^{\text{GAE}(\gamma,\lambda)}_{{\pi_{\bm{\theta}_k}},C}(s,a)\right]-b
\bigg).
\]

\textbf{Learning from Samples.}

Then, according to gradient decent method, we have
\begin{flalign}
\label{learning-para}
\bm{\theta}\leftarrow \bm{\theta}-\eta\dfrac{\calL(\bm{\theta},\nu)}{\partial \bm{\theta}},~~~
\nu\leftarrow \left\{\nu+\eta\dfrac{\calL(\bm{\theta},\nu)}{\partial \nu}\right\}_{+},
 \end{flalign}
where $\{\cdot\}_{+}$ denote the positive part, i.e., if $x\leq0$, $\{x\}_{+}=0$, else $\{x\}_{+}=x$.
Particularly, 
\begin{flalign}
\dfrac{\calL(\bm{\theta},\nu)}{\partial \nu}=J^{c}(\pi_{{\bm{\theta}}_k})+\dfrac{1}{1-\tilde\gamma}\E_{s\sim{d}_{\pi_{\bm{\theta}_k}}^{\lambda}(\cdot),a\sim\policy(\cdot|s)}
\left[
A^{\text{GAE}(\gamma,\lambda)}_{{\pi_{\bm{\theta}_k}},C}(s,a)\right]-b,
\end{flalign}
where the term $\E_{s\sim{d}_{\pi_{\bm{\theta}_k}}^{\lambda}(\cdot),a\sim\policy(\cdot|s)}
\left[
A^{\text{GAE}(\gamma,\lambda)}_{{\pi_{\bm{\theta}_k}},C}(s,a)\right]$ can be estimated following the idea as (\ref{app-01-estimator}):
\[\E_{s\sim{d}_{\pi_{\bm{\theta}_k}}^{\lambda}(\cdot),a\sim\policy(\cdot|s)}
\left[
A^{\text{GAE}(\gamma,\lambda)}_{{\pi_{\bm{\theta}_k}},C}(s,a)\right]=
\E_{s\sim d_{\pi_{{\bm{\theta}}_k}}^{\lambda}(\cdot),a\sim\pi_{{\bm{\theta}}_k}(\cdot|s)}\left[\dfrac{\pi_{{\bm{\theta}}}(a|s)}{\pi_{{\bm{\theta}}_k}(a|s)}A^{\text{GAE}(\gamma,\lambda)}_{\pi_{{\bm{\theta}}_k},C}(s,a)\right].
\]

But recall (\ref{performance-improvement}) is a MM-iteration, i.e., 
we require to minimize $\E_{s\sim{d}_{\pi_{\bm{\theta}_k}}^{\lambda}(\cdot)}\text{KL}\left(\pi_{{\bm{\theta}}},\pi_{{\bm{\theta}}_{k}}\right)[s]$, which implies $\policy$ is close to $\pi_{\bm{\theta}_k}$.
Thus it is reasonable $\E_{s\sim{d}_{\pi_{\bm{\theta}_k}}^{\lambda}(\cdot),a\sim\policy(\cdot|s)}
\left[
A^{\text{GAE}(\gamma,\lambda)}_{{\pi_{\bm{\theta}_k}},C}(s,a)\right]\approx 0$, thus, in practice, we update $\nu$ following a simple way
\[
\nu\leftarrow\left\{\nu+\eta(J^{c}(\pi_{{\bm{\theta}}_k})-b)\right\}_{+}.
\]

Finally, according to (\ref{learning-para}), for each data sampled from 
$\cup_{i=1}^{M}\cup_{t=0}^{T}\left\{(s_{i,t},a_{i,t},r_{i,t+1},c_{i,t+1})\right\}$ according to current policy $\pi_{{\bm{\theta}}_k}$, 
we learn the parameter $\bm{\theta}_{k+1}$ as follows,
\begin{flalign}
\nonumber
{\bm{\theta}}_{k+1}&=\arg\min_{{\bm{\theta}}} \dfrac{1}{T}\sum_{t=1}^{T}
\left\{
\text{KL}\left({\pi_{\bm{\theta}_{k+\frac{1}{2}}}}(\cdot|s_{i,t}),\pi_{{\bm{\theta}}}(\cdot|s_{i,t})\right)+\nu_k
\dfrac{1-\gamma\lambda}{1-\gamma}\dfrac{\pi_{\bm{\theta}}(a_{i,t}|s_{i,t})}{\pi_{{\bm{\theta}}_k}(a_{i,t}|s_{i,t})}\hat{A}^{C}_{i,t}
\right\},
\end{flalign}
which can be solved via the first-order optimizer.

 \clearpage

 \section{Additional Discussion about Related Work}
\label{app-related-work}

This section reviews three typical safe reinforcement learning algorithms: CPO \cite{AchiamHTA17}, PCPO \cite{yang2020projection} and FOCOPS \cite{zhang2020first}.
Those algorithms also use new surrogate functions to replace the objective and constraints, which resembles the proposed CUP algorithm.
The goal is to present the contribution of our work.

\subsection{CPO \cite{AchiamHTA17}}
For a given policy $\pi_{{\bm{\theta}}_{k}}$, CPO updates new policy $\pi_{{\bm{\theta}}_{k+1}}$ as follows:
\begin{flalign}
\label{app-cpo-objective}
&\pi_{{\bm{\theta}}_{k+1}}=\arg\max_{\pi_{{\bm{\theta}}}\in\Pi_{{\bm{\theta}}}}~~~~\E_{s\sim d^{\rho_0}_{\pi_{{\bm{\theta}}_k}}(\cdot),a\sim\pi_{{\bm{\theta}}}(\cdot|s)}\left[A_{\pi_{{\bm{\theta}}_k}}(s,a)\right]\\
\label{app-cost-constraint}
&~~~~~~~~\text{s.t.}~~J^{c}(\pi_{{\bm{\theta}}_k})+\frac{1}{1-\gamma}\E_{s\sim d^{\rho_0}_{\pi_{{\bm{\theta}}_k}}(\cdot),a\sim\pi_{{\bm{\theta}}}(\cdot|s)}\left[A^{c}_{\pi_{{\bm{\theta}}_k}}(s,a)\right]\leq b,\\
\label{app-trust-region}
&~~~~~~~~\bar{D}_{\text{KL}}(\pi_{{\bm{\theta}}},\pi_{{\bm{\theta}}_k})=\E_{s\sim d^{\rho_0}_{\pi_{{\bm{\theta}}_k}}(\cdot)}[\text{KL}(\pi_{{\bm{\theta}}},\pi_{{\bm{\theta}}_k})[s]]\leq\delta.
\end{flalign}
It is impractical to solve the problem (\ref{cpo-objective}) directly due to the computational cost.
\cite{AchiamHTA17} suggest to find some convex approximations to replace the term $A_{\pi_{{\bm{\theta}}_k}}(s,a)$ and $\bar{D}_{\text{KL}}(\pi_{{\bm{\theta}}},\pi_{{\bm{\theta}}_k})$ Eq.(\ref{cpo-objective})-(\ref{trust-region}).

Concretely, according to (\ref{performance-difference-2002}), \cite{AchiamHTA17} suggest to use first-order Taylor expansion of $J(\pi_{\bm{\theta}})$ to replace the objective (\ref{cpo-objective}) as follows,
\begin{flalign}
\nonumber
\frac{1}{1-\gamma}\E_{s\sim d_{\pi_{{\bm{\theta}}_k}}^{\rho_0}(\cdot),a\sim\pi_{{\bm{\theta}}_k}(\cdot|s)}\left[\dfrac{\pi_{{\bm{\theta}}}(a|s)}{\pi_{{\bm{\theta}}_k}(a|s)}A_{\pi_{{\bm{\theta}}_k}}(s,a)\right]=J(\pi_{{\bm{\theta}}})-J(\pi_{{\bm{\theta}}_k})\approx({\bm{\theta}}-{\bm{\theta}}_k)^{\top}\nabla_{{\bm{\theta}}} J(\pi_{\bm{\theta}}).
\end{flalign}
Similarly, \cite{AchiamHTA17} use the following approximations to turn the constrained policy optimization (\ref{cpo-objective})-(\ref{trust-region}) to be a convex problem,
\begin{flalign}
\label{app-contrained-01}
\frac{1}{1-\gamma}\E_{s\sim d_{\pi_{{\bm{\theta}}_k}}^{\rho_0}(\cdot),a\sim\pi_{{\bm{\theta}}_k}(\cdot|s)}&\left[\frac{\pi_{{\bm{\theta}}}(a|s)}{\pi_{{\bm{\theta}}_k}(a|s)}A^{c}_{\pi_{{\bm{\theta}}_k}}(s,a)\right]\approx({\bm{\theta}}-{\bm{\theta}}_k)^{\top}\nabla_{{\bm{\theta}}} J^{c}(\pi_{\bm{\theta}}),\\
\label{app-contrained-02}
\bar{D}_{\text{KL}}(\pi_{{\bm{\theta}}},\pi_{{\bm{\theta}}_k})&\approx({\bm{\theta}}-{\bm{\theta}}_k)^{\top}\mathbf{H}({\bm{\theta}}-{\bm{\theta}}_k),
\end{flalign}
where $\mathbf{H}$ is Hessian matrix of $\bar{D}_{\text{KL}}(\pi_{{\bm{\theta}}},\pi_{{\bm{\theta}}_k})$, i.e., 
\[\mathbf{H}[i,j]=:\dfrac{\partial^2}{\partial {\bm{\theta}}_i \partial {\bm{\theta}}_j}\E_{s\sim d^{\rho_0}_{\pi_{{\bm{\theta}}_k}}(\cdot)}\left[\text{KL}(\pi_{{\bm{\theta}}},\pi_{{\bm{\theta}}_k})[s]\right],\]
Eq.(\ref{app-contrained-02}) is the second-oder approximation of (\ref{trust-region}).

Let $ \lambda_{\star},\nu_{\star}$ is the dual solution of the following problem
\[
\lambda_{\star},\nu_{\star}=\arg\max_{\lambda\ge0,\nu\ge0}
\left\{
\dfrac{-1}{2\lambda}
\left(
\bg^{\top}\bH^{-1}\bg-2\nu r+sv^2
\right)
+\nu c-\dfrac{\lambda\delta}{2}
\right\}
;
\]
where
$\bg=\nabla_{\bm{\theta}}\E_{s\sim d^{\rho_0}_{\pi_{{\bm{\theta}}_k}}(\cdot),a\sim\pi_{{\bm{\theta}}}(\cdot|s)}\left[A_{\pi_{{\bm{\theta}}_k}}(s,a)\right]$, $\ba=\nabla_{\bm{\theta}}\E_{s\sim d^{\rho_0}_{\pi_{{\bm{\theta}}_k}}(\cdot),a\sim\pi_{{\bm{\theta}}}(\cdot|s)}\left[A^{c}_{\pi_{{\bm{\theta}}_k}}(s,a)\right]$, $r=\bg^{\top}\bH\ba,s=\ba^{\top}\bH^{-1}\ba$, and $c=J^{c}(\pi_{\bm{\theta}_k})-b$.

Finally, CPO updates parameters according to conjugate gradient as follows:
if approximation to CPO is feasible:
\begin{flalign}
\nonumber
\bm{\theta}_{k+1}=\bm{\theta}_{k}+\frac{1}{\lambda_{\star}}\bH^{-1}(\bg-\nu_{\star}\ba),
\end{flalign}
else,
\[
\bm{\theta}_{k+1}=\bm{\theta}_{k}-\sqrt{\dfrac{2\delta}{\ba^{\top}\bH^{-1}\ba}}\bH^{-1} \ba.
\]

 \subsection{PCPO \cite{yang2020projection}}

Projection-Based Constrained Policy Optimization (PCPO) is an iterative method for optimizing policies in a two-step process: the first step performs a local reward improvement update, while the second step reconciles any constraint violation by projecting the policy back onto the constraint set.

 \textbf{Reward Improvement:}
 \begin{flalign}
 \nonumber
 \pi_{\bm{\theta}_{k+\frac{1}{2}}}=\arg\max_{\pi_{{\bm{\theta}}}\in\Pi_{{\bm{\theta}}}}\E_{s\sim d^{\rho_0}_{\pi_{{\bm{\theta}}_k}}(\cdot),a\sim\pi_{{\bm{\theta}}}(\cdot|s)}\left[A_{\pi_{{\bm{\theta}}_k}}(s,a)\right],\\
 \nonumber
  \text{ s.t.} \bar{D}_{\text{KL}}(\pi_{{\bm{\theta}}},\pi_{{\bm{\theta}}_k})=\E_{s\sim d^{\rho_0}_{\pi_{{\bm{\theta}}_k}}(\cdot)}[\text{KL}(\pi_{{\bm{\theta}}},\pi_{{\bm{\theta}}_k})[s]]\leq\delta;
\end{flalign}
 \textbf{Projection:}
\begin{flalign}
\nonumber
    & \pi_{{\bm{\theta}}_{k+1}}=\arg\min_{\pi_{{\bm{\theta}}}\in\Pi_{{\bm{\theta}}}}~D\left(\pi_{{\bm{\theta}}},\pi_{{\bm{\theta}}_{k+\frac{1}{2}}}\right),\\
 \nonumber    
    \text{s.t.}~ J^{c}(\pi_{{\bm{\theta}}_k})&+\dfrac{1}{1-\gamma}\E_{s\sim d^{\rho_0}_{\pi_{{\bm{\theta}}_k}}(\cdot),a\sim\pi_{{\bm{\theta}}}(\cdot|s)}\left[A^{c}_{\pi_{{\bm{\theta}}_k}}(s,a)\right]\leq b.
\end{flalign}

Then, \cite{yang2020projection} follows CPO \cite{AchiamHTA17} uses convex approximation to original problem, and calculate the update rule as follows,
\[
\bm{\theta}_{k+1}=\bm{\theta}_{k}-\sqrt{\dfrac{2\delta}{\bg^{\top}\bH^{-1}\bg}}\bH^{-1} \bg
-\max
\left(0,
\dfrac
{
\sqrt{\dfrac{2\delta}{\bg^{\top}\bH^{-1}\bg}}\ba^{\top}\bH^{-1} \bg+c
}
{\ba^{\top}\bL^{-1}\ba}
\right)\bL^{-1}\ba,
\]
where $\bL=\bI$ if $D$ is $\ell_2$-norm, and $\bL=\bH$ if $D$ is KL-divergence.

\subsection{FOCOPS \cite{zhang2020first}}

\cite{zhang2020first} propose the First Order Constrained Optimization in Policy Space (FOCOPS) that is a two-step approach. We present it as follows.

\textbf{Step1: Finding the optimal update policy.}

Firstly, for a given policy $\pi_{\bm{\theta}k}$, we find an optimal update policy $\pi^{\star}$ by solving the optimization problem (\ref{app-non-parameterized-02})-(\ref{app-non-parameterized-03}) in the non-parameterized policy space.
\begin{flalign}
\label{app-non-parameterized-01}
&\pi^{\star}=\arg\max_{\pi\in\Pi}~~~~\E_{s\sim d^{\rho_0}_{\pi_{{\bm{\theta}}_k}}(\cdot),a\sim\pi(\cdot|s)}\left[A_{\pi_{{\bm{\theta}}_k}}(s,a)\right]\\
\label{app-non-parameterized-02}
&~~~~~~~~\text{s.t.}~~J^{c}(\pi_{{\bm{\theta}}_k})+\frac{1}{1-\gamma}\E_{s\sim d^{\rho_0}_{\pi_{{\bm{\theta}}_k}}(\cdot),a\sim\pi(\cdot|s)}\left[A^{c}_{\pi_{{\bm{\theta}}_k}}(s,a)\right]\leq b,\\
\label{app-non-parameterized-03}
&~~~~~~~~\bar{D}_{\text{KL}}(\pi_{{\bm{\theta}}},\pi_{{\bm{\theta}}_k})=\E_{s\sim d^{\rho_0}_{\pi_{{\bm{\theta}}_k}}(\cdot)}[\text{KL}(\pi,\pi_{{\bm{\theta}}_k})[s]]\leq\delta.
\end{flalign}
If $\pi_{\bm{\theta}_k}$ is feasible, then the optimal policy for (\ref{app-non-parameterized-01})-(\ref{app-non-parameterized-03}) takes the following form:
\begin{flalign}
\label{optimal-solu}
\pi^{\star}(a|s)=\dfrac{\pi_{\bm{\theta}_k}(a|s)}{Z_{\lambda,\nu}(s)}\exp
\left(\dfrac{1}{\lambda}
\left(
A_{\pi_{\bm{\theta}_k}}(s,a)-\nu A^{c}_{\pi_{\bm{\theta}_k}}(s,a)
\right)
\right),
\end{flalign}
where $Z_{\lambda,\nu}(s)$ is the partition function which ensures (\ref{optimal-solu}) is a valid probability distribution, $\lambda$ and $\nu$ are solutions to the optimization problem:
\[
\min_{\lambda,\nu\ge0}\lambda\nu+\nu \tilde{b}+\lambda \E_{s\sim d^{\rho_0}_{\pi_{{\bm{\theta}}_k}}(\cdot),a\sim\pi^{\star}(\cdot|s)}\left[Z_{\lambda,\nu}(s)\right],
\]
the term $\tilde{b}=(1-\gamma)(b-J^{c}(\pi_{\bm{\theta}_k}))$.

\textbf{Step 2: Projection}

Then, we project the policy found in the previous step back into the parameterized policy space $\Pi_{\bm{\theta}}$ by solving for the closest policy $\policy\in\Pi_{\bm{\theta}}$ to $\pi^{\star}$ in order to obtain $\pi_{\bm{\theta}_{k+1}}$:
\[
\bm{\theta}_{k+1}=\arg\min_{\bm\theta} \E_{s\sim d^{\rho_0}_{\pi_{{\bm{\theta}}_k}}(\cdot)}[\text{KL}(\pi_{{\bm{\theta}}},\pi^{\star})[s]].
\]

\subsection{Comparison to CUP}

Comparing to CPO and PCPO, the implementation of CUP does not depend on any convex approximations.
CPO learns its objective with the deep neural network via the first-order method (see Appendix \ref{sec-app-cpu}).

Concretely, CPO and PCPO approximate the non-convex objective (or constraints) with first-order or second Taylor expansion,
but their implementations still lack a theory to show the error difference between the original objective (or constraints) and its convex approximations.
Additionally, their approaches involve the inverse of a high-dimension Fisher information matrix, which causes their algorithms to require a costly computation for each update when solving high-dimensional RL problems.
While the proposed CUP does not depend on any convex approximations, it learns the policy via first-order optimization approaches.
Thus, CUP does not involve the inverse of a high-dimension Fisher information matrix, which implies CUP requires less memory than CPO and PCPO.

Although FOCOPS is also a non-convex implementation, it heavily depends on the current best-satisfied policy. 
It is known that the current best policy may not be the optimal policy, and FOCOPS requires to project this policy back into the parametric policy space, which implies FOCOPS reduce the chances for an agent to explore the environment since it may lose in a locally optimal solution.
While the proposed CUP does not depend on the current optimal policy, in fact, CUP requires the agent to learn the policy according to (\ref{performance-improvement}), the numerical solution is not the current optimal policy, which helps CUP to explore the environment.

\begin{landscape}
\renewcommand{\arraystretch}{1.5}
\begin{table}[t]
 \caption{Comparison of some safe reinforcement algorithms.}
 \label{tab:comp}
 \centering
 \small
 \setlength{\extrarowheight}{0.04cm}
   \begin{adjustbox}{width=22cm,height=9cm}
 \begin{tabularx}{\linewidth}{m{2.5cm}|>{\centering}m{8cm}|>{\centering}m{7.5cm}|m{3cm}}
  \toprule
  Algorithm & Optimization problem & Implementation & Remark \\
  \midrule
  \makecell{CPO\\ \cite{AchiamHTA17}}
   & $\pi_{\bm{\theta}_{k+1}}=\arg\max_{\pi_{{\bm{\theta}}}\in\Pi_{{\bm{\theta}}}}\E_{s\sim d^{\rho_0}_{\pi_{{\bm{\theta}}_k}}(\cdot),a\sim\pi_{{\bm{\theta}}}(\cdot|s)}\left[A_{\pi_{{\bm{\theta}}_k}}(s,a)\right]$,
   \newline
    s.t. $J^{c}(\pi_{{\bm{\theta}}_k})+\E_{s\sim d^{\rho_0}_{\pi_{{\bm{\theta}}_k}}(\cdot),a\sim\pi_{{\bm{\theta}}}(\cdot|s)}\left[A^{c}_{\pi_{{\bm{\theta}}_k}}(s,a)\right]\leq b $,
    \newline
   $ \bar{D}_{\text{KL}}(\pi_{{\bm{\theta}}},\pi_{{\bm{\theta}}_k})=\E_{s\sim d^{\rho_0}_{\pi_{{\bm{\theta}}_k}}(\cdot)}[\text{KL}(\pi_{{\bm{\theta}}},\pi_{{\bm{\theta}}_k})[s]]\leq\delta$.
    & $\bm{\theta}_{k+1}=\arg\max_{\bm\theta}~\bg^{\top}(\bm{\theta}-\bm{\theta}_{k})$, 
    \newline
     s.t. $c+\mathbf{b}^{\top}(\bm{\theta}-\bm{\theta}_k)\leq0$,
     \newline
   $ \dfrac{1}{2}(\bm{\theta}-\bm{\theta}_k)^{\top}\mathbf{H}(\bm{\theta}-\bm{\theta}_k)\leq\delta$.
    &\makecell{Convex\\ Implementation}
    \\
   \hline 
    \makecell{PCPO\\ \cite{yang2020projection}}
   &Reward Improvement $\pi_{\bm{\theta}_{k+\frac{1}{2}}}=\arg\max_{\pi_{{\bm{\theta}}}\in\Pi_{{\bm{\theta}}}}\E_{s\sim d^{\rho_0}_{\pi_{{\bm{\theta}}_k}}(\cdot),a\sim\pi_{{\bm{\theta}}}(\cdot|s)}\left[A_{\pi_{{\bm{\theta}}_k}}(s,a)\right]$,
    \newline
   s.t. $\bar{D}_{\text{KL}}(\pi_{{\bm{\theta}}},\pi_{{\bm{\theta}}_k})=\E_{s\sim d^{\rho_0}_{\pi_{{\bm{\theta}}_k}}(\cdot)}[\text{KL}(\pi_{{\bm{\theta}}},\pi_{{\bm{\theta}}_k})[s]]\leq\delta$;
    \newline
    Projection
     \newline
    $\pi_{{\bm{\theta}}_{k+1}}=\arg\min_{\pi_{{\bm{\theta}}}\in\Pi_{{\bm{\theta}}}}~D\left(\pi_{{\bm{\theta}}},\pi_{{\bm{\theta}}_{k+\frac{1}{2}}}\right)$,
    \newline
    s.t. $J^{c}(\pi_{{\bm{\theta}}_k})+\dfrac{1}{1-\gamma}\E_{s\sim d^{\rho_0}_{\pi_{{\bm{\theta}}_k}}(\cdot),a\sim\pi_{{\bm{\theta}}}(\cdot|s)}\left[A^{c}_{\pi_{{\bm{\theta}}_k}}(s,a)\right]\leq b $.
    & Reward Improvement
     \newline
    $\bm{\theta}_{k+\frac{1}{2}}=\arg\max_{\bm\theta}~\bg^{\top}(\bm{\theta}-\bm{\theta}_{k})$, 
     \newline
     s.t.$\dfrac{1}{2}(\bm{\theta}-\bm{\theta}_k)^{\top}\mathbf{H}(\bm{\theta}-\bm{\theta}_k)\leq\delta$;
    \newline
    Projection
     \newline
      $\pi_{{\bm{\theta}}_{k+1}}=\arg\min_{\bm\theta}\dfrac{1}{2}(\bm{\theta}-\bm{\theta}_k)^{\top}\mathbf{L}(\bm{\theta}-\bm{\theta}_k)$,
       \newline 
     s.t. $c+\mathbf{b}^{\top}(\bm{\theta}-\bm{\theta}_k)\leq0$.
    &\makecell{Convex\\ Implementation}\\  
   \hline 
    \makecell{FOCOPS\\ \cite{zhang2020first}}
   &Optimal update policy
    \newline
    $\pi^{\star}=\arg\max_{\pi\in\Pi}\E_{s\sim d^{\rho_0}_{\pi_{{\bm{\theta}}_k}}(\cdot),a\sim\pi(\cdot|s)}\left[A_{\pi_{{\bm{\theta}}_k}}(s,a)\right]$,
   \newline
    s.t. $J^{c}(\pi_{{\bm{\theta}}_k})+\E_{s\sim d^{\rho_0}_{\pi_{{\bm{\theta}}_k}}(\cdot),a\sim\pi(\cdot|s)}\left[A^{c}_{\pi_{{\bm{\theta}}_k}}(s,a)\right]\leq b $,
    \newline
   $ \bar{D}_{\text{KL}}(\pi_{{\bm{\theta}}},\pi_{{\bm{\theta}}_k})=\E_{s\sim d^{\rho_0}_{\pi_{{\bm{\theta}}_k}}(\cdot)}[\text{KL}(\pi,\pi_{{\bm{\theta}}_k})[s]]\leq\delta$;
   \newline
   Projection
    \newline
    $\pi_{{\bm{\theta}}_{k+1}}=\arg\min_{\pi_{{\bm{\theta}}}\in\Pi_{{\bm{\theta}}}}~\E_{s\sim d^{\rho_0}_{\pi_{{\bm{\theta}}_k}}(\cdot)}[\text{KL}(\pi_{{\bm{\theta}}},\pi^{\star})[s]]$.
    & Optimal update policy
    \newline
     \newline
    $\pi^{\star}(a|s)=\frac{\pi_{\bm{\theta}_k}(a|s)}{Z_{\lambda,\nu}(s)}\exp\left(\frac{1}{\lambda}\left(A_{\pi_{\bm{\theta}_k}}(s,a)-\nu A^{c}_{\pi_{\bm{\theta}_k}}(s,a)\right)\right)$;
     \newline
      \newline
      Projection
       \newline
    $\bm{\theta}_{k+1}=\arg\min_{\bm\theta} \E_{s\sim d^{\rho_0}_{\pi_{{\bm{\theta}}_k}}(\cdot)}[\text{KL}(\pi_{{\bm{\theta}}},\pi^{\star})[s]].$
    & \makecell{Non-Convex \\Implementation}
    \\
     \hline 
    \makecell{CUP\\ (Our Work)}
    &
    Policy Improvement
    \begin{flalign} 
    \nonumber
    \pi_{{\bm{\theta}}_{k+\frac{1}{2}}}=\arg\max_{\pi_{{\bm{\theta}}}\in\Pi_{{\bm{\theta}}}}
\bigg\{
\E_{s\sim{d}_{\pi_{\bm{\theta}_k}}^{\lambda}(\cdot),a\sim\policy(\cdot|s)}
\left[
A^{\text{GAE}(\gamma,\lambda)}_{{\pi_{\bm{\theta}_k}}}(s,a)\right]
\\
 \nonumber
-\alpha_k
\sqrt{
\E_{s\sim{d}_{{\pi_{\bm{\theta}_k}}}^{\lambda}(\cdot)}\left[\text{KL}(\pi_{\bm{\theta}_k},\policy)[s]\right]}
\bigg\},
\end{flalign}
Projection
 \begin{flalign}
\nonumber
\pi_{{\bm{\theta}}_{k+1}}=\arg\min_{\pi_{{\bm{\theta}}}\in\Pi_{{\bm{\theta}}}}~D\Big(\pi_{{\bm{\theta}}},\pi_{{\bm{\theta}}_{k+\frac{1}{2}}}\Big),~~~~~~~~~~~~~~~~~\\
\nonumber
 \text{s.t.} J^{c}(\pi_{{\bm{\theta}}_k})+\dfrac{1}{1-\tilde\gamma}\E_{s\sim{d}_{\pi_{\bm{\theta}_k}}^{\lambda}(\cdot),a\sim\policy(\cdot|s)}
\left[A^{\text{GAE}(\gamma,\lambda)}_{{\pi_{\bm{\theta}_k}},C}(s,a)\right]\\
\nonumber
+\beta_k\sqrt{\E_{s\sim{d}_{{\pi_{\bm{\theta}_k}}}^{\lambda}(\cdot)}\left[\text{KL}(\pi_{\bm{\theta}_k},\policy)[s]\right]}\leq b. 
\end{flalign}
    &
    Policy Improvement
     \newline
    \begin{flalign}
    \nonumber
{\bm{\theta}}_{k+\frac{1}{2}}=\arg\max_{{\bm{\theta}}}\Bigg\{\frac{1}{T}\sum_{t=1}^{T}\frac{\pi_{\bm{\theta}}(a_{t}|s_{t})}{{\pi_{\bm{\theta}_{k}}}(a_{t}|s_{t})}\hat{A}_{t}~~~~~~~~~~~~~~~~~~~~~~~~~~~~~\\
\nonumber
-\alpha\sqrt{\frac{1}{T}\sum_{t=1}^{T}\text{KL}(\pi_{\bm{\theta}_{k}}(\cdot|s_{t}),\pi_{{\bm{\theta}}}(\cdot|s_{t}))}\Bigg\};
\end{flalign}
Projection
 \newline
\begin{flalign}
\nonumber
{\bm{\theta}}_{k+1}=\arg\min_{{\bm{\theta}}} \dfrac{1}{T}\sum_{t=1}^{T}
\bigg\{
\text{KL}\left({\pi_{\bm{\theta}_{k+\frac{1}{2}}}}(\cdot|s_{t}),\pi_{{\bm{\theta}}}(\cdot|s_{t})\right)\\
\nonumber
+\nu_k
\dfrac{1-\gamma\lambda}{1-\gamma}\dfrac{\pi_{\bm{\theta}}(a_{t}|s_{t})}{\pi_{{\bm{\theta}}_k}(a_{t}|s_{t})}\hat{A}^{C}_{t}
\bigg\}.
\end{flalign}
    &\makecell{Non-Convex\\ Implementation}
    \\
    \bottomrule
 \end{tabularx}
  \end{adjustbox}
 \label{app-table-com}
\end{table}
\end{landscape}

\clearpage

\section{Preliminaries}

In this section, we introduce some new notations about state distribution, policy optimization and $\lambda$-returns.

\subsection{State Distribution}
We use $\mathbf{P}_{\pi_{\bm \theta}}\in\R^{|\calS|\times|\calS|}$ to denote the state transition matrix by executing $\policy$, and their components are:
\[
\mathbf{P}_{\pi_{\bm \theta}}[s,s'] =\sum_{a\in\mathcal{A}}\pi_{\bm{\theta}}(a|s)\mathbb{P}(s'|s,a)=:\Pro_{\policy}(s^{'}|s),~~s,s^{'}\in\calS,
\]
which denotes one-step state transformation probability from $s$ to $s^{'}$.

We use $\mathbb{P}_{\policy}(s_t=s|s_0)$ to denote the probability of visiting $s$ after $t$
time steps from the initial state $s_0$ by executing $\policy$.
Particularly, we notice if $t=0$, $s_t\ne s_0$, then $\mathbb{P}_{\policy}(s_t=s|s_0)=0$, i.e.,
\begin{flalign}
\label{special-inititial-pro}
\mathbb{P}_{\policy}(s_t=s|s_0)=0,~~ t=0~\text{and}~s\ne s_0.
\end{flalign}
Then for any initial state $s_0\sim\rho(\cdot)$, the following holds,
\begin{flalign}
\label{pro-pi-t-step-app}
\mathbb{P}_{{\policy}}(s_t=s|s_0)=&\sum_{s^{'}\in\mathcal{S}}\mathbb{P}_{{\policy}}(s_t=s|s_{t-1}=s^{'})\mathbb{P}_{{\policy}}(s_{t-1}=s^{'}|s_0).
\end{flalign}

Recall $d_{\pi_{\bm {\theta}}}^{s_0}(s)$ denotes the normalized discounted distribution of the future state $s$ encountered starting at $s_0$ by executing $\pi_{\bm {\theta}}$,
\[d_{\pi_{\bm {\theta}}}^{s_0}(s)=(1-\gamma)\sum_{t=0}^{\infty}\gamma^{t}\mathbb{P}_{\pi_{\bm {\theta}}}(s_t=s|s_0).\]
Furthermore, since $s_0\sim\rho_{0}(\cdot)$, we define
\[
d_{\pi_{\bm {\theta}}}^{\rho_0}(s)=\mathbb{E}_{s_0\sim\rho_{0}(\cdot)}[d_{\pi_{\bm {\theta}}}^{s_0}(s)]=\int_{s_0\in\calS}\rho_{0}(s_0)d^{s_0}_{\pi_{\bm {\theta}}}(s)\text{d}s_0
\]
as the discounted state visitation distribution over the initial distribution $\rho_0(\cdot)$. We use $\bd_{\pi_{\bm {\theta}}}^{\rho_0}\in\R^{|\calS|}$ to store all the normalized discounted state distributions, and its components are:
\[
\bd_{\pi_{\bm {\theta}}}^{\rho_0}[s]=d_{\pi_{\bm {\theta}}}^{\rho_0}(s),~~s\in\calS.
\]

We use $\bm{\rho}_0\in\R^{|\calS|}$ to denote initial state distribution vector, and their components are:
\[
\bm{\rho}_{0}[s]=\rho_{0}(s),~~s\in\calS.
\]
Then, we rewrite $\bd_{\pi_{\bm {\theta}}}^{\rho_0}$ as the following matrix version,
\begin{flalign}
\bd_{\pi_{\bm {\theta}}}^{\rho_0}=(1-\gamma)\sum_{t=0}^{\infty}(\gamma\bP_{\policy})^{t}\bm{\rho}_{0}=(1-\gamma)(\bI-\gamma\bP_{\policy})^{-1}\bm{\rho}_{0}.
\end{flalign}

\subsection{Objective of MDP}

Recall $\tau=\{s_{t}, a_{t}, r_{t+1}\}_{t\ge0}\sim{\pi_{\bm \theta}}$, according to $\tau$,
we define the expected return $J({\pi_{\bm \theta}}|s_0)$ as follows,
\begin{flalign}
\label{Eq:J-theta-app}
  J({\pi_{\bm \theta}}|s_0)=&\mathbb{E}_{\tau\sim\pi_{\bm {\theta}}}[R(\tau)]
    =\dfrac{1}{1-\gamma}\mathbb{E}_{s\sim d_{\pi_{\bm {\theta}}}^{s_0}(\cdot),a\sim\pi_{\bm {\theta}}(\cdot|s),s^{'}\sim\Pro(\cdot|s,a)}\left[r(s^{'}|s,a)\right],
\end{flalign}
where $R(\tau)=\sum_{t\ge0}\gamma^{t}r_{t+1}$, and the notation $J({\pi_{\bm \theta}}|s_0)$ is ``conditional'' on $s_0$ is to emphasize the trajectory $\tau$ starting from $s_0$.
Since $s_0\sim\rho_{0}(\cdot)$, we define the objective of MDP as follows,
\begin{flalign}
\label{app-objective-fun-01}
J(\pi_{\bm {\theta}})
=\dfrac{1}{1-\gamma}\mathbb{E}_{s\sim d_{\pi_{\bm {\theta}}}^{\rho_0}(\cdot),a\sim\pi_{\bm {\theta}}(\cdot|s),s^{'}\sim\Pro(\cdot|s,a)}\left[r(s^{'}|s,a)\right].
\end{flalign}
The goal of reinforcement learning is to solve the following optimization problem:
\begin{flalign}
\label{Eq:thata-optimal-app}
    \bm{\theta}_{\star}=\arg\max_{\bm{\theta}\in\R^{p}} J(\policy).
\end{flalign}

\subsection{Bellman Operator}

Let $\mathcal{B}_{\pi_{\bm\theta}}$ be the \emph{Bellman operator}:
 \begin{flalign}
 \label{bellman-op}
\mathcal{B}_{\pi_{\bm\theta}}:  \mathbb{R}^{|\mathcal{S}|}\rightarrow \mathbb{R}^{|\mathcal{S}|},~~~~ v\mapsto \mathbf{r}_{\pi_{\bm \theta}}+\gamma \mathbf{P}_{\pi_{\bm \theta}}v,
 \end{flalign}
 where
$\mathbf{r}_{\pi_{\bm \theta}}\in\R^{|\calS|}$ is the expected reward according to $\pi_{\bm \theta}$, i.e., their components are: \[\mathbf{r}_{\pi_{\bm \theta}}[s] =\sum_{a\in\mathcal{A}}\sum_{s^{'}\in\mathcal{S}}\pi_{\bm{\theta}}(a|s)r(s'|s,a)=:R_{\policy}(s),~~s\in\calS.\]
Let $\bv_{\pi_{\bm{\theta}}}\in\R^{|\calS|}$ be a vector that stores all the state value functions, and its components are:
$
\bv_{\pi_{\bm{\theta}}}[s]=V_{\pi_{\bm{\theta}}}(s),~~s\in\calS.
$
Then, according to Bellman operator (\ref{bellman-op}), we rewrite Bellman equation \cite{bellman1957markovian} as the following matrix version:
\begin{flalign}
\label{bellman-eq-matrix}
\mathcal{B}_{\pi_{\bm\theta}}\bv_{\pi_{\bm{\theta}}}=\bv_{\pi_{\bm{\theta}}}.
\end{flalign}

Furthermore, we define $\lambda$-\emph{Bellman operator} $\mathcal{B}^{\lambda}_{\pi_{\bm\theta}}$ as follows,
\[
\mathcal{B}^{\lambda}_{\pi_{\bm\theta}}=(1-\lambda)\sum_{t=0}^{\infty}\lambda^{t} (\mathcal{B}_{\pi_{\bm\theta}})^{{t}+1},
\]
which implies
 \begin{flalign}
 \label{bellman-op-lam}
\mathcal{B}^{\lambda}_{\pi_{\bm\theta}}: \mathbb{R}^{|\mathcal{S}|}\rightarrow \mathbb{R}^{|\mathcal{S}|},~~~~ 
 v\mapsto \mathbf{r}^{(\lambda)}_{\pi_{\bm \theta}}+\tilde{\gamma}\mathbf{P}^{(\lambda)}_{\pi_{\bm \theta}}v,
 \end{flalign}
 where 
 \begin{flalign}
\label{def:matrix-p-lam-return}
 \mathbf{P}^{(\lambda)}_{\pi_{\bm \theta}}=(1-\gamma\lambda)\sum_{{t}=0}^{\infty}(\gamma\lambda)^{{t}}\bP^{{t}+1}_{\policy},~~ \mathbf{r}^{(\lambda)}_{\pi_{\bm \theta}}=\sum_{{t}=0}^{\infty}(\gamma\lambda\bP_{\policy})^{{t}}\mathbf{r}_{\pi_{\bm \theta}},~~\tilde{\gamma}=\dfrac{\gamma(1-\lambda)}{1-\gamma\lambda}.
 \end{flalign}

 Let
 \begin{flalign}
 \label{lam-pro-value}
 \Pro_{\policy}^{(\lambda)}(s^{'}|s)=\mathbf{P}^{(\lambda)}_{\pi_{\bm \theta}}[s,s^{'}]=:(1-\gamma\lambda)\sum_{{t}=0}^{\infty}(\gamma\lambda)^{{t}}\left(\bP^{{t}+1}_{\policy}[s,s^{'}]\right),
 \end{flalign}
 where $\bP^{{t}+1}_{\policy}[s,s^{'}]$ is the $(s,s^{'})$-th component of matrix $\bP^{{t}+1}_{\policy}$, which is the probability of visiting $s^{'}$ after $t+1$ time steps from
 the state $s$ by executing $\policy$, i.e.,
  \begin{flalign}
 \label{lam-pro-value-01}
 \bP^{{t}+1}_{\policy}[s,s^{'}]= \Pro_{\policy}(s_{t+1}=s^{'}|s).
\end{flalign}
Thus, we rewrite  $\Pro_{\policy}^{(\lambda)}(s^{'}|s)$ (\ref{lam-pro-value}) as follows
 \begin{flalign}
 \label{lam-pro-value-02}
 \Pro_{\policy}^{(\lambda)}(s^{'}|s)=(1-\gamma\lambda)\sum_{{t}=0}^{\infty}(\gamma\lambda)^{{t}}\Pro_{\policy}(s_{t+1}=s^{'}|s),~~s\in\calS.
 \end{flalign}

\subsection{$\lambda$-Return}
 Furthermore,
recall the following visitation sequence $\tau=\{s_{t}, a_{t}, r_{t+1}\}_{t\ge0}$ induced by $\policy$,
it is similar to the probability $\mathbb{P}_{{\policy}}(s_t=s^{'}|s_0)$, we introduce $\mathbb{P}^{(\lambda)}_{{\policy}}(s_t=s^{'}|s_0)$
as the probability of  transition from state $s$ to state $s^{'}$after $t$
time steps under the dynamic transformation matrix $ \mathbf{P}^{(\lambda)}_{\pi_{\bm \theta}}$.
Then, the following equity holds
\begin{flalign}
\label{pro-pi-t-step}
\mathbb{P}_{{\policy}}^{(\lambda)}(s_t=s|s_0)=&\sum_{s^{'}\in\mathcal{S}}\mathbb{P}_{{\policy}}^{(\lambda)}(s_t=s|s_{t-1}=s^{'})\mathbb{P}_{{\policy}}^{(\lambda)}(s_{t-1}=s^{'}|s_0).
\end{flalign}
Similarly, let
  \begin{flalign}
  \nonumber
 R^{(\lambda)}_{\pi_{\bm \theta}}(s)=:
  \mathbf{r}^{(\lambda)}_{\pi_{\bm \theta}}[s]=&\sum_{{t}=0}^{\infty}(\gamma\lambda\bP_{\policy})^{{t}}\mathbf{r}_{\pi_{\bm \theta}}[s]
  =  \sum_{{t}=0}^{\infty}(\gamma\lambda)^{t}\left(\sum_{s^{'}\in\calS}\Pro_{\policy}(s_{t}=s^{'}|s)R_{\policy}(s^{'})\right)\\
   \label{lam-pro-value-03}
   =&
     \sum_{{t}=0}^{\infty}\sum_{s^{'}\in\calS}(\gamma\lambda)^{t}\Pro_{\policy}(s_{t}=s^{'}|s)R_{\policy}(s^{'}).
 \end{flalign}
 It is similar to normalized discounted distribution $d_{\pi_{\bm {\theta}}}^{\rho_0}(s)$, we introduce $\lambda$-return version of discounted state distribution $d_{\pi_{\bm {\theta}}}^{\lambda}(s)$ as follows: $\forall s\in\calS$,
\begin{flalign}
\label{lambda-dis-state-distribution}
d_{\pi_{\bm {\theta}}}^{s_0,\lambda}(s)&=(1-\tilde \gamma)\sum_{t=0}^{\infty}\tilde{\gamma}^{t}\mathbb{P}^{(\lambda)}_{\pi_{\bm {\theta}}}(s_t=s|s_0),\\
d_{\pi_{\bm {\theta}}}^{\lambda}(s)&=\E_{s_0\sim\rho_{0}(\cdot)}\left[d_{\pi_{\bm {\theta}}}^{s_0,\lambda}(s)\right],\\
\label{mat-lambda-dis-state-distribution}
\bd_{\pi_{\bm {\theta}}}^{\lambda}[s]&=d_{\pi_{\bm {\theta}}}^{\lambda}(s),
\end{flalign}
where $\mathbb{P}^{(\lambda)}_{\pi_{\bm {\theta}}}(s_t=s|s_0)$ is the $(s_0,s)$-th component of the matrix $\left(\mathbf{P}^{(\lambda)}_{\pi_{\bm \theta}}\right)^{t}$, i.e.,
\[
\mathbb{P}^{(\lambda)}_{\pi_{\bm {\theta}}}(s_t=s|s_0)=:\left(\mathbf{P}^{(\lambda)}_{\pi_{\bm \theta}}\right)^{t}[s_0,s].
\]
Similarly, $\mathbb{P}^{(\lambda)}_{\pi_{\bm {\theta}}}(s_t=s^{'}|s)$ is the $(s,s^{'})$-th component of the matrix $\left(\mathbf{P}^{(\lambda)}_{\pi_{\bm \theta}}\right)^{t}$, i.e.,
\[
\mathbb{P}^{(\lambda)}_{\pi_{\bm {\theta}}}(s_t=s^{'}|s)=:\left(\mathbf{P}^{(\lambda)}_{\pi_{\bm \theta}}\right)^{t}[s,s^{'}].
\]

Finally, we rewrite $\bd_{\pi_{\bm {\theta}}}^{\rho_0,\lambda}$ as the following matrix version,
\begin{flalign}
\label{matrixversion-lambda-dis-state-distribution}
\bd_{\pi_{\bm {\theta}}}^{\lambda}=(1-\tilde \gamma)\sum_{t=0}^{\infty}\left(\gamma\bP^{(\lambda)}_{\policy}\right)^{t}\bm{\rho}_{0}=(1-\tilde \gamma)\left(\bI-\tilde{\gamma}\bP^{(\lambda)}_{\policy}\right)^{-1}\bm{\rho}_{0}.
\end{flalign}
 \begin{remark}[$\lambda$-Return Version of Bellman Equation]
 According to Bellman equation (\ref{bellman-eq-matrix}), $\bv_{\pi_{\bm{\theta}}}$ is fixed point of $\lambda$-operator $\mathcal{B}^{\lambda}_{\pi_{\bm\theta}}$, i.e.,
 \begin{flalign}
 \label{bellman-eq-return}
 \bv_{\pi_{\bm{\theta}}}=\mathbf{r}^{(\lambda)}_{\pi_{\bm \theta}}+{\tilde{\gamma}} \mathbf{P}^{(\lambda)}_{\pi_{\bm \theta}}\bv_{\pi_{\bm{\theta}}}.
 \end{flalign}
 Recall $\tau=\{s_t,a_t,r_{t+1}\}_{t\ge0}\sim\policy$, according to (\ref{bellman-eq-return}), the value function of initial state $s_0$ is
 \begin{flalign}
  \label{bellman-eq-1}
  V_{\policy}(s_0)= \bv_{\pi_{\bm{\theta}}}[s_0]=\mathbf{r}^{(\lambda)}_{\pi_{\bm \theta}}[s_0]+\tilde \gamma \mathbf{P}^{(\lambda)}_{\pi_{\bm \theta}}\bv_{\pi_{\bm{\theta}}}[s_0]
 =R^{(\lambda)}_{\pi_{\bm \theta}}(s_0)+{\tilde{\gamma}}\sum_{s^{'}\in\calS} \Pro_{\policy}^{(\lambda)}(s_1=s^{'}|s_0)V_{\policy}(s^{'}).
 \end{flalign}
 \end{remark}

We unroll the expression of (\ref{bellman-eq-1}) repeatedly, then we have
\begin{flalign}
\nonumber
V_{{\policy}}(s_0)
=&{R}^{(\lambda)}_{\policy}(s_0)+{\tilde{\gamma}}\sum_{s^{'}\in\mathcal{S}}\mathbb{P}_{{\policy}}^{(\lambda)}(s_1=s^{'}|s_0)
\underbrace{\left(
{R}^{(\lambda)}_{\policy}(s^{'})
+{\tilde{\gamma}}\sum_{s^{''}\in\mathcal{S}}\mathbb{P}_{{\policy}}^{(\lambda)}(s_2=s^{''}|s_1=s^{'})V_{{\policy}}(s^{''})
\right)}_{=V_{{\policy}}(s^{'})}\\
\nonumber
=&{R}^{(\lambda)}_{\policy}(s_0)+{\tilde{\gamma}}\sum_{s^{'}\in\mathcal{S}}\mathbb{P}_{{\policy}}^{(\lambda)}(s_1=s^{'}|s_0){R}^{(\lambda)}_{\policy}(s^{'})\\
\nonumber
&~~~~~~~~~~~~~+{\tilde{\gamma}}^2\sum_{s^{''}\in\mathcal{S}}\underbrace{
\left(
\sum_{s^{'}\in\mathcal{S}}\mathbb{P}_{{\policy}}^{(\lambda)}(s_1=s^{'}|s_0)
\mathbb{P}_{{\policy}}^{(\lambda)}(s_2=s^{''}|s_1=s^{'})
\right)
}_{\overset{(\ref{pro-pi-t-step})}=:\mathbb{P}^{(\lambda)}_{{\policy}}\left(s_2=s^{''}|s_0\right)}V_{{\policy}}(s^{''})\\
\nonumber
=&{R}^{(\lambda)}_{\policy}(s_0)+{\tilde{\gamma}}\sum_{s\in\mathcal{S}}\mathbb{P}_{{\policy}}^{(\lambda)}(s_1=s|s_0){R}^{(\lambda)}_{\policy}(s)
+{\tilde{\gamma}}^2\sum_{s\in\mathcal{S}}\mathbb{P}_{{\policy}}^{(\lambda)}(s_2=s|s_0)V_{{\policy}}(s)\\
\nonumber
=&{R}^{(\lambda)}_{\policy}(s_0)+{\tilde{\gamma}}\sum_{s\in\mathcal{S}}\mathbb{P}_{{\policy}}^{(\lambda)}(s_1=s|s_0){R}^{(\lambda)}_{\policy}(s)\\
\nonumber
&~~~~~~~~~~~~~+{\tilde{\gamma}}^2\sum_{s\in\mathcal{S}}\mathbb{P}_{{\policy}}^{(\lambda)}(s_2=s|s_0)
\left(
{R}^{(\lambda)}_{\policy}(s)+{\tilde{\gamma}}\sum_{s^{'}\in\mathcal{S}}\mathbb{P}_{{\policy}}^{(\lambda)}(s_3=s^{'}|s_2=s)V_{{\policy}}(s^{'})
\right)\\
\nonumber
=&{R}^{(\lambda)}_{\policy}(s_0)+{\tilde{\gamma}}\sum_{s\in\mathcal{S}}\mathbb{P}_{{\policy}}^{(\lambda)}(s_1=s|s_0){R}^{(\lambda)}_{\policy}(s)+{\tilde{\gamma}}^2\sum_{s\in\mathcal{S}}\mathbb{P}_{{\policy}}^{(\lambda)}(s_2=s|s_0){R}^{(\lambda)}_{\policy}(s)\\
\nonumber
&~~~~~~~~~~~~~+{\tilde{\gamma}}^3\sum_{s^{'}\in\mathcal{S}}\underbrace{
\left(
\sum_{s\in\mathcal{S}}\mathbb{P}_{{\policy}}^{(\lambda)}(s_2=s|s_0)
\mathbb{P}_{{\policy}}^{(\lambda)}(s_3=s^{'}|s_2=s)
\right)
}_{=\mathbb{P}_{{\policy}}^{(\lambda)}(s_3=s^{'}|s_0)}V_{{\policy}}(s^{'})\\
\nonumber
=&{R}^{(\lambda)}_{}(s_0)+{\tilde{\gamma}}\sum_{s\in\mathcal{S}}\mathbb{P}_{{\policy}}^{(\lambda)}(s_1=s|s_0){R}^{(\lambda)}_{\policy}(s)+{\tilde{\gamma}}^2\sum_{s\in\mathcal{S}}\mathbb{P}_{{\policy}}^{(\lambda)}(s_2=s|s_0){R}^{(\lambda)}_{\policy}(s)\\
\nonumber
&~~~~~~~~~~~~~+{\tilde{\gamma}}^3\sum_{s\in\mathcal{S}}\mathbb{P}^{(\lambda)}_{{\policy}}(s_3=s|s_0)V_{{\policy}}(s)\\
\nonumber
=&\cdots
\\
\label{re-bellman-eq-01}
=&\sum_{s\in\calS}\sum_{t=0}^{\infty}{\tilde{\gamma}}^{t}\mathbb{P}_{{\policy}}^{(\lambda)}(s_t=s|s_0){R}^{(\lambda)}_{\policy}(s)
\overset{(\ref{lambda-dis-state-distribution})}=\dfrac{1}{1-{\tilde{\gamma}}}\sum_{s\in\calS}d^{s_0,\lambda}_{\policy}(s) {R}^{(\lambda)}_{\policy}(s)
.
\end{flalign}

According to (\ref{Eq:J-theta-app}) and (\ref{re-bellman-eq-01}), we have
\begin{flalign}
\nonumber
J({\policy})=&\sum_{s_0\in\calS}\rho_{0}(s_0)V_{\policy}(s_0)\overset{(\ref{re-bellman-eq-01})}=
\dfrac{1}{1-{\tilde{\gamma}}}\sum_{s_0\in\calS}\rho_{0}(s_0)\sum_{s\in\calS}d^{s_0,\lambda}_{\policy}(s) R^{(\lambda)}_{\pi_{\bm \theta}}(s)\\
\nonumber
=&\dfrac{1}{1-{\tilde{\gamma}}}\sum_{s\in\calS}\underbrace{\left(\sum_{s_0\in\calS}\rho_{0}(s_0)d^{s_0,\lambda}_{\policy}(s)\right)}_{=d^{\lambda}_{\policy}(s)} R^{(\lambda)}_{\pi_{\bm \theta}}(s)\\
\label{lam-return-objective}
=&\dfrac{1}{1-{\tilde{\gamma}}}\sum_{s\in\calS}d^{\lambda}_{\policy}(s)R^{(\lambda)}_{\pi_{\bm \theta}}(s)
=
\dfrac{1}{1-{\tilde{\gamma}}}\E_{s\sim d^{\lambda}_{\policy}(\cdot)}
\left[R^{(\lambda)}_{\pi_{\bm \theta}}(s)\right]
.
\end{flalign}

Finally, we summarize above results in the following Lemma \ref{lem:lam-return-objective}.
\begin{lemma}
\label{lem:lam-return-objective}
The objective $J(\policy)$ (\ref{app-objective-fun-01}) can be rewritten as the following version:
\[
J({\policy})=\dfrac{1}{1-{\tilde{\gamma}}}\sum_{s\in\calS}d^{\lambda}_{\policy}(s)R^{(\lambda)}_{\pi_{\bm \theta}}(s)=
\dfrac{1}{1-{\tilde{\gamma}}}\E_{s\sim d^{\lambda}_{\policy}(\cdot)}
\left[R^{(\lambda)}_{\pi_{\bm \theta}}(s)\right].
\]
\end{lemma}

\clearpage

\section{Proof of Theorem \ref{them:general-performance-difference}}

\label{sec:proof-them-01}

We need the following Proposition \ref{objective-td-error-version} to prove Theorem \ref{them:general-performance-difference}, which illustrates an identity for the objective function of policy optimization.

\begin{proposition}
\label{objective-td-error-version}
For any function $\varphi(\cdot):\calS\rightarrow\R$, for any policy $\policy$, for any trajectory satisfies $\tau=\{s_{t}, a_{t}, r_{t+1}\}_{t\ge0}\sim{\pi_{\bm \theta}}$,
let 
\begin{flalign}
\nonumber
\delta_t^{\varphi}&=r(s_{t+1}|s_t,a_t)+\gamma\varphi(s_{t+1})-\varphi(s_{t}),
\\
\nonumber
\delta^{\varphi}_{\policy,t}(s)&=\E_{s_{t}\sim\Pro_{\policy}(\cdot|s),a_{t}\sim{\policy}(\cdot|s_t),s_{t+1}\sim\Pro(\cdot|s_t,a_t)}\left[\delta_t^{\varphi}\right],
\end{flalign}
then, the objective $J(\policy)$ (\ref{lam-return-objective}) can be rewritten as the following version:
\begin{flalign}
\label{lam-return-phi-objective-prop}
J(\policy)=&\E_{s_0\sim\rho_{0}(\cdot)}[\varphi(s_0)]
+
\dfrac{1}{1-\tilde\gamma}\sum_{s\in\calS}d^{\lambda}_{\policy}(s)
\left(
\sum_{t=0}^{\infty}\gamma^t \lambda^t
\delta^{\varphi}_{\policy,t}(s)
\right)
\\
\nonumber
=&\E_{s_0\sim\rho_{0}(\cdot)}[\varphi(s_0)]
+
\dfrac{1}{1-\tilde\gamma}\E_{s\sim d^{\lambda}_{\policy}(\cdot)}
\left[
\sum_{t=0}^{\infty}\gamma^t \lambda^t
\delta^{\varphi}_{\policy,t}(s)
\right]
.
\end{flalign}
\end{proposition}
We present the proof of of Proposition \ref{objective-td-error-version} at the end of this section, see Section \ref{proof-pro-01-app}.

We introduce a vector $\bm{\delta}^{\varphi}_{\policy,t}\in\R^{|\calS|}$ and its components are: for any $s\in\calS$
\begin{flalign}
\label{revist-td-ex-error}
\bm{\delta}^{\varphi}_{\policy,t}[s]={{\delta}}^{\varphi}_{\policy,t}(s).
\end{flalign}
Then, we rewrite the objective as the following vector version
\begin{flalign}
\label{lam-return-phi-objective-vec-version}
J(\policy)=\E_{s_0\sim\rho_{0}(\cdot)}[\varphi(s_0)]
+
\dfrac{1}{1-\tilde \gamma}\sum_{t=0}^{\infty}\gamma^t \lambda^t
\langle
\bd_{\pi_{\bm {\theta}}}^{\lambda},\bm{\delta}^{\varphi}_{\policy,t}
\rangle,
\end{flalign}
where $\langle\cdot,\cdot\rangle$ denotes inner production between two vectors.

\subsection{Proof of Theorem \ref{them:general-performance-difference}}

\textbf{Theorem \ref{them:general-performance-difference}} 
(Generalized Policy Performance Difference)
\emph{
For any function $\varphi(\cdot):\calS\rightarrow\R$, for two arbitrary policy $\pi_{\bm\theta}$ and $\pi_{{\bm\theta}^{'}}$,  
for any $p,q\in[1,\infty)$ such that $\frac{1}{p}+\frac{1}{q}=1$,  
The following bound holds:
\begin{flalign} 
\label{bound-diff}
\dfrac{1}{1-\tilde \gamma}\sum_{t=0}^{\infty}\gamma^t\lambda^{t} M^{\varphi,-}_{p,q,t}(\policy,\policyy)
\leq J(\pi_{\bm \theta})-J(\pi_{{\bm \theta}^{'}})
\leq\dfrac{1}{1-\tilde \gamma}\sum_{t=0}^{\infty}\gamma^t \lambda^t  M^{\varphi,+}_{p,q,t}(\policy,\policyy),
\end{flalign}
where the terms $M^{\varphi,-}_{p,q,t}$ and $M^{\varphi,+}_{p,q,t}$ are defined in (\ref{app-term-01})-(\ref{app-term-02}).
}

\begin{proof} (of Theorem \ref{them:general-performance-difference})

We consider two arbitrary policies $\pi_{\bm\theta}$ and $\pi_{{\bm\theta}^{'}}$ with different parameters $\bm{\theta}$ and $\bm{\theta}^{'}$, let
\begin{flalign}
\label{dfifference-two-polic}
D_{t}^{\varphi,(\lambda)}(\policy,\policyy)=:
\langle
\bd_{\pi_{\bm {\theta}}}^{\lambda},\bm{\delta}^{\varphi}_{\policy,t}
\rangle
-
\langle
\bd_{\policyy}^{\lambda},\bm{\delta}^{\varphi}_{\policyy,t}
\rangle.
\end{flalign}
According to (\ref{lam-return-phi-objective-vec-version}), we obtain performance difference as follows,
\begin{flalign} 
\nonumber
J(\pi_{\bm \theta})-J(\pi_{{\bm \theta}^{'}})=&\dfrac{1}{1-\tilde \gamma}\sum_{t=0}^{\infty}\gamma^t \lambda^t
\left(
\langle
\bd_{\pi_{\bm {\theta}}}^{\lambda},\bm{\delta}^{\varphi}_{\policy,t}
\rangle
-
\langle
\bd_{\policyy}^{\lambda},\bm{\delta}^{\varphi}_{\policyy,t}
\rangle
\right)
\\
\label{objective-difference-01-app}
=&\dfrac{1}{1-\tilde\gamma}\sum_{t=0}^{\infty}\gamma^t \lambda^tD_{t}^{\varphi,(\lambda)}(\policy,\policyy),
\end{flalign}
which requires us to consider the boundedness of the difference $D_{t}^{\varphi,(\lambda)}(\policy,\policyy)$ (\ref{dfifference-two-polic}) .

{
\color{blue}
{
\textbf{Step 1: Bound the term $D_{t}^{\varphi,(\lambda)}(\policy,\policyy)$ (\ref{dfifference-two-polic}).}
}
}

We rewrite the first term of (\ref{dfifference-two-polic}) as:
\begin{flalign}
\label{td-error-ex-04}
\langle \bd_{\pi_{\bm {\theta}}}^{\lambda},{\bm{\delta}}^{\varphi}_{\policy,t}\rangle
=
\langle \bd_{\policyy}^{\lambda},{\bm{\delta}}^{\varphi}_{\policy,t}\rangle
+\langle \bd_{\pi_{\bm {\theta}}}^{\lambda}-\bd_{\policyy}^{\lambda},{\bm{\delta}}^{\varphi}_{\policy,t}\rangle,
\end{flalign}
which is bounded by applying H{\"o}lder's inequality to the term $\langle \bd_{\pi_{\bm {\theta}}}^{\lambda}-\bd_{\policyy}^{\lambda},{\bm{\delta}}^{\varphi}_{\policy,t}\rangle$, we rewrite (\ref{td-error-ex-04}) as:
\begin{flalign}
\label{them:ineq-01}
\langle \bd_{\policyy}^{\lambda},{\bm{\delta}}^{\varphi}_{\policy,t}\rangle-\|\bd_{\pi_{\bm {\theta}}}^{\lambda}-\bd_{\policyy}^{\lambda}\|_{p}\|{\bm{\delta}}^{\varphi}_{\policy,t}\|_{q}
\leq
\langle \bd_{\pi_{\bm {\theta}}}^{\lambda},{\bm{\delta}}^{\varphi}_{\policy,t}\rangle
\leq
\langle \bd_{\policyy}^{\lambda},{\bm{\delta}}^{\varphi}_{\policy,t}\rangle
+
\|\bd_{\pi_{\bm {\theta}}}^{\lambda}-\bd_{\policyy}^{\lambda}\|_{p}\|{\bm{\delta}}^{\varphi}_{\policy,t}\|_{q}
,
\end{flalign}
where $p,q\in[1,\infty)$ and $\frac{1}{p}+\frac{1}{q}=1$. Let
\[
\epsilon^{\varphi,(\lambda)}_{p,q,t}(\policy,\policyy)=:\|\bd_{\pi_{\bm {\theta}}}^{\lambda}-\bd_{\policyy}^{\lambda}\|_{p}\|{\bm{\delta}}^{\varphi}_{\policy,t}\|_{q},
\]
then we rewrite Eq.(\ref{them:ineq-01}) as follows, 
\begin{flalign}
\label{them:ineq-01-01}
\langle \bd_{\policyy}^{\lambda},{\bm{\delta}}^{\varphi}_{\policy,t}\rangle-\epsilon^{\varphi,(\lambda)}_{p,q,t}(\policy,\policyy)
\leq
\langle \bd_{\pi_{\bm {\theta}}}^{\lambda},{\bm{\delta}}^{\varphi}_{\policy,t}\rangle
\leq
\langle \bd_{\policyy}^{\lambda},{\bm{\delta}}^{\varphi}_{\policy,t}\rangle
+
\epsilon^{\varphi,(\lambda)}_{p,q,t}(\policy,\policyy).
\end{flalign}
Let
\begin{flalign}
\label{def:l-t}
M_{t}^{\varphi}(\policy,\policyy)=:
\underbrace{
\langle \bd_{\policyy}^{\lambda},{\bm{\delta}}^{\varphi}_{\policy,t}\rangle
}_{\text{Term-I}}
-
\underbrace{
\langle
\bd_{\policyy}^{\lambda},\bm{\delta}^{\varphi}_{\policyy,t}
\rangle
}_{\text{Term-II}}
,
\end{flalign}
combining the (\ref{dfifference-two-polic}) and (\ref{them:ineq-01-01}), we achieve the boundedness of $D_{t}^{\varphi}(\policy,\policyy)$ as follows
\begin{flalign}
\label{them:ineq-01-02}
M_{t}^{\varphi}(\policy,\policyy)-\epsilon^{\varphi,(\lambda)}_{p,q,t}(\policy,\policyy)
\leq
D_{t}^{\varphi}(\policy,\policyy)
\leq
M_{t}^{\varphi}(\policy,\policyy)+\epsilon^{\varphi,(\lambda)}_{p,q,t}(\policy,\policyy).
\end{flalign}

{
\color{blue}
{
\textbf{Step 2: Analyze the term $M_{t}^{\varphi}(\policy,\policyy)$ (\ref{def:l-t}).}
}
}

To analyze (\ref{them:ineq-01-02}) further, we need to consider the first term appears in $M_{t}^{\varphi}(\policy,\policyy)$ (\ref{def:l-t}):
\begin{flalign}
\nonumber
\text{Term-I}~(\ref{def:l-t})=&
\langle \bd_{\policyy}^{\lambda},{\bm{\delta}}^{\varphi}_{\policy,t}\rangle\\
\label{revist-inner}
=&
\sum_{s\in\calS}d_{\policyy}^{\lambda}(s){{\delta}}^{\varphi}_{\policy,t}(s)
=\E_{s\sim d_{\policyy}^{\lambda}(\cdot)}\left[{{\delta}}^{\varphi}_{\policy,t}(s)\right]
\\
\label{revist-inner-01}
\overset{(\ref{revist-td-ex-error})}=&\E_{s\sim d_{\policyy}^{\lambda}(\cdot)}
\left[\E_{s_{t}\sim\Pro_{\policy}(\cdot|s)}[
\delta^{\varphi}_{\policy}(s_t)]\right].
\end{flalign}
We notice the following relationship
\begin{flalign}
\label{importance-sam-re}
\delta^{\varphi}_{\policy,t}(s)=
&\underset{\begin{subarray}{c} s_{t}\sim\Pro_{\policy}(\cdot|s)\\a_{t}\sim{\policy}(\cdot|s_t)\\ s_{t+1}\sim\Pro(\cdot|s_t,a_t) \end{subarray}}\E\left[\delta_t^{\varphi}\right]
=\underset{\begin{subarray}{c}  s_{t}\sim\Pro_{\policyy}(\cdot|s)\\a_{t}\sim{\policyy}(\cdot|s_t)\\ s_{t+1}\sim\Pro(\cdot|s_t,a_t) \end{subarray}}\E\left[\dfrac{\policy(a_t|s_t)}{\policyy(a_t|s_t)}\delta_t^{\varphi}\right],
\end{flalign}
which holds since we use importance sampling: for any distribution $p(\cdot)$ and $q(\cdot)$, for any random variable function $f(\cdot)$,
\[
\E_{x\sim p(x)}[f(x)]=\E_{x\sim q(x)}\left[\dfrac{p(x)}{q(x)}f(x)\right].
\]
According to (\ref{revist-inner}), (\ref{importance-sam-re}), we rewrite the term $\langle \bd_{\policyy}^{\lambda},{\bm{\delta}}^{\varphi}_{\policy,t}\rangle$ in Eq.(\ref{def:l-t}) as follows,
\begin{flalign}
\label{app-ex-td-01}
\text{Term-I}~(\ref{def:l-t})=
\langle \bd_{\policyy}^{\lambda},{\bm{\delta}}^{\varphi}_{\policy,t}\rangle
=
\sum_{s\in\calS}d_{\policyy}^{\lambda}(s)\left(\underset{\begin{subarray}{c} s_t \sim \Pro_{\policyy}(\cdot|s)\\ a_{t}\sim{\policyy}(\cdot|s_t)\\ s_{t+1}\sim\Pro(\cdot|s_t,a_t) \end{subarray}}\E\left[\dfrac{\policy(a_t|s_t)}{\policyy(a_t|s_t)}\delta_t^{\varphi}\right]\right).
\end{flalign}

Now, we consider the second term appears in $M_{t}^{\varphi}(\policy,\policyy)$ (\ref{def:l-t}):
\begin{flalign}
\nonumber
\text{Term-II}~(\ref{def:l-t})&=
\langle
\bd_{\policyy}^{\lambda},\bm{\delta}^{\varphi}_{\policyy,t}
\rangle\\
\label{app-ex-td-02}
&=\sum_{s\in\calS}d_{\policyy}^{\lambda}(s){\delta}^{\varphi}_{\policyy,t}(s)
=\sum_{s\in\calS}d_{\policyy}^{\lambda}(s)\left(\underset{\begin{subarray}{c} s_t \sim \Pro_{\policyy}(\cdot|s)\\ a_{t}\sim{\policyy}(\cdot|s_t)\\ s_{t+1}\sim\Pro(\cdot|s_t,a_t) \end{subarray}}\E\left[\delta_t^{\varphi}\right]\right).
\end{flalign}

Finally, take the results (\ref{app-ex-td-01}) and (\ref{app-ex-td-02}) to (\ref{def:l-t}),
we obtain the difference between $\langle \bd_{\policyy}^{\lambda},{\bm{\delta}}^{\varphi}_{\policy,t}\rangle$ and $\langle\bd_{\policyy}^{\lambda},\bm{\delta}^{\varphi}_{\policyy,t}\rangle$, i.e., we achieve a identity for $M_{t}^{\varphi}(\policy,\policyy)$ (\ref{def:l-t}) as follows,
\begin{flalign}
\nonumber
M_{t}^{\varphi}(\policy,\policyy)\overset{(\ref{def:l-t})}=&
\langle \bd_{\policyy}^{\lambda},{\bm{\delta}}^{\varphi}_{\policy,t}\rangle-
\langle
\bd_{\policyy}^{\lambda},\bm{\delta}^{\varphi}_{\policyy,t}
\rangle
\\
\label{diff-01}
\overset{(\ref{app-ex-td-01},(\ref{app-ex-td-02})}=&
\sum_{s\in\calS}d_{\policyy}^{\lambda}(s)\left(\underset{\begin{subarray}{c} s_t \sim \Pro_{\policyy}(\cdot|s)\\ a_{t}\sim{\policyy}(\cdot|s_t)\\ s_{t+1}\sim\Pro(\cdot|s_t,a_t) \end{subarray}}\E\left[\left(\dfrac{\policy(a_t|s_t)}{\policyy(a_t|s_t)}-1\right)\delta_t^{\varphi}\right]\right).
\end{flalign}
To simplify expression, we introduce a notation as follows,
\begin{flalign}
\Delta_{t}^{\varphi}(\policy,\policyy,s)&=:\underset{\begin{subarray}{c} s_t \sim \Pro_{\policyy}(\cdot|s)\\ a_{t}\sim{\policyy}(\cdot|s_t)\\ s_{t+1}\sim\Pro(\cdot|s_t,a_t) \end{subarray}}\E\left[\left(\dfrac{\policy(a_t|s_t)}{\policyy(a_t|s_t)}-1\right)\delta_t^{\varphi}\right],
\end{flalign}
and we use a vector $\bm{\Delta}_{t}^{\varphi}(\policy,\policyy)\in\R^{|\calS|}$ to store all the values $\{\Delta_{t}^{\varphi}(\policy,\policyy,s)\}_{s\in\calS}$:
\[
\bm{\Delta}_{t}^{\varphi}(\policy,\policyy)[s]=\Delta_{t}^{\varphi}(\policy,\policyy,s).
\]
Then we rewrite $\langle \bd_{\policyy}^{\lambda},{\bm{\delta}}^{\varphi}_{\policy,t}\rangle-\langle\bd_{\policyy}^{\lambda},\bm{\delta}^{\varphi}_{\policyy,t}\rangle$ (\ref{diff-01}) as follows,
\begin{flalign}
\nonumber
M_{t}^{\varphi}(\policy,\policyy)=
\langle \bd_{\policyy}^{\lambda},{\bm{\delta}}^{\varphi}_{\policy,t}\rangle-
\langle
\bd_{\policyy}^{\lambda},\bm{\delta}^{\varphi}_{\policyy,t}
\rangle
\overset{(\ref{diff-01})}=
\sum_{s\in\calS}d_{\policyy}^{\lambda}(s)\Delta_{t}^{\varphi}(\policy,\policyy,s)
=
\langle \bd_{\policyy}^{\lambda},\bm{\Delta}_{t}^{\varphi}(\policy,\policyy)\rangle
.
\end{flalign}

{
\color{blue}
{
\textbf{Step 3: Bound on $J(\policy)-J(\policyy)$.}
}
}

Recall (\ref{them:ineq-01-02}), taking above result in it, we obtain
\begin{flalign}
\label{them:ineq-01-03}
\langle \bd_{\policyy}^{\lambda},\bm{\Delta}_{t}^{\varphi}(\policy,\policyy)\rangle-\epsilon^{\varphi,(\lambda)}_{p,q,t}(\policy,\policyy)
\leq
D_{t}^{\varphi}(\policy,\policyy)
\leq
\langle \bd_{\policyy}^{\lambda},\bm{\Delta}_{t}^{\varphi}(\policy,\policyy)\rangle+\epsilon^{\varphi,(\lambda)}_{p,q,t}(\policy,\policyy).
\end{flalign}

Finally, let
\begin{flalign}
\label{app-term-01}
&M^{\varphi,-}_{p,q,t}(\policy,\policyy)=
\langle \bd_{\policyy}^{\lambda},\bm{\Delta}_{t}^{\varphi}(\policy,\policyy)\rangle-\epsilon^{\varphi,(\lambda)}_{p,q,t}(\policy,\policyy)\\
\nonumber
=&\sum_{s\in\calS}d_{\policyy}^{\lambda}(s)\left(\underset{\begin{subarray}{c} s_t \sim \Pro_{\policyy}(\cdot|s)\\ a_{t}\sim{\policyy}(\cdot|s_t)\\ s_{t+1}\sim\Pro(\cdot|s_t,a_t) \end{subarray}}\E\left[\left(\dfrac{\policy(a_t|s_t)}{\policyy(a_t|s_t)}-1\right)\delta_t^{\varphi}\right]\right)
-
\|\bd_{\pi_{\bm {\theta}}}^{\lambda}-\bd_{\policyy}^{\lambda}\|_{p}\|{\bm{\delta}}^{\varphi}_{\policy,t}\|_{q}
\\
\nonumber
=&
\E_{s\sim{d}_{\policyy}^{\lambda}(\cdot)}
\left[
\underset{\begin{subarray}{c} s_t \sim \Pro_{\policyy}(\cdot|s)\\ a_{t}\sim{\policyy}(\cdot|s_t)\\ s_{t+1}\sim\Pro(\cdot|s_t,a_t) \end{subarray}}\E\left[\left(\dfrac{\policy(a_t|s_t)}{\policyy(a_t|s_t)}-1\right)\delta_t^{\varphi}\right]
\right]
-
\|\bd_{\pi_{\bm {\theta}}}^{\lambda}-\bd_{\policyy}^{\lambda}\|_{p}\|{\bm{\delta}}^{\varphi}_{\policy,t}\|_{q}.
\end{flalign}
and
\begin{flalign}
\label{app-term-02}
&M^{\varphi,+}_{p,q,t}(\policy,\policyy)=
\langle \bd_{\policyy}^{\lambda},\bm{\Delta}_{t}^{\varphi}(\policy,\policyy)\rangle+\epsilon^{\varphi,(\lambda)}_{p,q,t}(\policy,\policyy)\\
\nonumber
=&\sum_{s\in\calS}d_{\policyy}^{\lambda}(s)\left(\underset{\begin{subarray}{c} s_t \sim \Pro_{\policyy}(\cdot|s)\\ a_{t}\sim{\policyy}(\cdot|s_t)\\ s_{t+1}\sim\Pro(\cdot|s_t,a_t) \end{subarray}}\E\left[\left(\dfrac{\policy(a_t|s_t)}{\policyy(a_t|s_t)}-1\right)\delta_t^{\varphi}\right]\right)
+
\|\bd_{\pi_{\bm {\theta}}}^{\lambda}-\bd_{\policyy}^{\lambda}\|_{p}\|{\bm{\delta}}^{\varphi}_{\policy,t}\|_{q}
\\
\nonumber
=&
\E_{s\sim{d}_{\policyy}^{\lambda}(\cdot)}
\left[
\underset{\begin{subarray}{c} s_t \sim \Pro_{\policyy}(\cdot|s)\\ a_{t}\sim{\policyy}(\cdot|s_t)\\ s_{t+1}\sim\Pro(\cdot|s_t,a_t) \end{subarray}}\E\left[\left(\dfrac{\policy(a_t|s_t)}{\policyy(a_t|s_t)}-1\right)\delta_t^{\varphi}\right]
\right]
+
\|\bd_{\pi_{\bm {\theta}}}^{\lambda}-\bd_{\policyy}^{\lambda}\|_{p}\|{\bm{\delta}}^{\varphi}_{\policy,t}\|_{q}.
\end{flalign}

According to (\ref{objective-difference-01-app}) and (\ref{them:ineq-01-03}), we achieve the boundedness of performance difference between two arbitrary policies $\pi_{\bm\theta}$ and $\pi_{{\bm\theta}^{'}}$:
\begin{flalign} 
\label{objective-difference-app-004}
\underbrace{\dfrac{1}{1-\tilde \gamma}
\sum_{t=0}^{\infty}\gamma^t\lambda^{t} M^{\varphi,-}_{p,q,t}(\policy,\policyy)}_{=:L^{\varphi,-}_{p,q,}}
\leq J(\pi_{\bm \theta})-J(\pi_{{\bm \theta}^{'}})
\leq
\underbrace{
\dfrac{1}{1-\tilde \gamma}\sum_{t=0}^{\infty}\gamma^t \lambda^t  M^{\varphi,+}_{p,q,t}(\policy,\policyy)
}_{=:L^{\varphi,+}_{p,q,}}
.
\end{flalign}

\end{proof}

\subsection{Proof of Proposition \ref{objective-td-error-version}}
\label{proof-pro-01-app}

\begin{proof}
{
\color{blue}
{
\textbf{Step 1: Rewrite the objective $J(\policy)$ in Eq.(\ref{lam-return-objective}).}
}
}

We rewrite the discounted distribution $\bd_{\pi_{\bm {\theta}}}^{\lambda}$ (\ref{matrixversion-lambda-dis-state-distribution}) as follows,
\begin{flalign}
\label{vector:state-distribution}
\bm{\rho}_{0}-\dfrac{1}{1-{\tilde{\gamma}}}\bd_{\policy}^{\lambda}+\dfrac{{\tilde{\gamma}}}{1-{\tilde{\gamma}}}\bP^{(\lambda)}_{\policy}\bd_{\policy}^{\lambda}=\bm{0}.
\end{flalign}
Let $\varphi(\cdot)$ be a real number function defined on the state space $\calS$, i.e., $\varphi:\calS\rightarrow\R$.
Then we define a vector function $\bm{\phi}(\cdot)\in\R^{|\calS|}$ to collect all the values  $\{\varphi(s)\}_{s\in\calS}$, and its components are 
\[
\bm{\phi}[s]=\varphi(s),~~s\in\calS.
\]
Now, we take the inner product between the vector $\bm{\phi}$ and (\ref{vector:state-distribution}), we have
\begin{flalign}
\nonumber
0=&\langle \bm{\rho}_{0}-\dfrac{1}{1-{\tilde{\gamma}}}\bd_{\policy}^{\lambda}+\dfrac{{\tilde{\gamma}}}{1-{\tilde{\gamma}}}\bP^{(\lambda)}_{\policy}\bd_{\policy}^{\lambda},\bm{\phi} \rangle\\
\label{state-distribution-inner-initial-vec}
=&
\langle \bm{\rho}_{0},\bm{\phi}\rangle
-\dfrac{1}{1-{\tilde{\gamma}}}\langle \bd_{\policy}^{\lambda},\bm{\phi}\rangle
+
\dfrac{{\tilde{\gamma}}}{1-{\tilde{\gamma}}}\langle\bP^{(\lambda)}_{\policy}\bd_{\policy}^{\lambda},\bm{\phi}\rangle.
\end{flalign}
We express the first term $\langle \bm{\rho}_{0},\bm{\phi}\rangle$ of (\ref{state-distribution-inner-initial-vec}) as follows,
\begin{flalign}
\label{first-term}
\langle \bm{\rho}_{0},\bm{\phi}\rangle=\sum_{s\in\calS}\rho_{0}(s)\varphi(s)=\E_{s\sim\rho_{0}(\cdot)}[\varphi(s)].
\end{flalign}
We express the second term $\langle \bd_{\policy}^{\lambda},\bm{\phi}\rangle$ of (\ref{state-distribution-inner-initial-vec}) as follows,
\begin{flalign}
\nonumber
-\dfrac{1}{1-{\tilde{\gamma}}}\langle \bd_{\policy}^{\lambda},\bm{\phi}\rangle&=-\dfrac{1}{1-{\tilde{\gamma}}}\sum_{s\in\calS} d_{\policy}^{\lambda} (s)\varphi(s)\\
\label{sec-term}
&=-\dfrac{1}{1-{\tilde{\gamma}}}\E_{s\sim d_{\policy}^{\lambda} (\cdot)} [\varphi(s)].
\end{flalign}
We express the third term $\langle {\tilde{\gamma}}\bP^{(\lambda)}_{\policy}\bd_{\policy}^{\lambda},\bm{\phi}\rangle$ of (\ref{state-distribution-inner-initial-vec}) as follows,
\begin{flalign}
\nonumber
\dfrac{{\tilde{\gamma}}}{1-{\tilde{\gamma}}}\langle\bP^{(\lambda)}_{\policy}\bd_{\policy}^{\lambda},\bm{\phi}\rangle
=&\dfrac{{\tilde{\gamma}}}{1-{\tilde{\gamma}}}\sum_{s^{'}\in\calS}\left(\bP^{(\lambda)}_{\policy}\bd_{\policy}^{\lambda}\right)[s^{'}]\varphi(s^{'})\\
\label{third-term}
=&\dfrac{{\tilde{\gamma}}}{1-{\tilde{\gamma}}}
\sum_{s^{'}\in\calS}
\left(
\sum_{s\in\calS} \Pro_{\policy}^{(\lambda)}(s^{'}|s)d_{\policy}^{\lambda}(s)
\right)
\varphi(s^{'}).
\end{flalign}

According to Lemma \ref{lem:lam-return-objective}, put the results (\ref{lam-return-objective}) and (\ref{state-distribution-inner-initial-vec}) together, we have
\begin{flalign}
\nonumber
J(\policy)\overset{(\ref{lam-return-objective}),(\ref{state-distribution-inner-initial-vec})}=&\dfrac{1}{1-{\tilde{\gamma}}}\sum_{s\in\calS}d^{\lambda}_{\policy}(s)R^{(\lambda)}_{\pi_{\bm \theta}}(s)+\left\langle \bm{\rho}_{0}-\dfrac{1}{1-{\tilde{\gamma}}}\bd_{\policy}^{\lambda}+\dfrac{{\tilde{\gamma}}}{1-{\tilde{\gamma}}}\bP^{(\lambda)}_{\policy}\bd_{\policy}^{\lambda},\bm{\phi}\right \rangle\\
\label{lam-return-objective-01}
=&\E_{s_0\sim\rho_{0}(\cdot)}[\varphi(s_0)]
+
\dfrac{1}{1-{\tilde{\gamma}}}\sum_{s\in\calS}d^{\lambda}_{\policy}(s)
\left(
R^{(\lambda)}_{\pi_{\bm \theta}}(s)+{\tilde{\gamma}}
\sum_{s^{'}\in\calS} \Pro_{\policy}^{(\lambda)}(s^{'}|s)\varphi(s^{'})
-\varphi(s)
\right),
\end{flalign}
where the last equation holds since we unfold (\ref{state-distribution-inner-initial-vec}) according to (\ref{first-term})-(\ref{third-term}).

{
\color{blue}
{
\textbf{Step 2: Rewrite the term $\left(
R^{(\lambda)}_{\pi_{\bm \theta}}(s)+{\tilde{\gamma}}
\sum_{s^{'}\in\calS} \Pro_{\policy}^{(\lambda)}(s^{'}|s)\varphi(s^{'})
-\varphi(s)
\right)$ in Eq.(\ref{lam-return-objective-01}).}
}
}

Then, we unfold the second term of (\ref{lam-return-objective-01}) as follows,
\begin{flalign}
\label{app-04}
&R^{(\lambda)}_{\pi_{\bm \theta}}(s)+{\tilde{\gamma}}
\sum_{s^{'}\in\calS} \Pro_{\policy}^{(\lambda)}(s^{'}|s)\varphi(s^{'})
-\varphi(s)
\\
\nonumber
\overset{(\ref{lam-pro-value-02}),(\ref{lam-pro-value-03})}=&\sum_{{t}=0}^{\infty}({{\gamma}}\lambda\bP_{\policy})^{{t}}\mathbf{r}_{\pi_{\bm \theta}}[s]
+{\tilde{\gamma}}
(1-\gamma\lambda)\sum_{s^{'}\in\calS} \sum_{{t}=0}^{\infty}({{\gamma}}\lambda)^{{t}}\left(\bP^{{t}+1}_{\policy}[s,s^{'}]\right)\varphi(s^{'})
-\varphi(s)\\
\overset{(\ref{def:matrix-p-lam-return})}=&
\sum_{{t}=0}^{\infty}({{\gamma}}\lambda\bP_{\policy})^{{t}}\mathbf{r}_{\pi_{\bm \theta}}[s]
+{{\gamma}}
(1-\lambda)\sum_{s^{'}\in\calS} \sum_{{t}=0}^{\infty}({{\gamma}}\lambda)^{{t}}\Pro_{\policy}(s_{t+1}=s^{'}|s)\varphi(s^{'})
-\varphi(s).
\end{flalign}

Recall the terms $ \mathbf{P}^{(\lambda)}_{\pi_{\bm \theta}},~\mathbf{r}^{(\lambda)}_{\pi_{\bm \theta}}[s]$ defined in (\ref{def:matrix-p-lam-return})-(\ref{lam-pro-value-03}),
\begin{flalign}
\label{app-004}
R^{(\lambda)}_{\pi_{\bm \theta}}(s)+\gamma(1-\lambda)\sum_{s^{'}\in\calS} \Pro_{\policy}^{(\lambda)}(s^{'}|s)\varphi(s^{'})-\varphi(s)
\end{flalign}

We consider the first term $R^{(\lambda)}_{\pi_{\bm \theta}}(s)$ of (\ref{app-04}) as follows,
\begin{flalign}
\label{app-005}
R^{(\lambda)}_{\pi_{\bm \theta}}(s)\overset{(\ref{def:matrix-p-lam-return})-(\ref{lam-pro-value-03})}=
\mathbf{r}^{(\lambda)}_{\pi_{\bm \theta}}[s]=\sum_{{t}=0}^{\infty}(\gamma\lambda)^{t}\bP_{\policy}^{{t}}\mathbf{r}_{\pi_{\bm \theta}}[s]= \sum_{{t}=0}^{\infty}\sum_{s_t\in\calS}(\gamma\lambda)^{t}\Pro_{\policy}(s_{t}|s)R_{\policy}(s_t).
\end{flalign}

We consider the second term $\tilde\gamma\sum_{s\in\calS} \Pro_{\policy}^{(\lambda)}(s^{'}|s)\varphi(s)-\varphi(s)$ of (\ref{app-04}) as follows,
\begin{flalign}
\nonumber
&\tilde\gamma\sum_{s^{'}\in\calS} \Pro_{\policy}^{(\lambda)}(s^{'}|s)\varphi(s^{'})-\varphi(s)\\
\overset{(\ref{lam-pro-value-02})}=&\tilde\gamma
(1-\gamma\lambda)\sum_{s^{'}\in\calS} \sum_{{t}=0}^{\infty}(\gamma\lambda)^{{t}}\Pro_{\policy}(s_{t+1}=s^{'}|s)\varphi(s^{'})
-\varphi(s)\\
\overset{(\ref{def:matrix-p-lam-return})}
=&\gamma
(1-\lambda)\sum_{s^{'}\in\calS} \sum_{{t}=0}^{\infty}(\gamma\lambda)^{{t}}\Pro_{\policy}(s_{t+1}=s^{'}|s)\varphi(s^{'})
-\varphi(s)\\
\nonumber
=&\gamma\sum_{s^{'}\in\calS} \sum_{{t}=0}^{\infty}(\gamma\lambda)^{{t}}\Pro_{\policy}(s_{t+1}=s^{'}|s)\varphi(s^{'})
-\sum_{s^{'}\in\calS} 
\underbrace{\left(\sum_{{t}=0}^{\infty}(\gamma\lambda)^{{t}+1}\Pro_{\policy}(s_{t+1}=s^{'}|s)\varphi(s^{'})\right)}_{=\sum_{{t}=1}^{\infty}(\gamma\lambda)^{{t}}\Pro_{\policy}(s_{t}=s^{'}|s)\varphi(s^{'})}
-\varphi(s)\\
\label{app-001}
=&\gamma\sum_{s^{'}\in\calS} \sum_{{t}=0}^{\infty}(\gamma\lambda)^{{t}}\Pro_{\policy}(s_{t+1}=s^{'}|s)\varphi(s^{'})
-
\underbrace
{
\left(
\sum_{s^{'}\in\calS} \sum_{{t}=1}^{\infty}(\gamma\lambda)^{{t}}\Pro_{\policy}(s_{t}=s^{'}|s)\varphi(s^{'})+\varphi(s)
\right)
}_{=\sum_{s^{'}\in\calS} \sum_{{t}=0}^{\infty}(\gamma\lambda)^{{t}}\Pro_{\policy}(s_{t}=s^{'}|s)\varphi(s^{'})}
\\
\label{app-003}
=&\gamma\sum_{s^{'}\in\calS} \sum_{{t}=0}^{\infty}(\gamma\lambda)^{{t}}\Pro_{\policy}(s_{t+1}=s^{'}|s)\varphi(s^{'})
-\sum_{s_t\in\calS} \sum_{{t}=0}^{\infty}(\gamma\lambda)^{{t}}\Pro_{\policy}(s_{t}|s)\varphi(s),
\end{flalign}
where the equation from Eq.(\ref{app-001}) to Eq.(\ref{app-003}) holds since: according to (\ref{special-inititial-pro}), we use the following identity 
\begin{flalign}
\nonumber
\sum_{s^{'}\in\calS} \Pro_{\policy}(s_{0}=s^{'}|s)\varphi(s^{'})=\varphi(s).
\end{flalign}

Furthermore,  take the result (\ref{app-005}) and (\ref{app-003}) to (\ref{app-004}), we have
\begin{flalign}
\nonumber
&R^{(\lambda)}_{\pi_{\bm \theta}}(s)+\tilde\gamma
\sum_{s^{'}\in\calS} \Pro_{\policy}^{(\lambda)}(s^{'}|s)\varphi(s^{'})-\varphi(s)\\
\label{app-020}
=& \sum_{{t}=0}^{\infty}(\gamma\lambda)^{t}
\left(
\sum_{s_t\in\calS}\Pro_{\policy}(s_{t}|s)
R_{\policy}(s_t)+\gamma\sum_{s^{'}\in\calS}
\underbrace{\Pro_{\policy}(s_{t+1}=s^{'}|s)\varphi(s^{'})}_{\overset{(\ref{pro-pi-t-step-app})}=
\sum_{s_t\in\calS}
\Pro_{\policy}(s_{t+1}=s^{'}|s_t)\Pro_{\policy}(s_t|s)\varphi(s^{'})}
-\sum_{s_{t}\in\calS}\Pro_{\policy}(s_{t}|s)\varphi(s_t)
\right)
\\
\label{app-021}
=& \sum_{{t}=0}^{\infty}(\gamma\lambda)^{t}
\left(
\sum_{s_t\in\calS}\Pro_{\policy}(s_{t}|s)R_{\policy}(s_t)+\gamma\sum_{s_t\in\calS}\Pro_{\policy}(s_{t}|s)\sum_{s_{t+1}\in\calS}\Pro_{\policy}(s_{t+1}|s_{t})\varphi(s_{t+1})
-\sum_{s_t\in\calS}\Pro_{\policy}(s_{t}|s)\varphi(s_t)
\right)
\\
\nonumber
= &\sum_{{t}=0}^{\infty}(\gamma\lambda)^{t}\sum_{s_t\in\calS}
\Pro_{\policy}(s_{t}|s)
\left(
\underbrace{
\sum_{a_t\in\mathcal{A}}{\policy}(a_t|s_t)\sum_{s_{t+1}\in\calS}\Pro(s_{t+1}|s_t,a_t)r(s_{t+1}|s_t,a_t)
}_{=R_{\policy}(s_t)}
\right.
\\
\nonumber
&\left. ~~~~~~~~~~~~~~~~~~~~~~~~~~~~~~~~~~~~~~~~~~~~~~~~~~~~~~~~+\gamma\underbrace{\sum_{a_t\in\mathcal{A}}{\policy}(a_t|s_t)\sum_{s_{t+1}\in\calS}\Pro(s_{t+1}|s_t,a_t)}_{=\Pro_{\policy}(s_{t+1}|s_{t})}\varphi(s_{t+1})
-\varphi(s_{t})
\right)\\
\label{app-022}
=& \sum_{{t}=0}^{\infty}(\gamma\lambda)^{t}\sum_{s_t\in\calS}\Pro_{\policy}(s_{t}|s)\sum_{a_t\in\mathcal{A}}{\policy}(a_t|s_t)\sum_{s_{t+1}\in\calS}\Pro(s_{t+1}|s_t,a_t)
\left(r(s_{t+1}|s_t,a_t)+\gamma\varphi(s_{t+1})-\varphi(s_{t})\right)\\
\label{app-023}
=& \sum_{{t}=0}^{\infty}(\gamma\lambda)^{t}\E_{s_{t}\sim\Pro_{\policy}(\cdot|s),a_{t}\sim{\policy}(\cdot|s_t),s_{t+1}\sim\Pro(\cdot|s_t,a_t)}\left[r(s_{t+1}|s_t,a_t)+\gamma\varphi(s_{t+1})-\varphi(s_{t})\right],
\end{flalign}
the equation from Eq.(\ref{app-003}) to Eq.(\ref{app-020}) holds since:
\[
\Pro_{\policy}(s_{t+1}|s)\overset{(\ref{pro-pi-t-step-app})}=
\sum_{s_t\in\calS}
\Pro_{\policy}(s_{t+1}|s_t)\Pro_{\policy}(s_t|s);
\]
the equation from Eq.(\ref{app-020}) to Eq.(\ref{app-021}) holds since we use the Markov property of the definition of MDP: for each time $t\in\N$,
\[\Pro_{\policy}(s_{t+1}=s^{'}|s_t=s)=\Pro_{\policy}(s^{'}|s);\]
the equation (\ref{app-022}) the following identity:
\[
\sum_{a_t\in\mathcal{A}}{\policy}(a_t|s_t)=1,~~~~\sum_{s_{t+1}\in\calS}\Pro(s_{t+1}|s_t,a_t)=1,
\]
then
\[
\varphi(s_{t})=\sum_{a_t\in\mathcal{A}}{\policy}(a_t|s_t)\sum_{s_{t+1}\in\calS}\Pro(s_{t+1}|s_t,a_t)\varphi(s_{t}).
\]

{
\color{blue}
{
\textbf{Step 3: Put all the result together.}
}
}

Finally, let 
\begin{flalign}
\nonumber
\delta_t^{\varphi}&=r(s_{t+1}|s_t,a_t)+\gamma\varphi(s_{t+1})-\varphi(s_{t}),
\\
\nonumber
\delta^{\varphi}_{\policy,t}(s)&=\E_{s_{t}\sim\Pro_{\policy}(\cdot|s),a_{t}\sim{\policy}(\cdot|s_t),s_{t+1}\sim\Pro(\cdot|s_t,a_t)}\left[\delta_t^{\varphi}\right],
\end{flalign}
combining the results (\ref{lam-return-objective-01}) and (\ref{app-023}), we have
\begin{flalign}
\label{lam-return-phi-objective-prop}
J(\policy)=&\E_{s_0\sim\rho_{0}(\cdot)}[\varphi(s_0)]
+
\dfrac{1}{1-\tilde\gamma}\sum_{s\in\calS}d^{\lambda}_{\policy}(s)
\left(
\sum_{t=0}^{\infty}\gamma^t \lambda^t
\delta^{\varphi}_{\policy,t}(s)
\right)
\\
\nonumber
=&\E_{s_0\sim\rho_{0}(\cdot)}[\varphi(s_0)]
+
\dfrac{1}{1-\tilde\gamma}\E_{s\sim d^{\lambda}_{\policy}(\cdot)}
\left[
\sum_{t=0}^{\infty}\gamma^t \lambda^t
\delta^{\varphi}_{\policy,t}(s)
\right]
.
\end{flalign}
This concludes the proof of Proposition \ref{objective-td-error-version}.
\end{proof}

\clearpage
\section{Lemma \ref{lem:difference-distri}}
\label{sec:difference-distri}

\begin{lemma}
\label{lem:difference-distri}
Let $\|\bm{\Pi}_{\policyy}-\bm{\Pi}_{\policy}\|_{1,1}$ denote as the $L_{1,1}$-norm for the difference between two policy space $\{\policy(a|s)\}_{(s,a)\in\calS\times\calA}$, $\{\policyy(a|s)\}_{(s,a)\in\calS\times\calA}$, i.e.,
\begin{flalign}
\label{difference-01}
\|\bm{\Pi}_{\policyy}-\bm{\Pi}_{\policy}\|_{1,1}=:
\sum_{s\in\mathcal{S}}
\sum_{a\in\calA}
\left|
{{\policyy}}(a|s)-{{\policy}}(a|s)
\right|.
\end{flalign}
The divergence between discounted future state visitation distributions, $\|\bd_{\policyy}^{\lambda}-\bd_{\pi_{\bm {\theta}}}^{\lambda}\|_1$, is bounded as follows,
\[
\|\bd_{\policyy}^{\lambda}-\bd_{\pi_{\bm {\theta}}}^{\lambda}\|_1
\leq
\dfrac{(1-\gamma\lambda)^2}{(1-\gamma)\left(1-\gamma\lambda\|\bm{\Pi}_{\policyy}-\bm{\Pi}_{\policy}\|_{1,1}\right)}\E_{s\sim d_{\pi_{\bm {\theta}}}^{\lambda}(\cdot)}
\Big[2D_{\emph{TV}}(\policyy,\policy)[s]\Big]
\]
and 
\[
\|\bd_{\policyy}^{\lambda}-\bd_{\pi_{\bm {\theta}}}^{\lambda}\|_1
\leq
\dfrac{(1-\gamma\lambda)^2}{(1-\gamma)\left(1-\gamma\lambda\|\bm{\Pi}_{\policyy}-\bm{\Pi}_{\policy}\|_{1,1}\right)}\E_{s\sim d_{\policyy}^{\lambda}(\cdot)}
\Big[2D_{\emph{TV}}(\policyy,\policy)[s]\Big],
\]
where
\[
2D_{\emph{TV}}(\policyy,\policy)[s]=:\sum_{a\in\calA}\left|{{\policyy}}(a|s)-{{\policy}}(a|s)\right|.
\]
Furthermore, we achieve the boundedness of $\|\bd_{\policyy}^{\lambda}-\bd_{\pi_{\bm {\theta}}}^{\lambda}\|_1$ as follows,
\begin{flalign}
\nonumber
\|\bd_{\policyy}^{\lambda}-\bd_{\pi_{\bm {\theta}}}^{\lambda}\|_1
\leq
\dfrac{1}{1-\tilde\gamma}\cdot
\dfrac{1-\gamma\lambda}{\left|1-2\gamma\lambda|\calS||\calA|\right|}\E_{s\sim d_{\pi_{\bm {\theta}}}^{\lambda}(\cdot)}
\Big[2D_{\emph{TV}}(\policyy,\policy)[s]\Big],
\end{flalign}
\begin{flalign}
\nonumber
\|\bd_{\policyy}^{\lambda}-\bd_{\pi_{\bm {\theta}}}^{\lambda}\|_1
\leq
\dfrac{1}{1-\tilde\gamma}\cdot
\dfrac{1-\gamma\lambda}{\left|1-2\gamma\lambda|\calS||\calA|\right|}\E_{s\sim d_{\policyy}^{\lambda}(\cdot)}
\Big[2D_{\emph{TV}}(\policyy,\policy)[s]\Big].
\end{flalign}
\end{lemma}

\begin{proof}

Recall Eq.(\ref{matrixversion-lambda-dis-state-distribution}), let
\begin{flalign}
\label{app-g-01}
\bG_{\policy}=\left(\bI-\tilde{\gamma}\bP^{(\lambda)}_{\policy}\right)^{-1},~~\bG_{\policyy}=\left(\bI-\tilde{\gamma}\bP^{(\lambda)}_{\policyy}\right)^{-1},~~\bD=\bP^{(\lambda)}_{\policyy}-\bP^{(\lambda)}_{\policy}.
\end{flalign}
Then, the following holds
\begin{flalign}
\label{app-g-02}
\bG_{\policy}^{-1}-\bG_{\policyy}^{-1}=\left(\bI-\tilde{\gamma}\bP^{(\lambda)}_{\policy}\right)-\left(\bI-\tilde{\gamma}\bP^{(\lambda)}_{\policyy}\right)=\tilde\gamma\bD.
\end{flalign}
Furthermore, by left-multiplying by $\bG_{\policy}$ and right-multiplying by $\bG_{\policyy}$, we achieve
\begin{flalign}
\label{app-g-03}
\bG_{\policyy}-\bG_{\policy}=\tilde\gamma\bG_{\policyy}\bD\bG_{\policy}.
\end{flalign}
Grouping all the results from (\ref{app-g-01})-(\ref{app-g-03}), recall (\ref{matrixversion-lambda-dis-state-distribution}),
\begin{flalign}
\label{matrixversion-lambda-dis-state-distribution-001}
\bd_{\pi_{\bm {\theta}}}^{\lambda}=(1-\tilde \gamma)\sum_{t=0}^{\infty}\left(\gamma\bP^{(\lambda)}_{\policy}\right)^{t}\bm{\rho}_{0}=(1-\tilde \gamma)\left(\bI-\tilde{\gamma}\bP^{(\lambda)}_{\policy}\right)^{-1}\bm{\rho}_{0}=(1-\tilde \gamma)\bG_{\policy}\bm{\rho}_{0},
\end{flalign}
then we have
\begin{flalign}
\nonumber
\bd_{\policyy}^{\lambda}-\bd_{\pi_{\bm {\theta}}}^{\lambda}
=&(1-\tilde \gamma)\left(\bG_{\policyy}-\bG_{\policy}\right)\bm{\rho}_{0}\\
\nonumber
\overset{(\ref{app-g-03})}=&
(1-\tilde \gamma)\tilde\gamma\bG_{\policyy}\bD\bG_{\policy}\bm{\rho}_{0}
\\
\label{app-error-gap-01}
\overset{(\ref{matrixversion-lambda-dis-state-distribution-001})}=&\tilde\gamma\bG_{\policyy}\bD\bd_{\pi_{\bm {\theta}}}^{\lambda}.
\end{flalign}
Applying (\ref{app-error-gap-01}), we have
\begin{flalign}
\label{app-distri-difference-01}
\|\bd_{\policyy}^{\lambda}-\bd_{\pi_{\bm {\theta}}}^{\lambda}\|_1\overset{(\ref{app-error-gap-01})}\leq
\tilde\gamma\|\bG_{\policyy}\|_1\|\bD\bd_{\pi_{\bm {\theta}}}^{\lambda}\|_1.
\end{flalign}
Firstly, we bound the term $\|\bG_{\policyy}\|_1$ as follows,
\begin{flalign}
\label{app-distri-difference-02}
\|\bG_{\policyy}\|_1=\left\|\left(\bI-\tilde{\gamma}\bP^{(\lambda)}_{\policyy}\right)^{-1}\right\|_1
\leq\sum_{t=0}^{\infty}\tilde{\gamma}^{t}\left\|\bP^{(\lambda)}_{\policyy}\right\|_1=\dfrac{1}{1-\tilde\gamma}=\dfrac{1-\gamma\lambda}{1-\gamma}.
\end{flalign}
Now, we analyze the second term as follows,
\begin{flalign}
\nonumber
&\|\bD\bd_{\pi_{\bm {\theta}}}^{\lambda}\|_1\\
\nonumber
=&\sum_{s^{'}\in\calS}\left|\sum_{s\in\calS}\bD(s^{'}|s)d_{\pi_{\bm {\theta}}}^{\lambda}(s)\right|\\
\nonumber
\overset{(\ref{lam-pro-value-02})}=&\sum_{s^{'}\in\calS}\left|\sum_{s\in\calS}\left(
 \Pro_{\policyy}^{(\lambda)}(s^{'}|s)-\Pro_{\policy}^{(\lambda)}(s^{'}|s)\right)\right|d_{\pi_{\bm {\theta}}}^{\lambda}(s)\\
\nonumber
 \overset{(\ref{lam-pro-value-02})}=&\sum_{s^{'}\in\calS}\left|
 (1-\gamma\lambda)\sum_{{t}=0}^{\infty}(\gamma\lambda)^{{t}}\sum_{s\in\calS}\left(\Pro_{\policyy}(s_{t+1}=s^{'}|s)-\Pro_{\policy}(s_{t+1}=s^{'}|s)\right)\right|d_{\pi_{\bm {\theta}}}^{\lambda}(s)\\
 \label{error-gap-01}
 \leq&\sum_{s\in\calS}
 \left(
 (1-\gamma\lambda)\sum_{{t}=0}^{\infty}(\gamma\lambda)^{{t}} \sum_{s^{'}\in\calS}\left|\Pro_{\policyy}(s_{t+1}=s^{'}|s)-\Pro_{\policy}(s_{t+1}=s^{'}|s)\right|
 \right)
 d_{\pi_{\bm {\theta}}}^{\lambda}(s).
\end{flalign}.

Before we provide a further analyze (\ref{error-gap-01}), we need to bound $|\Pro_{\policyy}(s_{t+1}=s^{'}|s)-\Pro_{\policy}(s_{t+1}=s^{'}|s)|$.
Let $s_0=s$, then 
\begin{flalign}
\nonumber
\mathbb{P}_{{\policy}}(s_{t+1}=s^{'}|s)\overset{\ref{pro-pi-t-step-app}}=&\sum_{s_1\in\mathcal{S}}\mathbb{P}_{{\policy}}(s_{t+1}=s^{'}|s_1)\mathbb{P}_{{\policy}}(s_1|s_0)\\
\nonumber
=&\sum_{s_1\in\mathcal{S}}\sum_{s_2\in\mathcal{S}}\mathbb{P}_{{\policy}}(s_{t+1}=s^{'}|s_2)\mathbb{P}_{{\policy}}(s_{2}|s_1)\mathbb{P}_{{\policy}}(s_1|s_0)\\
\nonumber
=&\cdots
\\
\nonumber
=&\sum_{s_1\in\mathcal{S}}\sum_{s_2\in\mathcal{S}}\cdots\sum_{s_{t}\in\mathcal{S}}\left(\prod_{i=1}^{t+1}\mathbb{P}_{{\policy}}(s_{i}|s_{i-1})\right)
\\
\label{multi-step-pro-01}
=&\sum_{s_1\in\mathcal{S}}\sum_{s_2\in\mathcal{S}}\cdots\sum_{s_{t}\in\mathcal{S}}\left(\prod_{i=1}^{t+1}
\left(
\sum_{a_i\in\calA}\mathbb{P}(s_{i}|s_{i-1},a_i){{\policy}}(a_i|s_{i-1})
\right)
\right).
\end{flalign}
Similarly, we have
\begin{flalign}
\label{multi-step-pro-02}
\mathbb{P}_{{\policyy}}(s_{t+1}=s^{'}|s)=&\sum_{s_1\in\mathcal{S}}\sum_{s_2\in\mathcal{S}}\cdots\sum_{s_{t}\in\mathcal{S}}\left(\prod_{i=1}^{t+1}
\left(
\sum_{a_i\in\calA}\mathbb{P}(s_{i}|s_{i-1},a_i){{\policyy}}(a_i|s_{i-1})
\right)
\right).
\end{flalign}

Then, according to the results (\ref{multi-step-pro-01})-(\ref{multi-step-pro-02}), let $s_0=s$, the following holds
\begin{flalign}
\nonumber
&\sum_{s^{'}\in\calS}|\mathbb{P}_{{\policyy}}(s_{t+1}=s^{'}|s)-\mathbb{P}_{{\policy}}(s_{t+1}=s^{'}|s)|\\
\nonumber
=&\sum_{s^{'}\in\calS}\left|
\sum_{s_1\in\mathcal{S}}\sum_{s_2\in\mathcal{S}}\cdots\sum_{s_{t}\in\mathcal{S}}\left(\prod_{i=1}^{t+1}
\left(
\sum_{a_i\in\calA}\mathbb{P}(s_{i}|s_{i-1},a_i)
\left(
{{\policyy}}(a_i|s_{i-1})-{{\policy}}(a_i|s_{i-1})
\right)
\right)
\right)\right|
\\
\nonumber
\leq&
\sum_{s_1\in\mathcal{S}}\sum_{s_2\in\mathcal{S}}\cdots\sum_{s_{t}\in\mathcal{S}}\left(\prod_{i=1}^{t+1}
\sum_{a_i\in\calA}
\left|
{{\policyy}}(a_i|s_{i-1})-{{\policy}}(a_i|s_{i-1})
\right|
\right)\\
\nonumber
=&
\sum_{s_1\in\mathcal{S}}\sum_{s_2\in\mathcal{S}}\cdots\sum_{s_{t}\in\mathcal{S}}\left(\prod_{i=2}^{t+1}
\sum_{a_i\in\calA}
\left|
{{\policyy}}(a_i|s_{i-1})-{{\policy}}(a_i|s_{i-1})
\right|
\right)
\cdot
\left(
\sum_{a_1\in\calA}
\left|
{{\policyy}}(a_1|s_{0})-{{\policy}}(a_1|s_{0})
\right|
\right)
\\
\nonumber
=&
\prod_{i=2}^{t+1}
\left(
\sum_{s_{i-1}\in\mathcal{S}}
\sum_{a_i\in\calA}
\left|
{{\policyy}}(a_i|s_{i-1})-{{\policy}}(a_i|s_{i-1})
\right|
\right)
\cdot
\left(
\sum_{a_1\in\calA}
\left|
{{\policyy}}(a_1|s_{0})-{{\policy}}(a_1|s_{0})
\right|
\right)
\\
\label{multi-step-pro-03}
=&\left(
\underbrace{
\sum_{s\in\mathcal{S}}
\sum_{a\in\calA}
\left|
{{\policyy}}(a|s)-{{\policy}}(a|s)
\right|
}_{=:\left\|\bm{\Pi}_{\policyy}-\bm{\Pi}_{\policy}\right\|_{1,1}}
\right
)^{t}
\cdot
\left(
\sum_{a\in\calA}
\left|
{{\policyy}}(a|s)-{{\policy}}(a|s)
\right|
\right).
\end{flalign}

Taking the result (\ref{multi-step-pro-03}) to (\ref{error-gap-01}), we have
\begin{flalign}
\nonumber
\|\bD\bd_{\pi_{\bm {\theta}}}^{\lambda}\|_1
 \leq&
 (1-\gamma\lambda)\sum_{{t}=0}^{\infty}(\gamma\lambda)^{{t}} 
\left\|\bm{\Pi}_{\policyy}-\bm{\Pi}_{\policy}\right\|^{t}_{1,1}
\sum_{s\in\calS}
\underbrace{\sum_{a\in\calA}\left|{{\policyy}}(a|s)-{{\policy}}(a|s)\right|}_{=:2D_{\text{TV}}(\policyy,\policy)[s]}d_{\pi_{\bm {\theta}}}^{\lambda}(s)
\\
\nonumber
=& (1-\gamma\lambda)\sum_{{t}=0}^{\infty}(\gamma\lambda)^{{t}} 
\|\bm{\Pi}_{\policyy}-\bm{\Pi}_{\policy}\|^{t}_{1,1}
\E_{s\sim d_{\pi_{\bm {\theta}}}^{\lambda}(\cdot)}
\Big[2D_{\text{TV}}(\policyy,\policy)[s]\Big]\\
\label{app-distri-difference-03}
=&\dfrac{1-\gamma\lambda}{1-\gamma\lambda\|\bm{\Pi}_{\policyy}-\bm{\Pi}_{\policy}\|_{1,1}}\E_{s\sim d_{\pi_{\bm {\theta}}}^{\lambda}(\cdot)}
\Big[2D_{\text{TV}}(\policyy,\policy)[s]\Big].
\end{flalign}

Finally, according to (\ref{app-distri-difference-01}), (\ref{app-distri-difference-02}) and (\ref{app-distri-difference-03}), we have 
\begin{flalign}
\|\bd_{\policyy}^{\lambda}-\bd_{\pi_{\bm {\theta}}}^{\lambda}\|_1
\leq
\dfrac{\tilde\gamma}{1-\tilde\gamma}\cdot
\dfrac{\gamma(1-\lambda)}{1-\gamma\lambda\|\bm{\Pi}_{\policyy}-\bm{\Pi}_{\policy}\|_{1,1}}\E_{s\sim d_{\pi_{\bm {\theta}}}^{\lambda}(\cdot)}
\Big[2D_{\text{TV}}(\policyy,\policy)[s]\Big].
\end{flalign}
Recall (\ref{difference-01}), we have
\begin{flalign}
\label{difference-02}
\|\bm{\Pi}_{\policyy}-\bm{\Pi}_{\policy}\|_{1,1}=
\sum_{s\in\mathcal{S}}
\sum_{a\in\calA}
\left|
{{\policyy}}(a|s)-{{\policy}}(a|s)
\right|\leq2|\calS||\calA|.
\end{flalign}
Then, we achieve the boundedness of $\|\bm{\Pi}_{\policyy}-\bm{\Pi}_{\policy}\|_{1,1}$ as follows,
\begin{flalign}
\|\bd_{\policyy}^{\lambda}-\bd_{\pi_{\bm {\theta}}}^{\lambda}\|_1
\leq
\dfrac{\tilde\gamma}{1-\tilde\gamma}\cdot
\dfrac{1-\gamma\lambda}{\left|1-2\gamma\lambda|\calS||\calA|\right|}\E_{s\sim d_{\pi_{\bm {\theta}}}^{\lambda}(\cdot)}
\Big[2D_{\text{TV}}(\policyy,\policy)[s]\Big].
\end{flalign}
Similarly, we obtain
\begin{flalign}
\nonumber
\|\bd_{\policyy}^{\lambda}-\bd_{\pi_{\bm {\theta}}}^{\lambda}\|_1
\leq
\dfrac{1}{1-\tilde\gamma}\cdot
\dfrac{1-\gamma\lambda}{\left|1-2\gamma\lambda|\calS||\calA|\right|}\E_{s\sim d_{\policyy}^{\lambda}(\cdot)}
\Big[2D_{\emph{TV}}(\policyy,\policy)[s]\Big].
\end{flalign}
\end{proof}

\clearpage

\section{Proof of Theorem \ref{them-re-cost}}
\label{sec:app-them2}

\textbf{Theorem \ref{them-re-cost}}
\emph{
Let $\chi_k=\E_{s\sim{d}_{{\pi_{\bm{\theta}_k}}}^{\lambda}(\cdot)}\left[\emph{KL}\left(\pi_{\bm{\theta}_k},\pi_{\bm{\theta}_{k+\frac{1}{2}}}\right)[s]\right]$, if $\pi_{\bm{\theta}_k}$ and 
$\pi_{\bm{\theta}_{k+1}}$ are related to (\ref{performance-improvement})-(\ref{projection}),
then the lower bound on policy improvement, and upper bound on constraint violation are
\[
J(\pi_{\bm{\theta}_{k+1}})-J(\pi_{\bm{\theta}_{k}})\ge-\frac{\gamma(1-\lambda)\alpha_k\sqrt{2\chi_k}\epsilon^{V}_{\policy}(\policyy)}{(1-\gamma)\left|1-2\gamma\lambda|\calS||\calA|\right|},
J^{c}(\pi_{\bm{\theta}_{k+1}})\leq b+\frac{\gamma(1-\lambda)\beta_k\sqrt{2\chi_k}\epsilon^{C}_{\policy}(\policyy)}{(1-\gamma)\left|1-2\gamma\lambda|\calS||\calA|\right|}.
\]
}

\begin{proof}(of Theorem \ref{them-re-cost})

According to Bregman divergence, if policy $\pi_{\bm{\theta}_{k}}$ is feasible, policy $\pi_{\bm{\theta}_{k+1}}$ is generated according to (\ref{projection}), then
the following 
\[
\text{KL}\left(\pi_{\bm{\theta}_{k}},\pi_{\bm{\theta}_{k+\frac{1}{2}}}\right)
\ge 
\text{KL}\left(\pi_{\bm{\theta}_{k}},\pi_{\bm{\theta}_{k+1}}\right)
+
\text{KL}\left(\pi_{\bm{\theta}_{k+1}},\pi_{\bm{\theta}_{k+\frac{1}{2}}}\right)
\]
implies
\[
\chi_k=\E_{s\sim{d}_{{\pi_{\bm{\theta}_k}}}^{\lambda}(\cdot)}\left[\text{KL}\left(\pi_{\bm{\theta}_k},\pi_{\bm{\theta}_{k+\frac{1}{2}}}\right)[s]\right]
\ge
\E_{s\sim{d}_{{\pi_{\bm{\theta}_k}}}^{\lambda}(\cdot)}\left[\text{KL}\left(\pi_{\bm{\theta}_{k+1}},\pi_{\bm{\theta}_{k}}\right)[s]\right].
\]
 According to the asymptotically symmetry of KL divergence if we update the policy within a local region, then, we have
 \[
 \chi_k\ge\E_{s\sim{d}_{{\pi_{\bm{\theta}_k}}}^{\lambda}(\cdot)}\left[\text{KL}\left(\pi_{\bm{\theta}_{k+\frac{1}{2}}},\pi_{\bm{\theta}_{k}}\right)[s]\right]\ge
\E_{s\sim{d}_{{\pi_{\bm{\theta}_k}}}^{\lambda}(\cdot)}\left[\text{KL}\left(\pi_{\bm{\theta}_{k+1}},\pi_{\bm{\theta}_{k}}\right)[s]\right].
 \]

Furthermore,  according to Proposition \ref{propo-01} and Proposition \ref{propo-03}, we have
\begin{flalign}
\nonumber
&J(\pi_{\bm{\theta}_{k+1}})-J(\pi_{\bm{\theta}_{k}})\\
\nonumber
\ge&
\frac{1}{1-\tilde\gamma}\E_{s\sim{d}_{\pi_{\bm{\theta}_{k}}}^{\lambda}(\cdot),a\sim\pi_{\bm{\theta}_{k+1}}(\cdot|s)}
\left[
A^{\text{GAE}(\gamma,\lambda)}_{\pi_{\bm{\theta}_{k}}}(s,a)
-
\frac{2\gamma(1-\lambda)\epsilon^{V}_{\pi_{\bm{\theta}_{k+1}}}(\pi_{\bm{\theta}_{k}})}{(1-\gamma\lambda)\left|1-2\gamma\lambda|\calS||\calA|\right|}
D_{\text{TV}}(\pi_{\bm{\theta}_{k}},\pi_{\bm{\theta}_{k+1}})[s]
\right]\\
\nonumber
\ge&
\frac{1}{1-\tilde\gamma}\E_{s\sim{d}_{\pi_{\bm{\theta}_{k}}}^{\lambda}(\cdot),a\sim\pi_{\bm{\theta}_{k+1}}(\cdot|s)}
\left[
-
\frac{2\gamma(1-\lambda)\alpha_k\epsilon^{V}_{\pi_{\bm{\theta}_{k+1}}}(\pi_{\bm{\theta}_{k}})}{(1-\gamma\lambda)\left|1-2\gamma\lambda|\calS||\calA|\right|}
\sqrt{\dfrac{1}{2}\E_{s\sim{d}_{\pi_{\bm{\theta}_{k}}}^{\lambda}(\cdot)}\left[\text{KL}(\pi_{\bm{\theta}_{k}},\pi_{\bm{\theta}_{k+1}})[s]\right]}
\right]\\
\nonumber
\ge&
\frac{1}{1-\tilde\gamma}\E_{s\sim{d}_{\pi_{\bm{\theta}_{k}}}^{\lambda}(\cdot),a\sim\pi_{\bm{\theta}_{k+1}}(\cdot|s)}
\left[
-
\frac{\gamma(1-\lambda)\alpha_k\sqrt{2\chi_k}\epsilon^{V}_{\pi_{\bm{\theta}_{k+1}}}(\pi_{\bm{\theta}_{k}})}{(1-\gamma\lambda)\left|1-2\gamma\lambda|\calS||\calA|\right|}
\right].
\end{flalign}

Similarly,  according to Proposition \ref{propo-01} and Proposition \ref{pro-02}, and since policy $\pi_{\bm{\theta}_{k+1}}$ satisfies
\begin{flalign}
\label{app-c-01}
J^{c}(\pi_{{\bm{\theta}}_k})+\dfrac{1}{1-\tilde\gamma}\E_{s\sim{d}_{\pi_{\bm{\theta}_k}}^{\lambda}(\cdot),a\sim\pi_{\bm{\theta}_{k+1}}(\cdot|s)}
\left[
A^{\text{GAE}(\gamma,\lambda)}_{{\pi_{\bm{\theta}_k}},C}(s,a)\right]+\beta_k
\sqrt{
\E_{s\sim{d}_{{\pi_{\bm{\theta}_k}}}^{\lambda}(\cdot)}\left[\text{KL}(\pi_{\bm{\theta}_k},\pi_{\bm{\theta}_{k+1}})[s]\right]}\leq b,
\end{flalign}
and 
\begin{flalign}
\label{app-c-02}
&J^{c}(\pi_{\bm{\theta}_{k+1}})-J^{c}(\pi_{\bm{\theta}_{k}})\\
\nonumber
\leq&
\frac{1}{1-\tilde\gamma}\E_{s\sim{d}_{\pi_{\bm{\theta}_{k}}}^{\lambda}(\cdot),a\sim\pi_{\bm{\theta}_{k+1}}(\cdot|s)}
\left[
A^{\text{GAE}(\gamma,\lambda)}_{\pi_{\bm{\theta}_{k}},C}(s,a)
+
\frac{2\gamma(1-\lambda)\beta_k\epsilon^{C}_{\pi_{\bm{\theta}_{k+1}}}(\pi_{\bm{\theta}_{k}})}{(1-\gamma\lambda)\left|1-2\gamma\lambda|\calS||\calA|\right|}
D_{\text{TV}}(\pi_{\bm{\theta}_{k}},\pi_{\bm{\theta}_{k+1}})[s]
\right].
\end{flalign}
Combining (\ref{app-c-01})- (\ref{app-c-02}), we have
\begin{flalign}
\label{app-c-02}
&J^{c}(\pi_{\bm{\theta}_{k+1}})-J^{c}(\pi_{\bm{\theta}_{k}})\\
\nonumber
\leq&b+
\frac{1}{1-\tilde\gamma}\E_{s\sim{d}_{\pi_{\bm{\theta}_{k}}}^{\lambda}(\cdot),a\sim\pi_{\bm{\theta}_{k+1}}(\cdot|s)}
\left[
\frac{2\gamma(1-\lambda)\beta_k\epsilon^{C}_{\pi_{\bm{\theta}_{k+1}}}(\pi_{\bm{\theta}_{k}})}{(1-\gamma\lambda)\left|1-2\gamma\lambda|\calS||\calA|\right|}
\sqrt{\dfrac{1}{2}\E_{s\sim{d}_{\pi_{\bm{\theta}_{k}}}^{\lambda}(\cdot)}\left[\text{KL}(\pi_{\bm{\theta}_{k}},\pi_{\bm{\theta}_{k+1}})[s]\right]}
\right]\\
\leq&b+
\frac{1}{1-\tilde\gamma}\E_{s\sim{d}_{\pi_{\bm{\theta}_{k}}}^{\lambda}(\cdot),a\sim\pi_{\bm{\theta}_{k+1}}(\cdot|s)}
\left[
\frac{\gamma(1-\lambda)\beta_k\sqrt{2\chi_k}\epsilon^{C}_{\pi_{\bm{\theta}_{k+1}}}(\pi_{\bm{\theta}_{k}})}{(1-\gamma\lambda)\left|1-2\gamma\lambda|\calS||\calA|\right|}
\right].
\end{flalign}

\end{proof}

\clearpage
\section{More Details for Experiments}
\label{sec:app-ex}
All experiments were implemented in Pytorch 1.7.0 with CUDA 11.0 and conducted on an Ubuntu 20.04.2 LTS (GNU/Linux 5.8.0-59-generic x86 64) with 40 CPU cores (Intel(R)
Xeon(R) Silver 4210R CPU @ 2.40GHz), 251G memory and 4 GPU cards (GeForce RTX 3080).
The baseline algorithm FOCOPS based on \url{https://github.com/ymzhang01/focops}, which were offical code library. The other baseline algorithms include CPO, TRPO-L, PPO-L based on \url{https://github.com/openai/safety-starter-agents}, which published by OpenAI.
\subsection{Robots with Speed limit}
We used the MuJoCo environments provided by OpenAI Gym \cite{brockman2016openai} for this set of experiments. For agents manuvering on a two-dimensional plane, the cost is calculated as
$$
C(s,a) = \sqrt{v_x^2 + v_y^2}
$$
For agents moving along a straight line, the cost is calculated as
$$
C(s, a) = |v_x|
$$
where $v_x,v_y$ are the velocities of the agent in the $x$ and $y$ directions respectively.
\subsection{Robots with Circle}
\begin{figure}[H]
    \centering
    \includegraphics[width=0.22\textwidth]{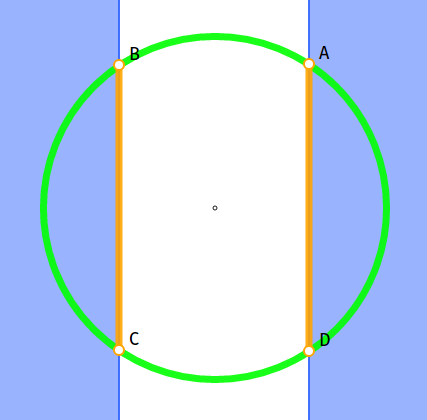}
    \label{fig:circle}
    \caption{In the Circle task, reward is maximized by moving along the green circle. The agent is not allowed to enter the blue regions, so its optimal constrained path follows the line segments $AD$ and $BC$.}
\end{figure}
The environment is inspired by \cite{AchiamHTA17}. Reward is maximized by moving along a circle of radius $d$:
\begin{align*}
R(s) &= \frac{v^T [-y, x]}{1+ \left| \|[x,y]\|_2 - d \right|}, \\
C(s) &= \pmb{1} \left[ |x| > x_{lim}\right],
\end{align*}
where $x,y$ are the coordinates in the plane, $v$ is the velocity, and $d, x_{lim}$ are environmental parameters. 
\newpage
\subsection{Algorithm Parameters}
\begin{table}[h]
    \centering
    
    \vskip 0.15in
    \begin{tabular}{l l l l l l}
    \toprule
       Hyperparameter & CUP & PPO-L & TRPO-L & CPO & FOCOPS \\
        \midrule
    No. of hidden layers & 2 & 2 & 2 & 2 & 2 \\
      No. of hidden nodes  & 64 & 64 & 64 & 64 & 64 \\
      Activation & $\tanh$ & $\tanh$ & $\tanh$ & $\tanh$ & $\tanh$\\
      Initial log std & -0.5 & -0.5 & -1 & -0.5 & -0.5\\
      Discount for reward $\gamma$ & 0.99 & 0.99 & 0.99 & 0.99 & 0.99 \\
      Discount for cost $\gamma_{C}$ & 0.99 & 0.99 & 0.99 & 0.99 & 0.99 \\
      Batch size & 5000 & 5000 & 5000 & 5000 & 5000 \\
      Minibatch size & 64 & 64 & N/A & N/A & 64 \\
      No. of optimization epochs & 10 & 10 & N/A & N/A & 10 \\
      Maximum episode length & 1000 & 1000 & 1000 & 1000 & 1000 \\
      GAE parameter (reward) & 0.95 & 0.95 & 0.95 & 0.95 & 0.95 \\ 
      GAE parameter (cost) & 0.95 & 0.95 & 0.95 & 0.95 & 0.95 \\
      Learning rate for policy & $3\times 10^{-4}$ & $3\times 10^{-4}$ & N/A & N/A & $3\times 10^{-4}$\\
      Learning rate for reward value net & $3\times 10^{-4}$ & $3\times 10^{-4}$ & $3\times 10^{-4}$ & $3\times 10^{-4}$ & $3\times 10^{-4}$ \\
      Learning rate for cost value net & $3\times 10^{-4}$ & $3\times 10^{-4}$ & $3\times 10^{-4}$ & $3\times 10^{-4}$ & $3\times 10^{-4}$ \\
      Learning rate for $\nu$ & 0.01 & 0.01 & 0.01 & N/A & 0.01 \\
      $L2$-regularization coeff. for value net & $10^{-3}$ & $3\times 10^{-3}$ & $3\times 10^{-3}$ & $3\times 10^{-3}$ & $10^{-3}$ \\
      Clipping coefficient & N/A & 0.2 & N/A & N/A & N/A \\
      Damping coeff. & N/A & N/A & 0.01 & 0.01 & N/A \\
       Backtracking coeff. & N/A & N/A & 0.8 & 0.8 & N/A \\
       Max backtracking iterations & N/A & N/A & 10 & 10 & N/A \\
       Max conjugate gradient iterations & N/A & N/A & 10 & 10 & N/A\\
       Iterations for training value net & 1 & 1 & 80 & 80 & 1 \\
     Temperature $\lambda$ & 1.5 & N/A & N/A & N/A & 1.5 \\
      Trust region bound $\delta$ & 0.02 & N/A & 0.01 & 0.01 & 0.02 \\
      Initial $\nu$, $\nu_{\max}$ & 0, 2 & 0, 1 & 0, 2 & N/A & 0, 2 \\
       \bottomrule
    \end{tabular}
    \caption{Hyper-parameters for robots.}
    \label{tab:hyperparameters}
\end{table}

\end{document}